\documentclass[11pt]{article}

\usepackage{fullpage}
\usepackage{url}
\usepackage{microtype}

\usepackage[utf8]{inputenc}
\usepackage{amsmath}
\usepackage{graphicx}
\usepackage{upgreek}
\usepackage{amsfonts}
\usepackage{amssymb}
\usepackage{amsthm}
\usepackage[mathscr]{euscript}
\usepackage{mathtools}
\usepackage{multirow}

\usepackage{xcolor}
\usepackage{xr}
\usepackage{chngcntr}
\usepackage[numbers]{natbib}
\usepackage{wrapfig}
\usepackage{caption}
\usepackage{subcaption}

\usepackage{apptools}
\usepackage[toc, page]{appendix}

\usepackage{tikz, pgf,pgfplots}
\pgfplotsset{compat=1.14}
\usetikzlibrary{calc,3d, patterns,quotes}
\usetikzlibrary{arrows,shapes,automata}
\usetikzlibrary{decorations.markings, positioning, fit}

\usepackage[ruled, vlined, linesnumbered]{algorithm2e}

\usepackage{thmtools}
\usepackage{thm-restate}
\declaretheorem{theorem}
\declaretheorem[name=Corollary]{corollary}

\RequirePackage{amsmath}
\RequirePackage{amssymb}
\ifx\proof\undefined 
\RequirePackage{amsthm}
\fi
\RequirePackage{bm} 
\RequirePackage{url}
\RequirePackage{color}

\newcommand{\Bcal}{\mathcal{B}}
\newcommand{\Ccal}{\mathcal{C}}
\newcommand{\Dcal}{\mathcal{D}}
\newcommand{\Ecal}{\mathcal{E}}
\newcommand{\Fcal}{\mathcal{F}}

\newcommand{\Hcal}{\mathcal{H}}
\newcommand{\Ical}{\mathcal{I}}

\newcommand{\Kcal}{\mathcal{K}}

\newcommand{\Mcal}{\mathcal{M}}
\newcommand{\Ncal}{\mathcal{N}}
\newcommand{\Ocal}{\mathcal{O}}

\newcommand{\Vcal}{\mathcal{V}}

\newcommand{\Xcal}{\mathcal{X}}

\newcommand{\EE}{\mathbb{E}} 
\newcommand{\NN}{\mathbb{N}} 
\newcommand{\PP}{\mathbb{P}} 
\newcommand{\RR}{\mathbb{R}} 
\newcommand*{\argmax}{\mathop{\mathrm{argmax}}}

\newcommand*{\one}{{\bf 1}}

\ifx\BlackBox\undefined
\newcommand{\BlackBox}{\rule{1.5ex}{1.5ex}}  
\fi

\ifx\QED\undefined
\def\QED{~\rule[-1pt]{5pt}{5pt}\par\medskip}
\fi

\ifx\proof\undefined
\newenvironment{proof}{\par\noindent{\bf Proof\ }}{\hfill\BlackBox\\[2mm]}
\fi

\ifx\theorem\undefined
\newtheorem{theorem}{Theorem}
\fi
\ifx\example\undefined
\newtheorem{example}{Example}
\fi
\ifx\property\undefined
\newtheorem{property}{Property}
\fi
\ifx\lemma\undefined
\newtheorem{lemma}[theorem]{Lemma}
\fi
\ifx\proposition\undefined
\newtheorem{proposition}[theorem]{Proposition}
\fi
\ifx\remark\undefined
\newtheorem{remark}[theorem]{Remark}
\fi
\ifx\corollary\undefined
\newtheorem{corollary}[theorem]{Corollary}
\fi
\ifx\definition\undefined
\newtheorem{definition}[theorem]{Definition}
\fi
\ifx\conjecture\undefined
\newtheorem{conjecture}[theorem]{Conjecture}
\fi
\ifx\fact\undefined
\newtheorem{fact}[theorem]{Fact}
\fi
\ifx\claim\undefined
\newtheorem{claim}[theorem]{Claim}
\fi
\ifx\assum\undefined
\newtheorem{assum}[theorem]{Assumption}
\fi

\newcommand{\wt}[1]{\widetilde{#1}}
\newcommand{\wh}[1]{\widehat{#1}}

\allowdisplaybreaks[4]
\newcommand{\prob}{\mathbb P}
\newcommand{\Ex}{\mathbb E}

\newcommand{\statespace}{\mathcal S}
\newcommand{\actionspace}{\mathcal A}

\newcommand{\numS}{S}
\newcommand{\numA}{A}
\newcommand{\wmin}{w_{\min}}

\newcommand{\range}{\operatorname{rng}}

\newcommand{\defeq}{:=}
\usepackage{xspace}

\mathtoolsset{showonlyrefs=true}

\usepackage{accents}
\newcommand{\Vub}{\widetilde V}
\newcommand{\Vlb}{\underaccent{\sim}{V}}
\newcommand{\Qub}{\widetilde Q}
\newcommand{\Qlb}{\underaccent{\sim}{Q}}
\newcommand{\psiub}{\widetilde \psi}
\newcommand{\psilb}{\underaccent{\widetilde{}}{\psi}}

\definecolor{DarkRed}{rgb}{0.75,0,0}
\definecolor{DarkGreen}{rgb}{0,0.5,0}
\definecolor{DarkPurple}{rgb}{0.5,0,0.5}
\definecolor{DarkBlue}{rgb}{0,0,0.7}

\definecolor{darkgreen}{rgb}{0,0.5,0}
\definecolor{mygreen}{RGB}{20, 120, 20}

\definecolor{commentcolor}{RGB}{0, 56, 160}

\SetCommentSty{mycommfont}
\usepackage[bookmarks, colorlinks=true, plainpages = false, citecolor = DarkBlue, urlcolor = DarkBlue, filecolor = black, linkcolor =DarkRed]{hyperref}

\makeatletter
\newcommand{\removelatexerror}{\let\@latex@error\@gobble}
\makeatother

\usepackage{tocbibind}

\newcommand*\circledmarker[1]{\tikz[baseline=(char.base)]{
            \node[shape=circle,draw,inner sep=.5pt] (char) {\small  #1};}}
\newcommand*\prnmarker[1]{(\MakeUppercase{\text{#1}})}
\newcommand*\circledmarked[2]{\overset{\circledmarker{#1}}{#2}}

\newcommand*\markedterm[2]{\underbrace{#2}_{\prnmarker{#1}}}

\SetKwFunction{FSample}{SampleEpisode}
\SetKwFunction{FOptPlan}{OptimistPlan}

\newcommand{\mas}{M}
\newcommand{\rhatsq}{\widehat{r^2}}
\newcommand{\gedge}{\overset{G}{\rightarrow}}

\newcommand{\ldef}{\vcentcolon=}

\title{Reinforcement Learning with Feedback Graphs}

\usepackage{authblk}

 \author[2]{Christoph Dann$^{1}$, Yishay Mansour$^{1, 3}$, Mehryar Mohri$^{1,4}$ \\ Ayush Sekhari$^{2}$  and  Karthik Sridharan}
\affil[1]{Google Research}
\affil[2]{Cornell University}
\affil[3]{Tel Aviv University}
\affil[4]{Courant Institute of Mathematical Sciences}

\date{}

\usepackage{multicol}
\usepackage{enumitem}
\usepackage{pifont}

\definecolor{ForestGreen}{RGB}{0,146,69}
\definecolor{BrickRed}{RGB}{230,22,22}
\definecolor{DarkGreen}{RGB}{1, 121, 91}

\begin{document}
\maketitle

\begin{abstract}\noindent
We study episodic reinforcement learning in Markov decision processes when the agent receives additional feedback per step in the form of several transition observations. Such additional observations are available in a range of tasks through extended sensors or prior knowledge about the environment (e.g., when certain actions yield similar outcome). We formalize this setting using a feedback graph over state-action pairs and show that model-based algorithms can leverage the additional feedback for more sample-efficient learning. We give a regret bound that, ignoring logarithmic factors and lower-order terms, depends only on the size of the maximum acyclic subgraph of the feedback graph, in contrast with a polynomial dependency on the number of states and actions in the absence of a feedback graph. Finally, we highlight challenges when leveraging a small dominating set of the feedback graph as compared to the bandit setting and propose a new algorithm that can use knowledge of such a dominating set for more sample-efficient learning of a near-optimal policy.
\end{abstract}
\nopagebreak[4]
\section{Introduction}
\label{sec:intro}
There have been many empirical successes of reinforcement learning (RL) in tasks where an abundance of samples is available \citep{mnih2015human,silver2017mastering}. 
However, for many real-world applications the sample complexity of RL is still prohibitively high. It is therefore crucial to simplify the learning task by leveraging domain knowledge in these applications. 
A common approach is imitation learning where demonstrations from domain experts can greatly reduce the number of samples required to learn a good policy~\citep{ross2011reduction}. Unfortunately, in many challenging tasks such as drug discovery or tutoring system optimization, even experts may not know how to perform the task well. They can nonetheless give insights into the structure of the task, e.g., that certain actions yield similar behavior in certain states. These insights could in principle be baked into a function class for the model or value-function, but this is often non-trivial for experts and RL with complex function classes is still very challenging, both in theory and practice \citep{jiang2017contextual, dann2018oracle, du2019provably, henderson2017deep}. 

A simpler, often more convenient approach to incorporating structure from domain knowledge is to provide additional observations to the algorithm. In supervised learning, this is referred to as \emph{data augmentation} and best practice in areas like computer vision with tremendous performance gains \citep{krizhevsky2012imagenet,zhang2017mixup}. Recent empirical work \citep{lin2019towards, kostrikov2020image, laskin2020reinforcement} suggests that data augmentation is similarly beneficial in RL. However, to the best of our knowledge, little is theoretically known about the question:
\begin{quote}
\emph{How do side observations in the form of transition samples (e.g. through data augmentation) affect the sample-complexity of online RL?}
\end{quote}
\begin{wrapfigure}{r}{0.25\textwidth}
    \centering
    \includegraphics[width=\linewidth]{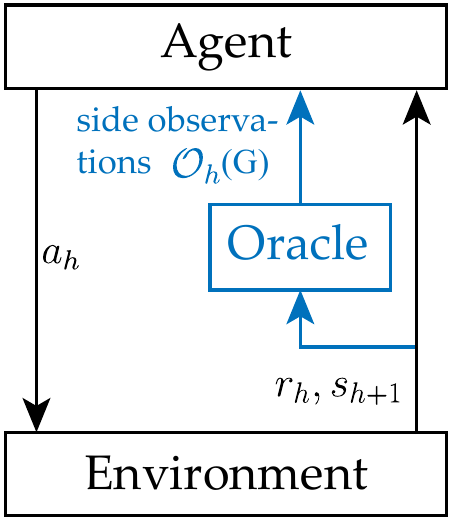}
    \caption{RL loop with side observations from a data augmentation oracle}
    \label{fig:protocol}
\end{wrapfigure}

In this paper, we take a first step toward answering this question and study RL in finite episodic Markov decision processes (MDPs) where, at each step, the agent receives some side information from an \emph{online data augmentation oracle}, in addition to the reward and next state  information $(r_h, s_{h + 1})$ directly supplied by the environment (Figure~\ref{fig:protocol}). This side information is a collection of observations, pairs of reward and next state, for some state-action pairs other than the one taken by the agent in that round. What can be observed is specified by a \emph{feedback graph} \citep{mannor2011bandits} over state-action pairs: an edge in the feedback graph from state-action pair $(s, a)$ to state-action pair $(\bar s, \bar a)$ indicates that, when the agent takes action $a$ at state $s$, the oracle also provides the reward and next-state sample $(r', s')$ that it would have seen if it would have instead taken the action $\bar a$ at state $\bar s$. Specifically, at each time step, the agents not only gets to see the outcome of executing the current (s, a), but also an outcome of executing all the corresponding state-action pairs that have an edge from (s, a) in the feedback graph. 

To illustrate this setting, consider a robot moving in a grid world. Through auxiliary sensors, it can sense positions in its line of sight. When the robot takes an action to move in a certain direction, it can also predict what would have happened for the same action in other positions in the line of sight. The oracle formalizes this ability and provides the RL algorithm with transition observations of (hypothetical) movements in the same direction from nearby states. 
Here, the feedback graph connects state-action pairs with matching states and actions in the line of sight (Figure~\ref{fig:robot_example}).

For another illustrative example where feedback graphs occur naturally, consider a robot arm grasping different objects and putting them in bins. In this task, the specific shape of the object is relevant only when the robot hand is close to the object or has grasped it. In all other states, the actual shape is not significant and thus,  the oracle can provide additional observations to the learning algorithm by substituting different object shapes in the state description of current  transition. In this case, all such state-action pairs that are identical up to the object shape are connected in the feedback graph. This additional information can be easily modeled using a feedback graph but is much harder to incorporate in models such as factored MDPs \citep{boutilier1999decision} or linear MDPs \citep{jin2019provably}.
RL with feedback graphs also generalizes previously studied RL settings, such as learning with aggregated state representations~\citep{dong2019provably} and certain optimal stopping domains~\citep{goel2017sample}. Furthermore, it can also be used to analyze RL with auxiliary tasks (see Section~\ref{sec:multitaskrl}).  

In this paper, we present an extensive study of RL with MDPs in the presence of side observations through feedback graphs. We prove in Section~\ref{sec:modelbased} that optimistic algorithms such as \textsc{Euler} or \textsc{ORLC} \citep{zanette2019tighter, dann2019policy} augmented with the side 
observations 
achieve significantly more favorable regret guarantees: the dominant terms of the bounds only depend on the \emph{mas-number}\footnote{\emph{mas-number} of a graph is defined as the size of its largest acyclic subgraph.} $\mas$ of the feedback graph  as opposed to an explicit dependence on $S$ and $A$, the number of states and actions of the MDP, which can be substantially larger than $M$ in many cases (See Table~\ref{tab:summary} for a summary of our results). We further give lower bounds which show that our regret bounds are in fact minimax-optimal, up to lower-order terms, in the case of symmetric feedback graphs (see Section~\ref{sec:lowerbounds}). 

\begin{figure}
    \centering
    \includegraphics[width=0.74\linewidth]{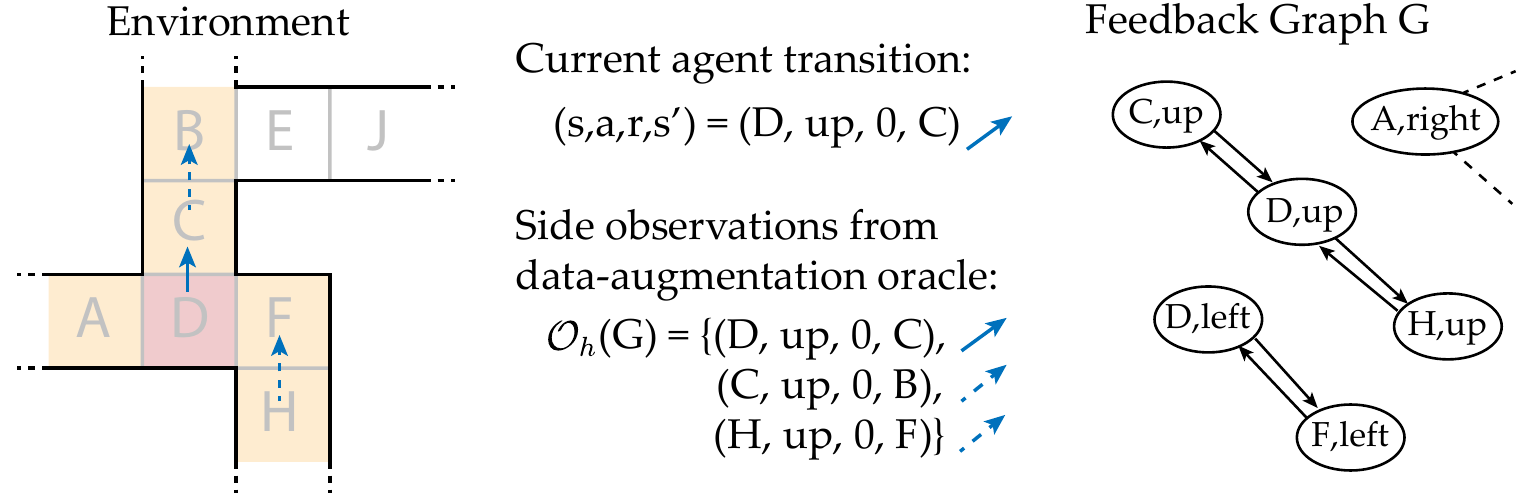}
    \caption{Example for RL with feedback graphs: Through additional sensors, the robot in state $D$ can also observe nearby states (yellow) and when taking the action \emph{up}, the oracle provides the actual transition (solid arrow) as well as hypothetical transitions (dashed arrow) from nearby states. This is formalized by a \emph{feedback graph} $G$ over state-action pairs shown on the right (snippet). Since $(D, up)$ has an edge to $(C, up)$ and $(G, up)$ in the feedback graph, the agent receives a (hypothetical) transition observation for both from the oracle.}
    \label{fig:robot_example}
\end{figure}

While learning with feedback graphs has been widely studied in the multi-armed bandit setting \citep[e.g.][]{mannor2011bandits,AlonCesa-BianchiGentileMansour2013,CortesDeSalvoGentileMohriYang2018,CortesDeSalvoGentileMohriYang2019,AroraMarinovMohri2019}, the corresponding in  the MDP setting is qualitatively different as the agent cannot readily access all vertices in the feedback graph (see section~\ref{sec:domination}). A vertex $(s, a)$ of the feedback graph may be very informative but the agent does not know how to reach state $s$ yet. To formalize this, we prove through a statistical lower bound that leveraging a small \emph{dominating set}\footnote{Dominating set of a graph (D) is defined as a subset of the vertices of a graph such that every vertex is either belongs to D or has an edge from a vertex in D. In our problem setting, the dominating set reveals information about the entire MDP.}  to improve over the sample complexity of RL is fundamentally harder in MDPs than in multi-armed bandits. 
Finally, we propose a simple algorithm to addresses the additional challenges of leveraging a small dominating set in MDPs when learning an $\epsilon$-optimal policy and prove that its sample complexity scales with the size of the dominating set in the main $1/\epsilon^2$- term only.

\section{Background and Notation}
\paragraph{Episodic Tabular MDPs:} The agent interacts with an MDP in episodes indexed by $k$. Each episode is a sequence $(s_{k, 1}, a_{k, 1}, r_{k, 1}, \ldots,\allowbreak s_{k, H}, a_{k, H}, r_{k, H})$ of $H$ states $s_{k,h} \in \statespace$, actions $a_{k,h} \in \actionspace$ and scalar rewards  $r_{k,h} \in [0,1]$. The initial state $s_{k, 1}$ can be chosen arbitrarily, possibly adversarially. Actions are taken as prescribed by the agent's policy $\pi_k$ which are deterministic and time-dependent mappings from states to actions, i.e., $a_{k,h} = \pi_k(s_{k,h}, h)$ for all time steps $h \in [H] := \{1, 2, \dots H\}$. The successor states and rewards are sampled from the MDP as $s_{k,h+1} \sim P(s_{k,h}, a_{k,h})$ and $r_{k,h} \sim P_R(s_{k,h}, a_{k,h})$. 

\paragraph{State-action pairs $\Xcal$: } We denote by $\Xcal$ the space of all state-action pairs $(s,a)$ that the agent can encounter, i.e.,  visit $s$ and take $a$. The state space and action space are then defined as $\statespace = \{s \colon \exists a : ~ (s,a) \in \Xcal\}$ and $\actionspace = \{a \colon \exists s : ~ (s,a) \in \Xcal\}$, respectively. This notation is more general than the typical definition of $\statespace$ and $\actionspace$ and more convenient for our purposes. We restrict ourselves to tabular MDPs where $\Xcal$ is finite. The agent only knows the horizon $H$ and $\Xcal$, but has no access to the reward and transition distributions. For a pair $x \in \Xcal$, we denote by $s(x)$ and $a(x)$ its state and action respectively.

\paragraph{Value Functions and Regret:} 
The Q-value of a policy is defined as the reward to go given the current state and action when the agent follows $\pi$ afterwards
\[    
    Q^\pi_h(s,a) \defeq \Ex\left[ \sum_{t=h}^H r_{k,t} \Bigg| a_{k,h} = a, s_{k,h} = s, a_{k,h+1:H} \sim \pi \right],
\]
and the state-values of $\pi$ are $V^\pi_h(s) \defeq Q_h^\pi(s, \pi_h(s))$. The expected return of a policy in episode $k$ is simply the initial value $V^{\pi}_1(s_{k,1})$. Any policy that achieves optimal reward to go, i.e., $\pi(s,h) \in \argmax_{a} Q_h^\pi(s,a)$ is called optimal. We use superscript $\star$ to denote any optimal policy and its related quantities.
The quality of an algorithm can be measured by its \emph{regret}, the cumulative difference of achieved and optimal return, which after $T$ episodes is 
\[
R(T) \defeq \sum_{k=1}^T ( V^\star_1(s_{k,1}) - V^{\pi_k}_1(s_{k,1})).
\]

\begin{table}
    \centering
    \bgroup
\def\arraystretch{1.7}
    \begin{tabular}{@{\hspace{0cm}}l|lll@{\hspace{0cm}}}
       \multicolumn{1}{l}{}&\multicolumn{1}{l}{} & \multicolumn{1}{l}{\textbf{Worst-Case Regret}} & 
       \multicolumn{1}{l}{\textbf{Sample Complexity}} \\
        \hline 
        \multirow{ 2}{*}{\begin{minipage}{1.5cm}\centering without feedback graph
        \end{minipage}} & \textsc{ORLC} \citep{dann2019strategic} & $\tilde O(\sqrt{\textcolor{DarkRed}{\numS \numA} H^2 T} + \textcolor{DarkRed}{\numS \numA} \hat \numS H^2)$ & 
        $\tilde O\left( \frac{\textcolor{DarkRed}{\numS \numA} H^2}{\epsilon^2} + \frac{\textcolor{DarkRed}{\numS \numA} \hat \numS H^2}{\epsilon}\right)$
        \\
         & Lower bounds \citep{dann2015sample, osband2016lower} &
         $\tilde \Omega(\sqrt{\textcolor{DarkRed}{\numS \numA} H^2 T})$ &
                 $\tilde \Omega\left( \frac{\textcolor{DarkRed}{\numS \numA} H^2}{\epsilon^2}\right)$\\
        \hline 
        \multirow{ 3}{*}{\begin{minipage}{1.5cm}\centering with feedback graph
        \end{minipage}} & \textsc{ORLC} [Thm.~\ref{thm:cipoc_independencenumber}, Cor.~\ref{cor:epspolicy}] & $\tilde O(\sqrt{\textcolor{mygreen}{{M}} H^2 T} + \textcolor{mygreen}{{M}} \hat \numS H^2)$ & 
        $\tilde O\left( \frac{\textcolor{mygreen}{{M}} H^2}{\epsilon^2} + \frac{\textcolor{mygreen}{{M}} \hat \numS H^2}{\epsilon}\right)$
        \\  
        & Algorithm~\ref{alg:dominatingset} [Thm.~\ref{thm:samplecomplexity_domset}]
        & at least $O(\sqrt{\textcolor{mygreen}{{\gamma}} T^{2/3}})$ &
        $\tilde O\left( \frac{\textcolor{mygreen}{{\gamma}} H^3}{p_0 \epsilon^2} +\frac{\textcolor{mygreen}{{\gamma}} \hat \numS H^2}{p_0 \epsilon}  +  \frac{\textcolor{mygreen}{{\mas}} \hat \numS H^2}{p_0}\right)$
        \\ 
        & Lower bounds  [Thm.~\ref{thm:indeplowerbound}, Thm.~\ref{thm:domsetlowerbound}]
        & $\tilde \Omega( \sqrt{\textcolor{mygreen}{{\alpha}} H^2 T})$  &
        $\tilde \Omega\left( \frac{\textcolor{mygreen}{{\gamma}} H^2}{p_0 \epsilon^2} + \frac{\textcolor{mygreen}{{\alpha}}}{p_0} \wedge \frac{\textcolor{mygreen}{{\alpha}} H^2}{\epsilon^2} \right)$
        \\\hline
    \end{tabular}
        \caption{Comparison of our main results.  $\alpha$, $\gamma$ and $\mas$ denote the independence number, domination number and mas-number of the feedback graph respectively, with $\gamma \leq \alpha \leq \mas \leq \numS \numA$. }\label{tab:summary}
    \egroup
    
\end{table}

\section{Reinforcement Learning in MDPs with Feedback Graphs}
\label{sec:feedbackgraphs}

In the typical RL setting, when the agent takes action $a_h$ at state $s_h$,
it can only observe the reward $r_h$ and next-state
$s_{h+1}$. Thus, it only observes the transition
$(s_{h},a_{h},r_{h},s_{h+1})$.
Here, we assume that the agent additionally receives some side
observations from an oracle (Figure~\ref{fig:protocol}).  We denote by
$\Ocal_{k,h}(G) \subseteq \Xcal \times [0,1] \times \statespace$ the set
of transition observations thereby available to the agent in episode
$k$ and time $h$.\footnote{We often omit episode indices $k$ when unambiguous to reduce clutter.}  $\Ocal_{k,h}(G)$ thus consists of the tuples
$(s, a, r, s')$ with state $s$, action $a$, reward $r$ and next state
$s'$, including the current transition
$(s_{k,h}, a_{k,h}, r_{k,h}, s_{k,h+1})$.  For notational convenience,
we also sometimes write transition tuples in the form $(x, r, s')$
where $x = (s, a)$ is the state-action pair.

As discussed in Section~\ref{sec:intro}, the oracle information is
typically based on prior knowledge about the environment and
additional sensors. 
The goal of this paper is not to study how specific oracles work but rather how RL algorithms can benefit from side observations.
To that end, we formalize the side observations available to
the agent by a directed graph $G = (\Xcal, E)$ over state-action pairs
called a \emph{feedback graph}. An edge $x \gedge \bar x$ (short for $(x, \bar x) \in E$)
from $x \in \Xcal$ to $\bar x \in \Xcal$ indicates that, when the agent
takes action $a(x)$ at state $s(x)$, it can observe a reward and
next-state sample $(r', s')$ it would have received, had it taken
action $a(\bar x)$ at state $s(\bar x)$. To simplify the discussion, self-loops will be implicit and not included in the feedback graph. Essentially, $G$ only stipulates which side observations are available \emph{in addition} to the currently performed transition.
See Figure~\ref{fig:robot_example} for a concrete
example.
Formally, the set of transition observations received by the agent
when it takes action $a_{h}$ at state $s_{h}$ is thus
\begin{align}
    \Ocal_{h}(G) = \{((s_{h}, a_{h}), r_{h}, s_{h+1}) \} \cup \{(x, r', s') \colon (s_{h}, a_{h})
\gedge x \},
\end{align} where each observation $(x, r, s')$ contains an
independent sample from the next state distribution $s' \sim P(x)$ and
reward distribution $r \sim P_R(x)$ given all \emph{previous}
observations. Note that we allow simultaneous transition observations to be dependent.\footnote{
This is important as it allows the oracle
to generate side observations from the current, possibly noisy, transition and
feed them to the algorithm without the need for a completely fresh
sample with independent noise.
}

\paragraph{Important Graph Properties:} The analysis of regret and
sample-complexity in this setting makes use of 
following properties of the feedback graph:
\begin{itemize}

    \item \textbf{Mas-number $\bm \mas$}: A set of vertices $\Vcal \subseteq
      \Xcal$ form an acyclic subgraph if the subgraph $(\Vcal, \{
      (v,w) \subseteq \Vcal \times \Vcal \colon v \gedge w
      \})$ of $G$ restricted to $\Vcal$ is loop-free. We call the size
      of the maximum acyclic subgraph the \emph{mas-number} $\mas$ of $G$.

    \item \textbf{Independence number $\bm \alpha$}: A set of vertices
      $\Vcal \subseteq \Xcal$ is an independent set if there is no
      edge between any two nodes of that set:
      $\forall v,w \in \Vcal \colon v \, \not \gedge w$. The
      size of the largest independent set is called the
      \emph{independence number} $\alpha$ of $G$.

    \item \textbf{Domination number $\bm \gamma$}: A set of vertices
      $\Vcal \subset \Xcal$ form a dominating set if there is an edge
      from a vertex in $\Vcal$ to any vertex in $G$:
      $\forall x \in \Xcal \; \exists v \in \Vcal \colon v
      \gedge x$. The size of the smallest dominating set is
      called the \emph{domination number} $\gamma$.
\end{itemize}
For any directed graph $G$, 
$$\gamma \leq \alpha \leq \mas \leq |\Xcal|,$$ where all the above inequalities can be a factor of 
$\Theta(|\Xcal|)$ apart in the worst case. Independence- and mas-number coincide,
$\alpha = \mas$, for symmetric (or undirected) graphs where for every
edge there is an edge pointing backward. We defer to
Appendix~\ref{app:graphprop_discussion} a more extensive discussion with
examples of how the above graph properties can differ from each other. 
Here, we only give two relevant examples:

\begin{enumerate}[label=\alph*)]
    \item \emph{State aggregation} \citep{dong2019provably} can be considered a special case of learning with feedback graphs  where the feedback graph consists of disjoint cliques, each consisting of the state-action pairs whose state belongs to a an aggregated state. Here $M = \alpha = \gamma = \numA B$ where $B$ is the number of aggregated states and $\numA=|\actionspace|$. 
    \item \emph{Learning with auxiliary tasks} where the agent aims to optimize
several auxiliary reward functions can also be modeled as RL with a
feedback graph where the MDP state space is augmented with a task
identifier. See Section~\ref{sec:multitaskrl} for an extended
discussion.
\end{enumerate}

\section{Graph Regret Bounds for Model-Based RL}
\label{sec:modelbased}
In this section, we show the benefits of a feedback graph in achieving more favorable  learning guarantees. We focus on model-based algorithms that
follow the \emph{optimism in the face of uncertainty principle}, a popular paradigm that achieves the tightest known  regret / PAC bounds for the tabular MDP setting. Specifically, we will analyze a version of the \textsc{Euler} or
\textsc{ORLC} algorithm \citep{dann2019policy, zanette2019tighter} shown in Algorithm~\ref{alg:mb}.

The algorithm proceeds in rounds, and maintains first and second moments of the immediate
reward i.e. $\wh r(x)$ and $\rhatsq(x)$ respectively,  transition frequencies $\wh P(x)$ and
the number of observations $n(x)$ for each state-action pair
$x \in \Xcal$ as statistics to form an estimate of the true model. At the start of every round, in
line~\ref{lin:optplan_main}, we  compute
a new policy $\pi_k$ using \FOptPlan, a version of value iteration with reward bonuses.\footnote{The subroutine \FOptPlan returns an upper-confidence bound $\Vub_{k,h}$ on the optimal value function $V^\star_{h}$ as well as a lower-confidence
bound $\Vlb_{k,h}$ on the value function of the returned policy
$V^{\pi_k}_{h}$, and can be can be viewed as an extension of the UCB policy from the bandit literature to the MDP setting}
Next, we execute the policy $\pi_k$ for one episode and update the model statistics using the samples collected. 

The main difference between 
Algorithm~\ref{alg:mb} and the \textsc{Euler} or \textsc{ORLC}
algorithms without feedback graphs is in the way we update our model statistics. Specifically, our algorithm also includes the additional observations $\Ocal_h(G)$ available along with the current transition (as stipulated by the feedback graph $G$) to update the model statistics at the end of every round. This is highlighted in green in lines~\ref{lin:fg1}--\ref{lin:fg2} of  Algorithm~\ref{alg:mb}.

Though being a straightforward extension of the previous algorithms, we show that Algorithm~\ref{alg:mb} can benefit from the feedback graph structure, and satisfies the following regret and certificates sizes (IPOC \citep{dann2019policy}) bound that only scales with the mas-number $\mas$ of the feedback graph $G$, and does not have any explicit dependence on size of the state or action space (in the dominant terms). Our main technical contribution is in the analysis, which we describe in the rest of this section.

\begin{algorithm}[t]
\SetAlgoLined
\SetNoFillComment
\SetInd{0.7em}{0.5em}
\SetKwInOut{Inputa}{input}
\SetKwInOut{Outputa}{ouput}
\SetKwInOut{Returna}{return}
  \SetKwProg{Fn}{function}{:}{}
  \SetNlSty{bfseries}{\color{black}}{}
\Inputa{failure tolerance $\delta \in (0, 1]$, state-action space $\Xcal$, episode length $H$;}
\Inputa{maximum transition support $\wh \numS \leq \|P(x)\|_0 \leq \numS$;}

    Initialize $n_1(x) \gets 0, ~ \wh r_1(x) \gets 0, ~ \rhatsq_1(x) \gets 0, ~ \wh P_1(x) \gets e_1 \in \{0, 1\}^{\numS} $ for all $x  \in \Xcal$\;
\For(\tcp*[f]{Main Loop}){episode $k=1, 2, 3, \dots$}{
    $(\pi_k, \Vub_{k,h}, \Vlb_{k,h}) \gets$ \FOptPlan{$n_k, \wh r_k, \rhatsq_k, \wh P_k$}
        \label{lin:optplan_main}
    \tcp*{VI with reward bonuses, see appendix}
    Receive initial state $s_{k,1}$
    \;
    $(n_{k+1}, \wh r_{k+1}, \rhatsq_{k+1}, \wh P_{k+1}) \gets $\FSample{$\pi_k, s_{k,1}, n_k, \wh r_k, \rhatsq_k, \wh P_k$}
    \label{lin:sample_ep}\;
}

\Fn{\FSample($\pi, s_1, n, \wh r, \rhatsq, \wh P$)}{
\For{$h=1, \dots H$}{
    Take action $a_{h} = \pi(s_{h}, h)$ and transition to $s_{h+1}$ with reward $r_h$\;
    \textcolor{mygreen}{Receive transition observations
      $\Ocal_{h}(G)$}\tcp*{As stipulated by feedback graph $G$} \label{lin:fg1}
    \textcolor{mygreen}{\For(\tcp*[f]{Update empirical model and number of observations}){transition $(x,r, s') \in \Ocal_{h}(G)$ \label{lin:fg2}}{
    \textcolor{black}{
     $\begin{array}{ll}
         n(x) \gets n(x) + 1, 
        &\wh P(x) \gets \frac{n(x)-1}{n(x)} \wh P(x) +  \frac{1}{n(x)} e_{s'},  \\ 
         \wh r(x) \gets \frac{n(x)-1}{n(x)} \wh r(x) + \frac{1}{n(x)}r, \qquad
        & \rhatsq(x) \gets \frac{n(x)-1}{n(x)} \rhatsq(x) + \frac{1}{n(x)}r^2, 
       \end{array}$ \\
       $\qquad$where $e_{s'} \in \{0,1\}^{\numS}$ has $1$ on the $s'$-th position\;
    }}}
    \Returna{$(n, \wh r, \rhatsq, \wh P)$}
}
}
\caption{Optimistic model-based RL algorithm}
\label{alg:mb}
\end{algorithm}

\begin{theorem}[Cumulative IPOC and regret bound]
For any tabular episodic MDP with episode length $H$, state-action
space $\Xcal \subseteq \statespace \times \actionspace$ and directed
feedback graph $G$, Algorithm~\ref{alg:mb} satisfies with probability
at least $1 - \delta$ an IPOC bound for all number of episodes $T$ of
\begin{align}
	 \label{eqn:nobias_regret}
	 \wt O\left(\sqrt{\mas H^2 T} 
	 + \mas \wh \numS H^2 \right),
\end{align}
where $\mas$ is the size of the maximum acyclic
subgraph of $G$ and algorithm parameter $\wh \numS \leq \numS$ is a
bound on the number of possible successor states of each $x \in
\Xcal$.\\ 
Equation~\eqref{eqn:nobias_regret} bounds the
cumulative certificate size $\sum_{k=1}^T (\Vub_1(s_{k,1}) -
\Vlb_1(s_{k,1}))$ and the regret $R(T)$.
\label{thm:cipoc_independencenumber}
\end{theorem}

The above regret bound replaces a factor of $\numS \numA$ in the regret
bounds for RL without side observations \citep{dann2019policy,zanette2019tighter} with
the mas-number $\mas$ (see also Table~\ref{tab:summary}). This is a substantial improvement since, in many
feedback graphs $\numS \numA$ may be very large while $\mas$ is a constant. 
 The only remaining polynomial dependency on $\wh \numS \leq \numS$ in
the lower-order term is typical for model-based algorithms.

On the lower bound side, we show in Section~\ref{sec:lowerbounds} that the regret is at-least $\wt \Omega(\sqrt{\alpha H^2 T})$, where $\alpha$ denotes the independence number of $G$.  While $\mas$ and $\alpha$ can differ by as much as $|\Xcal| - 1$ for general graphs, they match for symmetric feedback graphs
(i.e. $\mas = \alpha$).\footnote{We call a graph $G$
  symmetric if for every edge
  $x \gedge y$, there also exists a back edge $y \gedge x$} In that
case, our regret bound in Theorem \ref{thm:cipoc_independencenumber}
 is optimal up to constant terms and $\log$-factors, and Algorithm~\ref{alg:mb} cannot be improved further.

We now discuss how the analysis of Theorem \ref{thm:cipoc_independencenumber} differs from existing ones, with the full proof deferred to Appendix~\ref{app:modelbasedproofs}. Assuming that the value
functions estimated in $\FOptPlan$ are valid confidence bounds, that
is, $\Vlb_{k,h} \leq V^{\pi}_h \leq V^\star_h \leq \Vub_{k,h}$ for all $k \in [T]$ and $h \in [H]$, we
bound regret as their differences
\begin{equation}
\label{eqn:basicbound1}
  R(T) \leq \sum_{k=1}^T \big[ \Vub_{k,1}(s_{k,1}) - \Vlb_{k,1}(s_{k,1}) \big]
    \lesssim \sum_{k=1}^T \sum_{h=1}^H \sum_{\substack{x \in \Xcal}} w_{k,h}(x) \left[ H \wedge
    \left[ 
    \frac{\sigma_{k,h}(x)}{\sqrt{n_k(x)}} + \frac{\wh \numS H^2}{n_k(x)}
    \right]
    \right],
\end{equation}
where $\lesssim$ and $\gtrsim$ ignore constants and $\log$-terms and
where $\wedge$ denotes the minimum operator.  The second step is a
bound on the value estimate differences derived through a standard
recursive argument. Here,
$w_{k,h}(x) = \PP\big((s_{k,h}, a_{k,h}) = x \mid \pi_k, s_{k,1} \big)$ is the
probability that policy $\pi_k$ visits $x$ in episode $k$ at time $h$.  In
essence, each such expected visit incurs regret $H$ or a term that
decreases with the number of observations $n_k(x)$ for $x$ so far. In
the expression above,
$\sigma^2_{k,h}(x) = \operatorname{Var}_{r \sim P_R(x)}(r) +
\operatorname{Var}_{s' \sim P(x)}( V^{\pi_k}_{h+1}(s'))$ is the
variance of immediate rewards and the policy value with respect to one
transition.

In the bandit case, one would now apply a concentration argument to
turn $w_{k,h}(x)$ into actual visitation indicators but this would
yield a loose regret bound of $\Omega(\sqrt{H^3T})$ here. Hence,
techniques in the analysis of UCB in bandits with graph feedback
  \citep{lykouris2019graph} based on discrete pigeon-hole arguments
  cannot be applied here without incurring suboptimal regret in $H$.  Instead, we apply a probabilistic argument
to the number of observations $n_k(x)$. We show that, with high
probability, $n_k(x)$ is not much smaller than the total visitation
probability so far of all nodes
$x' \in \Ncal_{G}(x) \defeq \{x\} \cup \{x' \in \Xcal \colon x'
\gedge x\}$ that yield observations for $x$:
\begin{align}
    n_k(x) \gtrsim \sum_{i=1}^{k} \sum_{x' \in \Ncal_G(x)} w_{i}(x'), \textrm{ with }
    w_i(x) = \sum_{h=1}^H w_{i,h}(x). \nonumber
\end{align}
This only holds when $\sum_{i=1}^{k} \sum_{x' \in \Ncal_G(x)} w_{i}(x') \gtrsim H$.
Hence, we split the sum over $\Xcal$ in \eqref{eqn:basicbound1} in $U_k = \left\{ x \in \Xcal \colon \sum_{i=1}^{k} \sum_{x' \in \Ncal_G(x)} w_{i}(x') \gtrsim H \right\}$ and complement $U^c_k$. Ignoring fast decaying  $1 / n_k(x)$ terms, this yields
\begin{align}
    \eqref{eqn:basicbound1}
  & \lesssim
    \sum_{k=1}^T\!\left[ \sum_{x \in U_k^c} w_{k}(x) H
    + \!\!
    \sum_{x \in U_k} \sum_{h=1}^H w_{k,h}(x)
    \frac{\sigma_{k,h}(x)}{\sqrt{n_k(x)}}\right]\nonumber\\
    & \lesssim \markedterm{a}{\sum_{k=1}^T \sum_{x \in U_k^c}
      w_{k}(x)} H
     + 
    \markedterm{b}{\sqrt{\sum_{k=1}^T \sum_{x \in \Xcal} \sum_{h=1}^H w_{k,h}(x)\sigma_{k,h}^2(x)}} ~ \cdot \markedterm{c}{\sqrt{\sum_{k=1}^T \sum_{x \in U_k} \frac{w_k(x)}{\sum_{i=1}^{k} \sum_{x' \in \Ncal_G(x)} w_{i}(x')}}} ,
    \nonumber
\end{align}
where the second step uses the Cauchy-Schwarz inequality. The law of
total variance for MDPs \citep{azar2012sample} implies that
$\prnmarker{b} \lesssim H\sqrt{T}$. It then remains to bound
$\prnmarker{a}$ and $\prnmarker{c}$, which is the main technical 
innovation in our proof. Observe that both $\prnmarker{a}$ and $\prnmarker{c}$ are sequences of functions that map each
node $x$ to a real value $w_k(x)$. While $\prnmarker{a}$ is a
thresholded sequence that effectively stops once a node has
accumulated enough weight from the in-neighbors, $\prnmarker{c}$
is a self-normalized sequence. We derive the following two novel
results to control each term. We believe these could be of general interest.

\begin{lemma}[Bound on self-normalizing real-valued graph sequences]
Let $G = (\Xcal, \Ecal)$ be a directed graph with the finite vertex set
$\Xcal$ and mas-number $\mas$, and let $(w_k)_{k \in [T]}$ be a sequence
of weights $w_k:\Xcal \rightarrow \RR^+$ such that for all
$k$, $\sum_{x \in \Xcal}w_k(x) \leq w_{\max}$. Then, for any
$w_{\min} > 0$,
\begin{align}
	\sum_{k =1}^T \sum_{x \in \Xcal} \frac{\one\{w_k(x) \geq w_{\min}\} w_k(x)}{\sum_{i=1}^k \sum_{x' \in \Ncal_G(x)} w_i(x')} 
	\leq
	2 \mas \ln \left(eT  \cdot \frac{w_{\max}}{w_{\min}}\right), 
\end{align}
where $\Ncal_{G}(x) = \{x\} \cup \{ y \in \Xcal~ |~  y \gedge x \}$ denotes the set of all vertices that have an edge to $x$ in $G$ and $x$ itself.
\label{lem:wsum_mas_main}
\end{lemma}

\begin{lemma}
Let $G = (\Xcal, \Ecal)$ be a directed graph with vertex set $\Xcal$ and let $w_k$ be a sequence of weights $w_k \colon \Xcal \rightarrow \RR^+$. Then, for any threshold $C \geq 0$, 
\begin{align}
	\sum_{x \in \Xcal}\sum_{k=1}^\infty  w_k(x) \cdot \one\left\{ \sum_{i=1}^{k} \sum_{x' \in \Ncal_G(x)} w_i(x') \leq C\right\}
	\leq \mas C
\end{align}
	where $\Ncal_{G}(x)$ is defined as in Lemma~\ref{lem:wsum_mas_main}.
	\label{lem:mas_pigeonhole_main}
\end{lemma}

We apply Lemma~\ref{lem:wsum_mas_main} and Lemma~\ref{lem:mas_pigeonhole_main} to get the bounds $\prnmarker{a} \lesssim \mas H$ and $\prnmarker{c} \lesssim \sqrt{\mas}$ respectively.  Plugging these bounds back in \eqref{eqn:basicbound1} yields the desired regret bound.  Note that both Lemma~\ref{lem:wsum_mas_main} and Lemma~\ref{lem:mas_pigeonhole_main} above operate on a
sequence of node weights as opposed to one set of node weights as in
the technical results in the analyses of EXP-type algorithms
\citep{AlonCesa-BianchiGentileMansour2013}. The proof of
Lemma~\ref{lem:wsum_mas_main} uses a potential function and a
pigeon-hole argument.  The proof for
Lemma~\ref{lem:mas_pigeonhole_main} relies on a series of careful
reduction steps, first to integer sequences and then to certain binary
sequences and finally a pigeon-hole argument.  (full proofs are deferred to Appendix~\ref{app:graphproofs}).

\section{Example Application of Feedback Graphs: Multi-Task RL}
\label{sec:multitaskrl}
In this section, we show that various multi-task RL problems can be
naturally modelled using feedback graphs and present an analysis of
these problems.
We consider the setting where there are $m$ tasks in an episodic
tabular MDP. All tasks share the same dynamics $P$ but admit different
immediate reward distributions $P_R^{(i)}$, $i \in [m]$. We assume the
initial state is fixed, which generalizes without loss of generality
to stochastic initial states. We further assume that the reward
distributions of all but one task are known to the agent. Note that
this assumption holds in most auxiliary task learning settings and
does not trivialize the problem (see the next section for an
example). The goal is to learn a policy that, given the task identity
$i$, performs $\epsilon$-optimally. This is equivalent to learning an
$\epsilon$-optimal policy for each task.

The naive solution to this problem consists of using $m$ instances of
any existing PAC-RL algorithm to learn each task separately. Using
Algorithm~\ref{alg:mb} as an example, this would require $\wt
O\left( \frac{\mas H^2}{\epsilon^2} + \frac{\wh \numS M
  H^2}{\epsilon}\right)$ episodes per task and in total
\begin{align}
\wt  O\left( \frac{ m ( 1 + \epsilon \wh \numS) \mas H^2}{\epsilon^2} \right)
\label{eqn:separate_rl_sc}
\end{align} 
episodes. When there is no additional feedback, the mas-number is
simply the number of states and actions $\mas = \numS \numA$.  If the
number of tasks $m$ is large, this can be significantly more costly
than learning a single task. We will now show that this dependency on
$m$ can be removed when we learn the tasks jointly with the help of
feedback graphs.

We can jointly learn the $m$ tasks by effectively running
Algorithm~\ref{alg:mb} in an extended MDP $\bar \Mcal$. In this
extended MDP, the state is augmented with a task index, that is,
$\bar \statespace = \statespace \times [m]$. In states with index $i$,
the rewards are drawn from $P_R^{(i)}$ and the dynamics according to
$P$ with successor states having the same task index. Formally, the
dynamics $\bar P$ and immediate expected rewards $\bar r$ of the
extended MDP is given by 
\begin{align}
  \bar P((s', j) | (s, i), a) = \one\{i = j \} P(s' | s, a), \quad \text{and, } \quad  \bar r((s, i), a) = r_i(s,a)   
\end{align}
for all $s \in \statespace$, $a \in \actionspace$, $i,j \in [m]$ where
$r_i(s,a) = \EE_{r \sim P_R^{(i)}(s,a)}[r]$ are the expected immediate
rewards of task $i$.  Essentially, the extended MDP consists of $m$
disjoint copies of the original MDP, each with the rewards of the
respective task. Tabular RL without feedback graphs would also take as
many episodes as Equation~\eqref{eqn:separate_rl_sc} to learn an
$\epsilon$-optimal policy in this extended MDP for all tasks (e.g.,
when task index is drawn uniformly before each episode).

The key for joint learning is to define the feedback graph $\bar G$ so
that it connects all copies of state-action pairs that are connected
in the feedback graph $G$ of the original MDP. That is, for all $s, s'
\in \statespace$, $a, a' \in \actionspace$, $i,j \in [m]$,
\begin{align}
    ((s,i), a) \overset{\bar G}{\rightarrow} ~&((s',j), a')  \Leftrightarrow (s,a) \gedge (s', a').  
\end{align}
Note that we can simulate an episode of $\bar \Mcal$ by running the
same policy in the original MDP because we assumed that the immediate
rewards of all but one task are known.  Therefore, to run
Algorithm~\ref{alg:mb} in the extended MDP, it is only left to
determine the task index $i_k$ of each episode $k$. To ensure learning
all tasks equally fast and not wasting resources on a single task,
it is sufficient to choose the task for which the algorithm would
provide the largest certificate, i.e.,
 $i_k \in \argmax_{i \in [m]}~ \Vub_{k,1}((s_{k,1}, i)) - \Vlb_{k,1}((s_{k,1}, i))$.
This choice implies that if the certificate of the chosen task is
smaller than $\epsilon$, then the same holds for all other tasks. Thus, by
Theorem~\ref{thm:cipoc_independencenumber} above,
Algorithm~\ref{alg:mb} must output a certificate with
$\Vub_{k,1}((s_{k,1}, i_k)) - \Vlb_{k,1}((s_{k,1}, i_k)) \leq
\epsilon$ after at most
\begin{align}
\wt O\left( \frac{( 1 + \epsilon \wh \numS) \mas H^2}{\epsilon^2} \right)
\end{align}
episodes (see Corollary~\ref{cor:epspolicy} in the appendix). Note that we
used the mas-number $\mas$ and maximum number of successor states $\wh
\numS$ of the original MDP, as these quantities are identical in the
extended MDP. Since $\epsilon \geq \Vub_{k,1}((s_{k,1}, i_k)) -
\Vlb_{k,1}((s_{k,1}, i_k)) \geq \Vub_{k,1}((s_{k,1}, j)) -
\Vlb_{k,1}((s_{k,1}, j)) \geq V^\star((s_{k,1}, j)) -
V^{\pi_k}_{1}((s_{k,1}, j))$ for all $j \in [m]$, the current policy
$\pi_k$ is $\epsilon$-optimal for all tasks. Hence, by learning tasks
jointly through feedback graphs, \textbf{the total number of episodes
  needed to learn a good policy for all tasks does not grow with the
  number of tasks} and we save a factor of $m$ compared to the naive
approach without feedback graphs.  This might seem to be too good to
be true but it is possible because the rewards of all but one task are
known and the dynamics is identical across tasks. Hence, additional
tasks cannot add significant statistical complexity compared to the
worst-case for a single task. While it may be possible to derive and
analyze a specialized algorithm for this setting without feedback
graphs, this would likely be much more tedious compared to this
immediate approach leveraging feedback graphs.

\section{Faster Policy Learning Using a Small Dominating Set}
\label{sec:domination}

Algorithm~\ref{alg:mb} uses side observations efficiently (and close to optimally for symmetric feedback graphs),
despite being agnostic to the feedback graph structure. Yet, sometimes, an
alternative approach can be further beneficial.  In some tasks, there
are state-action pairs which are highly informative, that is, they have
a large out-degree in the feedback graph, but yield low
return. Consider for example a ladder in the middle of a maze. Going
to this ladder and climbing it is time-consuming (low reward) but it
reveals the entire structure of the maze, thereby making a subsequent
escaping much easier. Explicitly exploiting such state-action pairs is
typically not advantageous in regret terms (worst case
$\Omega(T^{2/3})$) but that can be useful when the goal is to learn a
good policy and when the return during learning is irrelevant. We
therefore study the sample-complexity of RL in MDPs given a small
dominating set $\Xcal_D = \{X_1, \dots, X_\gamma\}$ of the feedback
graph.

We propose a simple algorithm that aims to explore the MDP by
uniformly visiting state-action pairs in the dominating set. This
works because the dominating set admits outgoing edges to every
vertex, that is $\forall x \in \Xcal, \exists x' \in \Xcal_D \colon x'
\gedge x$.  However, compared to bandits
\citep{AlonCesa-BianchiGentileMansour2013} with immediate access to
all vertices, there are additional challenges for such an approach in
MDPs:
\begin{enumerate}

\item \textbf{Unknown policy for visiting the dominating}: While we
  assume to know the identity of the state-action pairs in a
  dominating set, we do not know how to reach those pairs.
  
\item \textbf{Low probability of reaching dominating set}: Some or all
  nodes of the dominating set might be hard to reach under any policy.

\end{enumerate} 
The lower bound in Theorem~\ref{thm:domsetlowerbound} in the next
section shows that these challenges are fundamental.  To address them,
Algorithm~\ref{alg:dominatingset} proceeds in two stages. In the first
stage (lines~\ref{lin:phase1_start}--\ref{lin:phase1_end}), we learn
policies $\pi^{(i)}$ that visit each element $X_i \in \Xcal_D$ in the
dominating set with probability at least $\frac{p^{(i)}}{2}$. Here,
$p^{(i)} = \max_{\pi} \EE\left[\sum_{h=1}^H \one\{(s_h, a_h) =
  X^{(i)}\} ~|~ \pi\right]$ is the highest expected number of visits
to $X_i$ per episode possible.

The first phase leverages the construction for multi-task learning from
Section~\ref{sec:multitaskrl}. We define an extended MDP for a set of
tasks $0, 1, \dots, \gamma$. While task $0$ is to maximize the
original reward, tasks $1, \dots, \gamma$ aim to maximize the number
of visits to each element of the dominating set. We therefore define
the rewards for each task of the extended MDP as
\begin{align}
\bar r((s, 0), a) & = r(s,a) &
\bar r((s, k), a)
& = \one\{ (s,a) = X_k\}, \quad \forall k \in [\gamma], s \in \statespace, a \in \actionspace.
\end{align}
The only difference with Section~\ref{sec:multitaskrl} is that we
consider a subset of the tasks and stop playing a task once we have
identified a sufficiently good policy for it.  The stopping condition in
Line~\ref{lin:stop_phase1} ensures that policy $\pi^{(i)}$ visits
$X_i$ in expectation at least $\wh p^{(i)} \geq \frac{p^{(i)}}{2}$
times.  In the second phase of the algorithm
(lines~\ref{lin:phase2_start}--\ref{lin:phase2_end}), each policy
$\pi^{(i)}$ is played until there are enough samples per state-action
pair to identify an $\epsilon$-optimal policy.

\begin{restatable}[Sample-Complexity of Algorithm~\ref{alg:dominatingset}]{theorem}{domsetalgsc}
For any tabular episodic MDP with state-actions $\Xcal$, horizon $H$,
feedback graph with mas-number $\mas$ and given dominating set $\Xcal_D$
with $|\Xcal_D| = \gamma$ and accuracy parameter $\epsilon > 0$,
Algorithm~\ref{alg:dominatingset} returns with probability at least $1
- \delta$ an $\epsilon$-optimal policy after
\begin{align}
 O\left(\left(\frac{\gamma H^3}{p_0\epsilon^2}
+ \frac{\gamma \wh \numS H^3}{p_0\epsilon} + \frac{\mas \wh \numS H^2}{p_0}\right)	 
 \ln^3 \frac{|\Xcal| H}{\epsilon\delta}\right)\label{eqn:sc_ds}
\end{align}
episodes. Here, $p_0 = \min_{i \in [\gamma]} p^{(i)}$ is expected number of visits to the node in the dominating set that is hardest to reach.
\label{thm:samplecomplexity_domset}
\end{restatable}
The proof of Theorem~\ref{thm:samplecomplexity_domset} builds on the
feedback graph techniques for Algorithm~\ref{alg:mb} and the arguments
in Section~\ref{sec:multitaskrl}. These arguments alone would yield an additional
$\frac{\mas H^2}{p_0^2}$ term, but we show that is can be avoided through a more
refined (and to the best of our knowledge, novel) argument in Appendix~\ref{app:domsetproofs}.

The last term $\frac{M \wh \numS H^2}{p_0}$ is spent in the first
phase on learning how to reach the dominating set. The first two terms
come from visiting the dominating set uniformly in the second phase. Comparing that to
Corollary~\ref{cor:epspolicy} for Algorithm~\ref{alg:mb}, $\mas$ is
replaced by $\frac{\gamma H}{p_0}$ in
$\operatorname{poly}(\epsilon^{-1})$ terms. This can yield substantial
savings when a small and easily accessible dominating set is known,
e.g., when $\gamma \ll \frac{\mas p_0}{H}$ and $\epsilon \ll p_0$. There
is a gap between the bound above and the lower bound in
Theorem~\ref{thm:domsetlowerbound}, but one can show that a slightly
specialized version of the algorithm reduces this gap to
$H$ in the class of MDPs of the lower bound (by using that $p_0 \leq
1, \hat \numS = 2$ in this class, see Appendix~\ref{app:domsetproofs}
for details).

\textbf{Extension to Unknown Dominating Sets:} Since we pay only a logarithmic price
for the number of tasks attempted to be learned in the first phase,
we can modify the algorithm to attempt to learn policies to reach all
$\numS$ states (and thus all $\Xcal$) and stop the phase when an appropriate dominating set is found.

\begin{algorithm}[t]
\SetKwInOut{Inputa}{input}
\SetKwInOut{Outputa}{ouput}
\SetKwInOut{Returna}{return}
\SetAlgoLined
\SetNoFillComment
\Inputa{failure tolerance $\delta \in (0,1]$, desired accuracy $\epsilon > 0$}
\Inputa{dominating set $\Xcal_D = \{X_1, \dots, X_{\gamma}\}$, maximum transition support $\wh \numS \leq \|P(x)\|_0 \leq \numS$}
Initialize 
$n(s,a) \gets 0, ~ 
\wh r(s,a) \gets 0, ~  
\rhatsq(s,a) \gets 0, ~  
\wh P(s,a) \gets e_1$ for all 
$s \in \bar \statespace$ and  $a \in \actionspace$\;
Set $\Ical \gets \{1, \dots, \gamma\}$\tcp*{index set of active tasks}  
\tcc{First phase: find policy to reach each vertex in given dominating set}
\While{$\Ical \neq \varnothing$}{
\label{lin:phase1_start}
    
       $\pi, \Vub_h, \Vlb_h \gets$ \FOptPlan{$n, \wh r, \rhatsq, \wh P$}
       \tcp*{Alg.~\ref{alg:optimistic_planning}, with probability parameter $\delta / 2$}
        $j \gets \argmax_{i \in \Ical} \Vub_{1}((s_{1}, i)) - \Vlb_{1}((s_{1}, i))$\;
       \For{$i \in \Ical$}{
        \If{$\Vub_1((s_1,i)) \leq 2 \Vlb_1((s_1,i))$ \label{lin:stop_phase1}}
        {
            
            $\pi^{(i)}((s, 0),h) \gets \pi((s, i),h) \qquad \forall s \in \statespace, h \in [H]$\tcp*{map policy to task 0}

            $\wh p^{(i)} \gets \Vlb_1((s_1,i))$\;
            $\Ical \gets \Ical \setminus \{i\}$\;
        }
        }
        
        $n, \wh r, \rhatsq, \wh P \gets$ \FSample{$\pi, (s_1, j), n, \wh r, \rhatsq, \wh P$} 
        \label{lin:phase1_end}
        \tcp*{from Alg.~\ref{alg:mb}, apply to extended MDP $\bar \Mcal$}

}
\tcc{Second phase: play learned policies to uniformly sample from dominating set}

\While{$\Vub_1((s_1,0)) - \Vlb_1((s_1, 0)) > \epsilon$\label{lin:phase2_start}}{

$j \gets (j \operatorname{mod} \gamma) + 1$ \tcp*[l]{Choose policy in circular order}
$n, \wh r, \rhatsq, \wh P \gets $\FSample{$\pi^{(j)}, (s_1, 0), n, \wh r, \rhatsq,  \wh P$}
\tcp*{from Alg.~\ref{alg:mb}, $\!\!$ apply to extended MDP $\bar \Mcal\!\!\!$}
$\pi, \Vub_h, \Vlb_h \gets$ \FOptPlan{$n, \wh r, \rhatsq, \wh P$}
 \tcp*{Alg.~\ref{alg:optimistic_planning}, with probability parameter $\delta / 2$}
}
$\hat \pi(s, h) \gets \pi((s, 0), h) \quad \forall s \in \statespace, h \in [H]$\tcp*[l]{map policy back to original MDP}
\Returna{$\hat \pi$}
\label{lin:phase2_end}
\caption{RL Using Dominating Set} 
\label{alg:dominatingset}
\end{algorithm}

\section{Statistical Lower Bounds}
\label{sec:lowerbounds}

RL in MDPs with feedback graphs is statistically easier due to side  observations compared to RL without feedback graphs. Thus, existing lower bounds are not applicable. We now present a new lower-bound that shows that for any given feedback graph, the worst-case expected regret of any learning algorithm has to scale with the size of the largest independent set of at least half of the feedback graph.
\begin{restatable}{theorem}{lowerboundindep}
	Let $\numA, N,  H, T \in \NN$ and $G_1, G_2$ be two graphs with $N\numA$ and $(N+1)\numA$ (disjoint) nodes each. 
	If $H \geq 2 + 2\log_{\numA} N$, then there exists a class of episodic MDPs with $2N+1$ states, $\numA$ actions, horizon $H$ and feedback graph $G_1 \cup G_2 := (V(G_1) \cup V(G_2), E(G_1) \cup E(G_2))$ such that the worst-case expected regret of any algorithm after $T$ episodes is at least
	    $
	    \frac{1}{46}\sqrt{\alpha H^2 T}
	    $
     when $T \geq \alpha^3 / \sqrt{2}$ and $\alpha \geq 2$ is the independence number of $G_1$. 
     \label{thm:indeplowerbound}
\end{restatable}
The states in the class of MDPs in this lower bound form a deterministic tree with degree $\numA$ (bottom half of Figure~\ref{fig:domsetowerbound}). $G_1$ is the feedback graph for the state-action pairs at the leaves of this tree. They transition with slightly different probabilities to terminal states with high or low reward. Following \citet{mannor2011bandits}, we show that learning in such MDPs cannot be much easier than learning in $\alpha$-armed bandits where rewards are scaled by $H$. The same construction can be used to show a lower sample-complexity bound of order $\frac{\alpha H^2}{\epsilon^2} \ln \frac{1}{\delta}$ for learning $\epsilon$-optimal policies with probability at least $1 - \delta$.
This regret lower bound shows that, up to a scaling of rewards of order $H$, the statistical difficulty is comparable to the bandit case where the regret lower-bound is $\sqrt{\alpha T}$ \citep{mannor2011bandits}.

The situation is different when we consider lower bounds in terms of domination number. Theorem~\ref{thm:domsetlowerbound} below proves that there is a fundamental difference between the two settings:
\begin{theorem}\label{thm:domsetlowerbound}
    Let $\gamma \in \NN$ and $p_0 \in (0, 1]$ and  $H, \numS, \numA \in \NN$ with $H \geq 2 \log_{\numA}(\numS / 4)$.  There exists a family of MDPs with horizon $H$ and a feedback graph with a dominating set of size $\gamma$ and independence set of size $\alpha = \Theta(\numS \numA)$. The dominating set can be reached uniformly with probability $p_0$. Any algorithm that returns an $\epsilon$-optimal policy in this family with probability at least $1 - \delta$ has to collect the following expected number of episodes in the worst case
    \begin{align}
        \Omega\left( \frac{\alpha H^2}{\epsilon^2}\ln\frac 1 \delta \wedge \left(\frac{\gamma H^2}{p_0 \epsilon^2}\ln\frac 1 \delta + \frac{\alpha}{p_0}\right)\right).
    \end{align}
\end{theorem}

\begin{wrapfigure}{r}{0.4\textwidth}
    \centering
    \includegraphics[width=\linewidth]{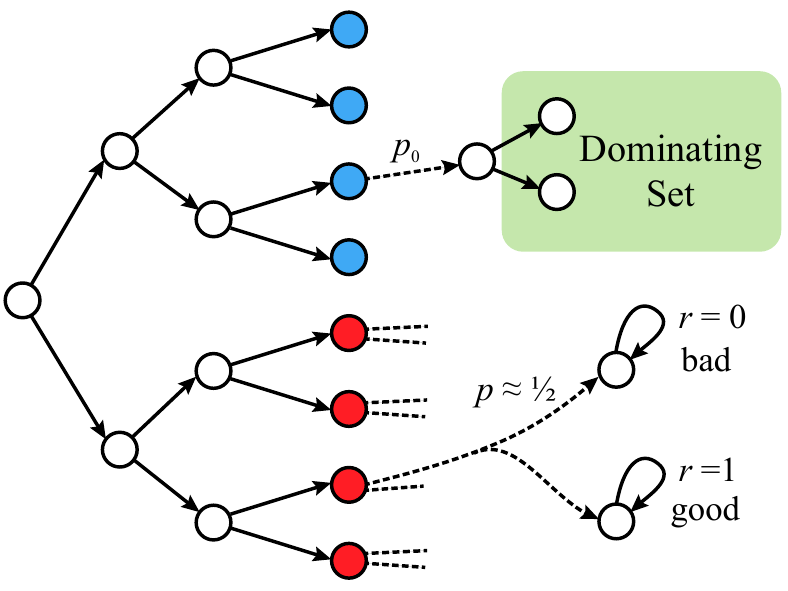}
    \caption{Difficult class of MDPs with a feedback graph and small dominating set. 
    Omitted transitions point to the bad state.}
    \label{fig:domsetowerbound}
\end{wrapfigure}

This lower bound depends on the probability $p_0$ with which the dominating set can be reached and has a dependency $\frac{\mas}{p_0} \approx \frac{\numS \numA}{p_0}$ on the number of states and actions. In bandits, one can easily avoid the linear dependency on number of arms by uniformly playing all actions in the given dominating set $\tilde \Theta(\epsilon^{-2})$ times. 
We illustrate where the difficulty in MDPs comes from in Figure~\ref{thm:domsetlowerbound}. 
States are arranged in a tree so that each state at the leafs can be reached by one action sequence. 

The lower half of state-action pairs at the leafs (red) transition to good or bad terminal states with similar probability. This mimics a bandit with $\Theta(\numS\numA)$ arms. There are no side observations available except in state-action pairs of the dominating set (shaded area). Each of them can be reached by a specific action sequence but only with probability $p_0$, otherwise the agent ends up in the bad state. 

To identify which arm is optimal in the lower bandit, the agent needs to observe $\Omega(H^2 / \epsilon^2)$ samples for each arm. It can either directly play all $\Theta(\numS \numA)$ arms or learn about them by visiting the dominating set uniformly. To visit the dominating set once takes $1/p_0$ attempts on average if the agent plays the right action sequence. However, the agent does not know which state-action at the leaf of the tree (blue states) can lead to the dominating set and therefore has to try each of the $\Theta(\numS \numA)$ options on average $1/p_0$ times to identify it.

\section{Related Work}
To the best of our knowledge, we are the first to study RL with feedback graphs in MDPs. In the bandit setting, there is a large body of works on feedback graphs going back to \citet{mannor2011bandits}. 
In stochastic bandits, \citet{CaronKvetonLelargeBhagat2012} provided the first regret bound for UCB in terms of clique covering number which was improved by \citet{lykouris2019graph} to mas-number.\footnote{They assume symmetric feedback graphs and state their results in terms of independence number.} 
Both are gap-dependent bounds as is common in bandits. \citet{simchowitz2019non} recently proved the first gap-dependent bounds in MDPs for an algorithm similar to Algorithm~\ref{alg:mb} without graph feedback. To keep the analysis and discussion to the point, we here provided worst-case problem-independent bounds but we assume that a slight generalization of our technical results could be used to prove similar problem-dependent bounds. 

While mas-number is the best-known dependency for UCB-style algorithms, \citet{CohenHazanKoren2016} achieved $\sqrt{\alpha T}$ regret with an elimination algorithm that uniformly visits independence sets in each round. Instead, \citet{AlonCesa-BianchiGentileMansour2013} explicitly leveraged a dominating set for $\sqrt{\alpha T}$ regret. Finally, \citet{BuccapatnamEryilmazShroff2014} also relies on the existence of a small dominating set to achieve problem-dependent regret scaling with $\gamma$. Unfortunately, all these techniques rely on immediate access to each node in the feedback graph which is unavailable in MDPs.

Albeit designed for different purposes, the first phase of Algorithm~\ref{alg:dominatingset} is similar to a concurrently developed algorithm \citep{jin2020reward} for exploration in absence of rewards. But there is a key technical difference: Algorithm~\ref{alg:dominatingset} learns how to reach each element of the dominating set jointly, while the approach by \citet{jin2020reward} learns how to reach each state-action pair separately. Following the discussion in Section~\ref{sec:multitaskrl}, we hypothesize that by applying our technique to their setting, one could reduce the state-space dependency in the $\epsilon^{-1}$ term of their sample complexity bound from $\numS^4 / \epsilon$ to $\numS^3 / \epsilon$.
\section{Conclusion}
We studied the effect of data augmentation in the form of side observations governed by a feedback graph on the sample-complexity of RL. Our results show that optimistic model-based algorithms achieve minimax-optimal regret up to lower-order terms in symmetric feedback graphs by just incorporating all available observations. We also proved that exploiting the feedback graph structure by visiting highly informative state-action pairs (dominating set) is fundamentally more difficult in MDPs compared to the well-studied bandit setting. As RL with feedback graph in MDPs captures existing settings such as learning with state abstractions and learning with auxiliary tasks, our work paves the way for a more extensive study of this setting. Promising directions include a regret analysis for feedback graphs in combination with function approximation motivated by impressive empirical successes \citep{lin2019towards, kostrikov2020image, laskin2020reinforcement}. Another question of interest is an analysis of model-free methods~\citep{jin2018q} with graph feedback which likely requires a very different analysis, as existing proofs hinge on observations arriving in trajectories.

\paragraph{Acknowledgements.}
The work of MM was partly supported by NSF CCF-1535987, NSF IIS-1618662, and a Google Research Award. KS would like to acknowledge NSF CAREER Award 1750575 and Sloan Research Fellowship. The work of YM was partly
supported by a grant of the Israel Science Foundation (ISF).

\bibliographystyle{abbrvnat}
\bibliography{cdann_thesis_lib, manual}

\begin{thebibliography}{42}
\providecommand{\natexlab}[1]{#1}
\providecommand{\url}[1]{\texttt{#1}}
\expandafter\ifx\csname urlstyle\endcsname\relax
  \providecommand{\doi}[1]{doi: #1}\else
  \providecommand{\doi}{doi: \begingroup \urlstyle{rm}\Url}\fi

\bibitem[Alon et~al.(2013)Alon, Cesa-Bianchi, Gentile, and
  Mansour]{AlonCesa-BianchiGentileMansour2013}
N.~Alon, N.~Cesa-Bianchi, C.~Gentile, and Y.~Mansour.
\newblock From bandits to experts: A tale of domination and independence.
\newblock In \emph{Advances in Neural Information Processing Systems}, pages
  1610--1618, 2013.

\bibitem[Arora et~al.(2019)Arora, Marinov, and Mohri]{AroraMarinovMohri2019}
R.~Arora, T.~V. Marinov, and M.~Mohri.
\newblock Bandits with feedback graphs and switching costs.
\newblock In \emph{Advances in Neural Information Processing Systems 32: Annual
  Conference on Neural Information Processing Systems 2019, NeurIPS 2019, 8-14
  December 2019, Vancouver, BC, Canada}, pages 10397--10407, 2019.

\bibitem[Azar et~al.(2012)Azar, Munos, and Kappen]{azar2012sample}
M.~G. Azar, R.~Munos, and H.~J. Kappen.
\newblock On the sample complexity of reinforcement learning with a generative
  model.
\newblock In \emph{Proceedings of the 29th International Coference on
  International Conference on Machine Learning}, pages 1707--1714. Omnipress,
  2012.

\bibitem[Azar et~al.(2017)Azar, Osband, and Munos]{azar2017minimax}
M.~G. Azar, I.~Osband, and R.~Munos.
\newblock Minimax regret bounds for reinforcement learning.
\newblock In \emph{International Conference on Machine Learning}, pages
  263--272, 2017.

\bibitem[Boutilier et~al.(1999)Boutilier, Dean, and
  Hanks]{boutilier1999decision}
C.~Boutilier, T.~Dean, and S.~Hanks.
\newblock Decision-theoretic planning: Structural assumptions and computational
  leverage.
\newblock \emph{Journal of Artificial Intelligence Research}, 11:\penalty0
  1--94, 1999.

\bibitem[Buccapatnam et~al.(2014)Buccapatnam, Eryilmaz, and
  Shroff]{BuccapatnamEryilmazShroff2014}
S.~Buccapatnam, A.~Eryilmaz, and N.~B. Shroff.
\newblock Stochastic bandits with side observations on networks.
\newblock In \emph{The 2014 ACM International Conference on Measurement and
  Modeling of Computer Systems}, SIGMETRICS '14, pages 289--300. ACM, 2014.

\bibitem[Caron et~al.(2012)Caron, Kveton, Lelarge, and
  Bhagat]{CaronKvetonLelargeBhagat2012}
S.~Caron, B.~Kveton, M.~Lelarge, and S.~Bhagat.
\newblock Leveraging side observations in stochastic bandits.
\newblock In \emph{UAI}, 2012.

\bibitem[Cohen et~al.(2016)Cohen, Hazan, and Koren]{CohenHazanKoren2016}
A.~Cohen, T.~Hazan, and T.~Koren.
\newblock Online learning with feedback graphs without the graphs.
\newblock In \emph{International Conference on Machine Learning}, pages
  811--819, 2016.

\bibitem[Cortes et~al.(2018)Cortes, DeSalvo, Gentile, Mohri, and
  Yang]{CortesDeSalvoGentileMohriYang2018}
C.~Cortes, G.~DeSalvo, C.~Gentile, M.~Mohri, and S.~Yang.
\newblock Online learning with abstention.
\newblock In \emph{35th ICML}, 2018.

\bibitem[Cortes et~al.(2019)Cortes, DeSalvo, Gentile, Mohri, and
  Yang]{CortesDeSalvoGentileMohriYang2019}
C.~Cortes, G.~DeSalvo, C.~Gentile, M.~Mohri, and S.~Yang.
\newblock Online learning with sleeping experts and feedback graphs.
\newblock In \emph{Proceedings of {ICML}}, pages 1370--1378, 2019.

\bibitem[Dann(2019)]{dann2019strategic}
C.~Dann.
\newblock \emph{Strategic Exploration in Reinforcement Learning - New
  Algorithms and Learning Guarantees}.
\newblock PhD thesis, Carnegie Mellon University, 2019.

\bibitem[Dann and Brunskill(2015)]{dann2015sample}
C.~Dann and E.~Brunskill.
\newblock Sample complexity of episodic fixed-horizon reinforcement learning.
\newblock In \emph{Advances in Neural Information Processing Systems}, pages
  2818--2826, 2015.

\bibitem[Dann et~al.(2017)Dann, Lattimore, and Brunskill]{dann2017unifying}
C.~Dann, T.~Lattimore, and E.~Brunskill.
\newblock Unifying {PAC} and regret: Uniform pac bounds for episodic
  reinforcement learning.
\newblock In \emph{Advances in Neural Information Processing Systems}, pages
  5713--5723, 2017.

\bibitem[Dann et~al.(2018)Dann, Jiang, Krishnamurthy, Agarwal, Langford, and
  Schapire]{dann2018oracle}
C.~Dann, N.~Jiang, A.~Krishnamurthy, A.~Agarwal, J.~Langford, and R.~E.
  Schapire.
\newblock On oracle-efficient {PAC} reinforcement learning with rich
  observations.
\newblock \emph{arXiv preprint arXiv:1803.00606}, 2018.

\bibitem[Dann et~al.(2019)Dann, Li, Wei, and Brunskill]{dann2019policy}
C.~Dann, L.~Li, W.~Wei, and E.~Brunskill.
\newblock Policy certificates: Towards accountable reinforcement learning.
\newblock \emph{International Conference on Machine Learning}, 2019.

\bibitem[Dong et~al.(2019)Dong, Van~Roy, and Zhou]{dong2019provably}
S.~Dong, B.~Van~Roy, and Z.~Zhou.
\newblock Provably efficient reinforcement learning with aggregated states.
\newblock \emph{arXiv preprint arXiv:1912.06366}, 2019.

\bibitem[Du et~al.(2019)Du, Krishnamurthy, Jiang, Agarwal, Dudik, and
  Langford]{du2019provably}
S.~Du, A.~Krishnamurthy, N.~Jiang, A.~Agarwal, M.~Dudik, and J.~Langford.
\newblock Provably efficient rl with rich observations via latent state
  decoding.
\newblock In \emph{International Conference on Machine Learning}, pages
  1665--1674, 2019.

\bibitem[Goel et~al.(2017)Goel, Dann, and Brunskill]{goel2017sample}
K.~Goel, C.~Dann, and E.~Brunskill.
\newblock Sample efficient policy search for optimal stopping domains.
\newblock In \emph{Proceedings of the 26th International Joint Conference on
  Artificial Intelligence}, pages 1711--1717. AAAI Press, 2017.

\bibitem[Henderson et~al.(2017)Henderson, Islam, Bachman, Pineau, Precup, and
  Meger]{henderson2017deep}
P.~Henderson, R.~Islam, P.~Bachman, J.~Pineau, D.~Precup, and D.~Meger.
\newblock Deep reinforcement learning that matters.
\newblock \emph{arXiv preprint arXiv:1709.06560}, 2017.

\bibitem[Howard et~al.(2018)Howard, Ramdas, Mc~Auliffe, and
  Sekhon]{howard2018uniform}
S.~R. Howard, A.~Ramdas, J.~Mc~Auliffe, and J.~Sekhon.
\newblock Uniform, nonparametric, non-asymptotic confidence sequences.
\newblock \emph{arXiv preprint arXiv:1810.08240}, 2018.

\bibitem[Jiang et~al.(2017)Jiang, Krishnamurthy, Agarwal, Langford, and
  Schapire]{jiang2017contextual}
N.~Jiang, A.~Krishnamurthy, A.~Agarwal, J.~Langford, and R.~E. Schapire.
\newblock Contextual decision processes with low bellman rank are
  pac-learnable.
\newblock In \emph{International Conference on Machine Learning}, pages
  1704--1713, 2017.

\bibitem[Jin et~al.(2018)Jin, Allen-Zhu, Bubeck, and Jordan]{jin2018q}
C.~Jin, Z.~Allen-Zhu, S.~Bubeck, and M.~I. Jordan.
\newblock Is {Q}-learning provably efficient?
\newblock \emph{arXiv preprint arXiv:1807.03765}, 2018.

\bibitem[Jin et~al.(2019)Jin, Yang, Wang, and Jordan]{jin2019provably}
C.~Jin, Z.~Yang, Z.~Wang, and M.~I. Jordan.
\newblock Provably efficient reinforcement learning with linear function
  approximation.
\newblock \emph{arXiv preprint arXiv:1907.05388}, 2019.

\bibitem[Jin et~al.(2020)Jin, Krishnamurthy, Simchowitz, and Yu]{jin2020reward}
C.~Jin, A.~Krishnamurthy, M.~Simchowitz, and T.~Yu.
\newblock Reward-free exploration for reinforcement learning.
\newblock \emph{arXiv preprint arXiv:2002.02794}, 2020.

\bibitem[Koc{\'a}k et~al.(2016)Koc{\'a}k, Neu, and Valko]{kocak2016online}
T.~Koc{\'a}k, G.~Neu, and M.~Valko.
\newblock Online learning with noisy side observations.
\newblock In \emph{AISTATS}, pages 1186--1194, 2016.

\bibitem[Kostrikov et~al.(2020)Kostrikov, Yarats, and
  Fergus]{kostrikov2020image}
I.~Kostrikov, D.~Yarats, and R.~Fergus.
\newblock Image augmentation is all you need: Regularizing deep reinforcement
  learning from pixels.
\newblock \emph{arXiv preprint arXiv:2004.13649}, 2020.

\bibitem[Krizhevsky et~al.(2012)Krizhevsky, Sutskever, and
  Hinton]{krizhevsky2012imagenet}
A.~Krizhevsky, I.~Sutskever, and G.~E. Hinton.
\newblock Imagenet classification with deep convolutional neural networks.
\newblock In \emph{Advances in neural information processing systems}, pages
  1097--1105, 2012.

\bibitem[Laskin et~al.(2020)Laskin, Lee, Stooke, Pinto, Abbeel, and
  Srinivas]{laskin2020reinforcement}
M.~Laskin, K.~Lee, A.~Stooke, L.~Pinto, P.~Abbeel, and A.~Srinivas.
\newblock Reinforcement learning with augmented data.
\newblock \emph{arXiv preprint arXiv:2004.14990}, 2020.

\bibitem[Lattimore and Czepesvari(2018)]{lattimore2018bandit}
T.~Lattimore and C.~Czepesvari.
\newblock \emph{Bandit Algorithms}.
\newblock Cambridge University Press, 2018.

\bibitem[Lattimore and Hutter(2012)]{lattimore2012pac}
T.~Lattimore and M.~Hutter.
\newblock {PAC} bounds for discounted {MDP}s.
\newblock In \emph{International Conference on Algorithmic Learning Theory},
  pages 320--334. Springer, 2012.

\bibitem[Lin et~al.(2019)Lin, Huang, Zimmer, Rojas, and Weng]{lin2019towards}
Y.~Lin, J.~Huang, M.~Zimmer, J.~Rojas, and P.~Weng.
\newblock Towards more sample efficiency in reinforcement learning with data
  augmentation.
\newblock \emph{arXiv preprint arXiv:1910.09959}, 2019.

\bibitem[Lykouris et~al.(2019)Lykouris, Tardos, and Wali]{lykouris2019graph}
T.~Lykouris, E.~Tardos, and D.~Wali.
\newblock Graph regret bounds for {T}hompson sampling and {UCB}.
\newblock \emph{arXiv preprint arXiv:1905.09898}, 2019.

\bibitem[Mannor and Shamir(2011)]{mannor2011bandits}
S.~Mannor and O.~Shamir.
\newblock From bandits to experts: On the value of side-observations.
\newblock In \emph{Advances in Neural Information Processing Systems}, pages
  684--692, 2011.

\bibitem[Mannor and Tsitsiklis(2004)]{mannor2004sample}
S.~Mannor and J.~N. Tsitsiklis.
\newblock The sample complexity of exploration in the multi-armed bandit
  problem.
\newblock \emph{Journal of Machine Learning Research}, 5\penalty0
  (Jun):\penalty0 623--648, 2004.

\bibitem[Maurer and Pontil(2009)]{maurer2009empirical}
A.~Maurer and M.~Pontil.
\newblock Empirical {B}ernstein bounds and sample variance penalization.
\newblock \emph{arXiv preprint arXiv:0907.3740}, 2009.

\bibitem[Mnih et~al.(2015)Mnih, Kavukcuoglu, Silver, Rusu, Veness, Bellemare,
  Graves, Riedmiller, Fidjeland, Ostrovski, et~al.]{mnih2015human}
V.~Mnih, K.~Kavukcuoglu, D.~Silver, A.~A. Rusu, J.~Veness, M.~G. Bellemare,
  A.~Graves, M.~Riedmiller, A.~K. Fidjeland, G.~Ostrovski, et~al.
\newblock Human-level control through deep reinforcement learning.
\newblock \emph{Nature}, 518\penalty0 (7540):\penalty0 529, 2015.

\bibitem[Osband and Van~Roy(2016)]{osband2016lower}
I.~Osband and B.~Van~Roy.
\newblock On lower bounds for regret in reinforcement learning.
\newblock \emph{arXiv preprint arXiv:1608.02732}, 2016.

\bibitem[Ross et~al.(2011)Ross, Gordon, and Bagnell]{ross2011reduction}
S.~Ross, G.~J. Gordon, and J.~A. Bagnell.
\newblock A reduction of imitation learning and structured prediction to
  no-regret online learning.
\newblock In \emph{International Conference on Artificial Intelligence and
  Statistics}, 2011.

\bibitem[Silver et~al.(2017)Silver, Schrittwieser, Simonyan, Antonoglou, Huang,
  Guez, Hubert, Baker, Lai, Bolton, et~al.]{silver2017mastering}
D.~Silver, J.~Schrittwieser, K.~Simonyan, I.~Antonoglou, A.~Huang, A.~Guez,
  T.~Hubert, L.~Baker, M.~Lai, A.~Bolton, et~al.
\newblock Mastering the game of go without human knowledge.
\newblock \emph{Nature}, 550\penalty0 (7676):\penalty0 354, 2017.

\bibitem[Simchowitz and Jamieson(2019)]{simchowitz2019non}
M.~Simchowitz and K.~Jamieson.
\newblock Non-asymptotic gap-dependent regret bounds for tabular {MDP}s.
\newblock \emph{arXiv preprint arXiv:1905.03814}, 2019.

\bibitem[Zanette and Brunskill(2019)]{zanette2019tighter}
A.~Zanette and E.~Brunskill.
\newblock Tighter problem-dependent regret bounds in reinforcement learning
  without domain knowledge using value function bounds.
\newblock \emph{https://arxiv.org/abs/1901.00210}, 2019.

\bibitem[Zhang et~al.(2017)Zhang, Cisse, Dauphin, and
  Lopez-Paz]{zhang2017mixup}
H.~Zhang, M.~Cisse, Y.~N. Dauphin, and D.~Lopez-Paz.
\newblock mixup: Beyond empirical risk minimization.
\newblock \emph{arXiv preprint arXiv:1710.09412}, 2017.

\end{thebibliography}
\clearpage

\appendix
\renewcommand{\contentsname}{Contents of Appendix}
\tableofcontents
\addtocontents{toc}{\protect\setcounter{tocdepth}{3}}
\clearpage

\section{Discussion of Graph Properties}
\label{app:graphprop_discussion}

In this section, we provide an extended discussion of the relevant graph properties that govern learning efficiency of RL with feedback graphs.
For convenience, we repeat the definitions of the properties from Section~\ref{sec:feedbackgraphs}.
\begin{itemize}

    \item \textbf{Mas-number $\bm \mas$}: A set of vertices $\Vcal \subseteq
      \Xcal$ form an acyclic subgraph if the subgraph $(\Vcal, \{
      (v,w) \subseteq \Vcal \times \Vcal \colon v \gedge w
      \})$ of $G$ restricted to $\Vcal$ is loop-free. We call the size
      of the maximum acyclic subgraph the \emph{mas-number} $\mas$ of $G$.

    \item \textbf{Independence number $\bm \alpha$}: A set of vertices
      $\Vcal \subseteq \Xcal$ is an independent set if there is no
      edge between any two nodes of that set:
      $\forall v,w \in \Vcal \colon v \, \not \gedge w$. The
      size of the largest independent set is called the
      \emph{independence number} $\alpha$ of $G$.

    \item \textbf{Domination number $\bm \gamma$}: A set of vertices
      $\Vcal \subseteq \Xcal$ form a dominating set if there is an edge
      from a vertex in $\Vcal$ to any vertex in $G$:
      $\forall x \in \Xcal \; \exists v \in \Vcal \colon v
      \gedge x$. The size of the smallest dominating set is
      called the \emph{domination number} $\gamma$.
      
      \item \textbf{Clique covering number $\bm \Ccal$}: A set of vertices $\Vcal \subseteq \Xcal$ is a clique if there it is a fully-connected subgraph, i.e., for any $x,y \in \Vcal \colon x \gedge y$. A set of such cliques $\{\Vcal_1, \dots \Vcal_n\}$ is called a clique cover if every node is included in at least one of the cliques, i.e., $\Xcal = \bigcup_{i=1}^n \Vcal_i$. The size of the smallest clique cover is called the \emph{clique covering number} $\Ccal$.
\end{itemize}
In addition to the properties appearing our bounds, we here include the clique covering number $\Ccal$ which has been used earlier analyses of UCB algorithms in bandits \citep{CaronKvetonLelargeBhagat2012}. One can show that in any graph, the following relation holds
\begin{align}
    |\Xcal| \geq \Ccal \geq \mas \geq \alpha \geq \gamma.
\end{align}
For example, $\Ccal \geq \mas$ follows from the fact that no two vertices that form a clique can be part of an acyclic subgraph and thus no acyclic subgraph can be larger than any clique cover. An important class of feedback graphs are symmetric feedback graphs where for each edge $x \gedge y$, there is a back edge $y \gedge x$. In fact, many analyses in the bandit settings assume undirected feedback graphs which is equivalent to symmetric directed graphs. For symmetric feedback graphs, the independence number and mas-number match, i.e., 
\begin{align}
    \alpha = \mas.
\end{align}   
This is true because acyclic subgraphs of symmetric graphs cannot contain any edges, otherwise the back edge would immediately create a loop. Thus any acyclic subgraph is also an independent set and $\mas \geq \alpha$.

\paragraph{Examples:} 
We now discuss the value of the graph properties in feedback graphs by example (see Figure~\ref{fig:example_graph_properties}). The graph in Figure~\ref{fig:clique_graph} consists of two disconnected cliques and thus the clique covering number and the domination number is $2$. While the total number of nodes can be much larger -- 8 in this example -- all graph properties equal the number of cliques in such a graph. In practice, feedback graphs that consists of disconnected cliques occur for example in state abstractions where all $(s,a)$ pairs with matching action and where the state belongs to the same abstract state form a clique. They are examples for a simple structure that can be easily exploited by RL with feedback graphs to substantially reduce the regret.

In the feedback graph in Figure~\ref{fig:ordered_graph}, the vertices are ordered and every vertex is connected to every vertex to the left. This graph is acyclic and hence $\mas$ coincides with the number of vertices but the independence number is $1$ as the graph is a clique if we ignore the direction of edges (and thus each independence set can only contain a single node).
A concrete example where feedback graphs can exhibit such structure is in tutoring systems where the actions represent the number of practice problems to present to a student in a certain lesson. The oracle can fill in the outcomes (how well the performed on each problem) for all actions that are would have given fewer problems than the chosen action.

Figure~\ref{fig:star_graph} shows a star-shaped feedback graph. Here, the center vertex reveals information about all other vertices and thus is a dominating set with size $\gamma = 1$. At the same time, the largest independence set are the tips of the star which is much larger. This is an example where approaches such as Algorithm~\ref{alg:dominatingset} that leverage a dominating set can learn a good policy with much fewer samples as compared to others that only rely independence sets.

The examples in Figure~\ref{fig:clique_graph}--\ref{fig:star_graph} exhibit structured graphs, but it is important to realize that our results do not rely a specific structure. They can work with any feedback graph and we expect that feedback graphs in practice are not necessarily structured. Figure~\ref{fig:general_graph} shows a generic graph where all relevant graph properties are distinct which highlights that even in seemingly unstructured graphs, it is important which graph property governs the learning speed of RL algorithms.

\begin{figure}
    \centering
    \begin{subfigure}[t]{.24\textwidth}\centering
    \includegraphics[width=3.8cm]{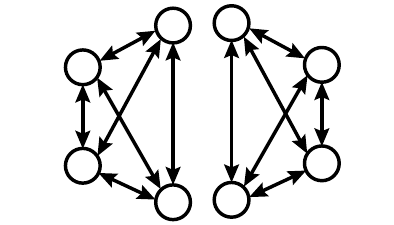}
    \caption{$\numS\numA = 8$\\\hspace*{5mm}$\Ccal = \mas =  \alpha = \gamma = 2$}
    \label{fig:clique_graph}
    \end{subfigure}
    \hfil
        \begin{subfigure}[t]{.24\textwidth}\centering
    \includegraphics[width=3.8cm]{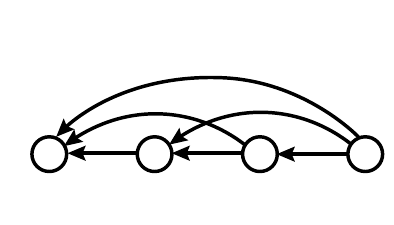}
    \caption{$\numS\numA = \Ccal = \mas = 4$\\\hspace*{5mm}$\alpha = \gamma = 1$}
    \label{fig:ordered_graph}
    \end{subfigure}
        \hfil
        \begin{subfigure}[t]{.24\textwidth}\centering
    \includegraphics[width=3.8cm]{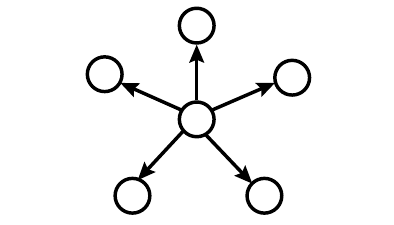}
    \caption{$\numS\numA = 6$ \\\hspace*{5mm}$\Ccal = \mas = \alpha = 5$\\\hspace*{5mm}$\gamma = 1$}
    \label{fig:star_graph}
    \end{subfigure}
        \hfil
        \begin{subfigure}[t]{.24\textwidth}
        \centering
    \includegraphics[width=3.8cm]{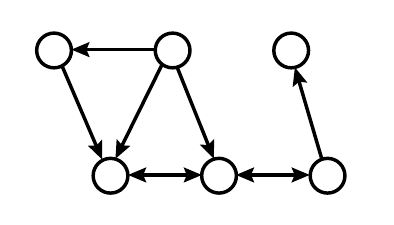}
    \caption{$\numS\numA = 6$\\\hspace*{4mm}
    $\Ccal = 5, \mas = 4$\\\hspace*{5mm}$\alpha = 3, \gamma = 2$}
    \label{fig:general_graph}
    \end{subfigure}

    \caption{Examples of feedback graphs with different vertex numbers $\numS \numA$, mas-number $\mas$, independence number $\alpha$ and domination number $\gamma$.}
    \label{fig:example_graph_properties}
\end{figure}

\section{Additional Details on Model-Based RL with Feedback Graphs}
Here, we provide additional details and extensions to Algorithm~\ref{alg:mb} in Section~\ref{sec:modelbased}.
\subsection{Optimistic Planning}
Algorithm~\ref{alg:optimistic_planning} presents the optimistic planning subroutine called by Algorithms~\ref{alg:mb} and \ref{alg:dominatingset}. 
In this procedure, the maximum value is set as $V^{\max}_h = H-h+1$ for each time step $h$ and notation $\wh P(x)f = \Ex_{s' \sim P(x)}[f(x)]$ denotes the expectation with respect to the next state distribution of any function $f \colon \statespace \rightarrow \RR$ on states.

The \FOptPlan procedure computes an optimistic estimate $\Qub$ of the optimal Q-function $Q^\star$ by dynamic programming. The policy $\pi$ is chosen greedily with respect to this upper confidence bound $\Qub$. In addition, a pessimistic estimate $\Qlb$ of the Q-function of this policy $Q^{\pi}$ is computed (lower confidence bound) analogously to $\Qub$. The two estimates only differ in the sign of the reward bonus $\psi_h$. Up to the specific form of the reward bonus $\psi_h(x)$, this procedure is identical to the policy computation in \textsc{ORLC} \citep{dann2019policy} and \textsc{Euler} \citep{zanette2019tighter}.\footnote{Note however that the lower confidence bound in \textsc{Euler} is only supposed to satisfy $\Qlb \leq Q^\star$ while we here follow the \textsc{ORLC} approach and its analysis and require $\Qlb$ to be a lower confidence bound on the Q-value of the computed policy $Q^{\pi}$.} 

\begin{algorithm}
\SetAlgoLined
\SetNoFillComment
\SetKwInOut{Inputa}{input}
\SetKwInOut{Outputa}{ouput}
\SetKwInOut{Returna}{return}
  \SetKwProg{Fn}{function}{:}{}
\Fn{\FOptPlan($n, \wh r, \rhatsq, \wh P$)}{
Set $\Vub_{H+1}(s) \gets 0; \quad  \Vlb_{H+1}(s) \gets 0 \quad \forall s \in \statespace$\;
    \For
    (\tcp*[f]{optimistic planning with upper and lower confidence bounds})
    {$h=H$ \KwTo $1$ \textbf{and} $s \in \statespace$}
    {
        
            \For{$a \in \actionspace$}
            {
            $x \gets (s,a)$\;
            \tcc{Compute reward bonus}
            $\eta \gets \sqrt{\rhatsq(x) - \wh r(x)^2} +  \sqrt{\wh P(x) (\Vub_{h+1}^2) - (\wh P(x) \Vub_{h+1})^2}$\tcp*[l]{Reward and next state variance}
            $\psi_h(x) \gets O\left(
    \frac{1}{H}\wh P(x)[\Vub_{h+1} - \Vlb_{h+1}]
    +
    \sqrt{\frac{\eta}{n(x)}\ln \frac{|\Xcal|H \ln n(x)}{\delta}} + \frac{\wh \numS H^2}{n(x)}\ln \frac{|\Xcal|H \ln n(x)}{\delta}\!\right)$
             \label{lin:psiterm}\;
            \tcc{Bellman backup of upper and lower confidence bounds}
                $\Qub_{h}(x) \gets 0 \vee \,\,( \wh r(x) 
                + \wh P(x)\Vub_{h+1} + \psi_{h}(x)) \,\,\wedge V^{\max}_h$
                \label{lin:ulcr_ub}
                \tcp*{UCB of $Q^\star_{h}$}
                $\Qlb_{h}(x) \gets 0 \vee \,\,(  \wh r(x) 
                + \wh P(x)\Vlb_{h+1} - \psi_h(x))\,\,\wedge V^{\max}_h$
                \label{lin:ulcr_lb}
                \tcp*{LCB of $Q^{\pi}_{h} \geq 0$}
             }
\tcc{Compute greedy policy of UCB}
 $\pi(s, h) \gets \argmax_{a} \Qub_{h}(s, a)$\;
 $\Vub_{h}(s) \gets \Qub_{h}(s, \pi(h))$;
  $\qquad$
 $\Vlb_{h}(s) \gets \Qlb_{h}(s, \pi(h)) $\;
}
\Returna{$\pi, \Vub_h, \Vlb_h$}
}
\label{alg:optimistic_planning}
\caption{Optimistic Planning Routine}
\end{algorithm}

\subsection{Runtime Analysis}
Just as in learning without graph feedback, the runtime of Algorithm~\ref{alg:mb} is $O(\numS \wh \numS \numA H)$ per episode where $\wh \numS$ is a bound on the maximum transition probability support ($\numS$ in the worst case).
The only difference to RL without side observations is that there are additional updates to the empirical model. However, sampling an episode and updating the empirical model requires $O(H \numS \numA)$ computation as there are $H$ time steps and each can provide at most $|\Xcal| \leq \numS \numA$ side observations. This is still dominated by the runtime of optimistic planning $O(\numS \wh \numS \numA H)$.
If the feedback graph is known ahead of time, one might be able to reduce the runtime, e.g., by maintaining only one model estimate for state-action pairs that form a clique in the feedback graph with no incoming edges. Then is suffices to only compute statistics of a single vertex per clique.

\subsection{Sample Complexity}
\label{sec:sample_complex}
Since Algorithm~\ref{alg:mb} is a minor modification of \textsc{ORLC}, it follows the IPOC framework \citep{dann2019policy} for accountable reinforcement learning.\footnote{To formally satisfy an IPOC guarantee, the algorithm has to output the policy and with a certificate before each episode. We omitted outputting of policy $\pi_k$ and certificate $[\Vlb_{k,1}(s_{k,1}), \Vub_{k,1}(s_{k,1})]$ after receiving the initial state $s_{k,1}$ in the listing of Algorithm~\ref{alg:mb} for brevity, but this can be added if readily.} 
As a result, we can build on the results for algorithms with cumulative IPOC bounds \citep[][Proposition~2]{dann2019strategic} and show that our algorithm satisfies a sample-complexity guarantee:

\begin{restatable}[PAC-style Bound]{corollary}{mbpacbound}
    For any episodic MDP with state-actions $\Xcal$, horizon $H$ and  feedback graph $G$, with probability at least $1 - \delta$ for all $\epsilon > 0$ jointly, Algorithm~\ref{alg:mb} can output a certificate with $\Vub_{k',1}(s_{k',1}) - \Vlb_{k', 1}(s_{k', 1})$ for some episode $k'$ within the first
    \begin{align}
        k' = O\left(
        	\frac{M H^2}{\epsilon^2} \ln^2\frac{H|\Xcal|}{\epsilon\delta} 
	 + \frac{M \hat \numS H^2}{\epsilon} \ln^3 \frac{H|\Xcal|}{\epsilon\delta}
        \right)
    \end{align}
    episodes. If the initial state is fixed, such a certificate identifies an $\epsilon$-optimal policy.
    \label{cor:epspolicy}
\end{restatable}
The proof of this corollary is available in Section~\ref{sec:sample_complex_proof}
\subsection{Generalization to Stochastic Feedback Graphs} 
\label{sec:stochfbg}
As presented in Section~\ref{sec:feedbackgraphs}, we assumed so far that the feedback graph $G$ is fixed and identical in all episodes. We can generalize our results and consider stochastic feedback graphs where the existence of an edge in the feedback graph in each episode is drawn independently (from other episodes and edges). 
This means the oracle provides a side observation for another state-action pair only with a certain probability.
We formalize this as the feedback graph $G_k$ in episode $k$ to be an independent sample from a fixed distribution where the probability an each edge is denoted as 
\begin{align}
    q(x,x') \defeq \prob\left(x \overset{G_k}{\rightarrow} x' \right).
\end{align} 
This model generalizes the well-studied Erdős–R\'enyi model \citep[e.g.][]{BuccapatnamEryilmazShroff2014} because different edges can have different probabilities. This can be used as a proxy for the strength of the user's prior. One could for example choose the probability of states being connected to decreases with their distance. This would encode a belief that nearby states behave similarly. 

Algorithm~\ref{alg:mb} can be directly applied to stochastic feedback graphs and as our analysis will show the bound in Theorem~\ref{thm:cipoc_independencenumber} still holds as long as the mas-number $\mas$ is replaced by 
\begin{align}
    \bar \mas = \inf_{\epsilon \in (0, 1]} \frac{\mas(G_{\geq \epsilon})}{\epsilon}
\end{align}
where $G_{\geq \nu}$ is the feedback graph that only contains an edge if its probability is at least $\nu$, i.e., $x \overset{G_{\geq \nu}}{\rightarrow} x'$ if and only if $q(x, x') \geq \nu$ for all $x, x' \in \Xcal$. The quantity $\bar \mas$ generalizes the mas-number of deterministic feedback graphs where $q$ is binary and thus $\mas = \bar \mas$.

\subsection{Generalization to Side Observations with Biases} 
\label{sec:biasedfbg}
While there are often additional observations available, they might not always have the same quality as the observation of the current transition \citep{kocak2016online}. For example in environments where we know the dynamics and rewards change smoothly (e.g. are Lipschitz-continuous), we can infer additional observations from the current transition but have error that increases with the distance to the current transition. We thus also consider the case where each feedback graph sample $(x, r, s', \epsilon')$ also comes with a bias $\epsilon' \in \RR$ and the distributions $\tilde P_R, \tilde P$ of this sample satisfy
\begin{align}
   |\EE_{r \sim \tilde P_R}[r] - \EE_{r \sim P_R(x)}[r]| \leq \epsilon'
   \quad \textrm{and} \quad \| \tilde P - P(x) \|_1 \leq  \epsilon'.
\end{align}
To allow biases in side observations, we adjust the bonuses in Line~\ref{lin:psiterm} of Algorithm~\ref{alg:mb} to 
\begin{align}
    \psi_h(x) + \tilde O\left(
    \sqrt{\frac{H \wh \epsilon(x)}{n(x)}\ln \frac{|\Xcal|H \ln n(x)}{\delta}} + H \wh \epsilon(x)\right)
\end{align}
for each state-action pair $x$
where $\wh \epsilon(x)$ is the average bound on bias in all observations of this $x$ so far. We defer the presentation of the full algorithm with these changes to the next section but first state the main result for learning with biased side observations here. The following theorem shows that the algorithm's performance degrades smoothly with the maximum encountered bias $\epsilon_{\max}$:
\begin{theorem}[Regret bound with biases] \label{thm:cipoc_stochastic}
In the same setting as Theorem~\ref{thm:cipoc_independencenumber} but where samples can have a bias of at most $\epsilon_{\max}$ , the cumulative certificate size and regret are bounded with probability at least $1 - \delta$ for all $T$ by  
\begin{align}
  O\left(\sqrt{\mas H^2 T} \ln \frac{H|\Xcal| T}{\delta} 
	 + \mas \hat \numS H^2 \ln^3 \frac{H|\Xcal| T}{\delta} + \sqrt{\mas H^3 T  \epsilon_{\max}} \ln \frac{|\Xcal| H T}{\delta} 
	 + H^2 T \epsilon_{\max} \right).
\end{align}
\end{theorem}
If $T$ is known, the algorithm can be modified to ignore all observations with bias larger than $T^{-1/2}$ and still achieve order $\sqrt{T}$ regret by effectively setting $\epsilon_{\max} = O(T^{1/2})$ (at the cost of increase in $\mas$).

\subsection{Generalized Algorithm and Main Regret Theorem}
We now introduce a slightly generalized version of Algorithm~\ref{alg:mb} that will be the basis for our theoretical analysis and all results for Algorithm~\ref{alg:mb} follow as special cases.
This algorithm, given in Algorithm~\ref{alg:mb_app} contains numerical values for all quantities -- as opposed to $O$-notation -- and differs from Algorithm~\ref{alg:mb} in the following aspects: 

\begin{enumerate}
\item \textbf{Allowing Biases: }
While Algorithm~\ref{alg:mb} assumes that the observations provided by the feedback graph are unbiased, Algorithm~\ref{alg:mb_app} allows biased observations where the bias (for every sample) is bounded by some $\epsilon' \geq 0$ (see Section~\ref{sec:biasedfbg}). 
For the unbiased case, one can set $\epsilon'$ or the average bias $\hat{\epsilon}$ as $0$ throughout. 

\item \textbf{Value Bounds: }
While the \FOptPlan subroutine of Algorithm~\ref{alg:mb} in Algorithm~\ref{alg:optimistic_planning} only uses the trivial upper-bound $V^{\max}_h = H - h + 1$ to clip the value estimates, Algorithm~\ref{alg:mb_app} uses upper-bounds $Q^{\max}_h(x)$ and $V^{\max}_{h+1}(x)$ that can depend on the given state-action pair $x$. This is useful in situations where one has prior knowledge on the optimal value for particular states and can a smaller value bound than the worst case bound of $H - h + 1$. This is the case in Algorithm~\ref{alg:dominatingset}, where we apply an instance of Algorithm~\ref{alg:mb_app} to the extended MDP with different reward functions per task.
\end{enumerate}

\noindent
We show that Algorithm~\ref{alg:mb_app} enjoys the IPOC bound (see \citet{dann2019policy}) in the theorem below. This is the main theorem and other statements follow as a special case. The proof can be found in the next section. 
\begin{theorem}[Main Regret / IPOC Theorem]
    For any tabular episodic MDP with episode length $H$, state-action space $\Xcal \subseteq \statespace \times \actionspace$ and directed, possibly stochastic, feedback graph $G$, Algorithm~\ref{alg:mb_app} satisfies with probability at least $1 - \delta$ a cumulative IPOC bound for all number of episodes $T$ of 
\begin{align}
	O\left(\sqrt{\bar \mas H \sum_{k=1}^T  V^{\pi_k}_1(s_{k,1})} \ln \frac{|\Xcal| H T}{\delta} 
	 + \bar \mas \wh \numS Q^{\max} H \ln^3 \frac{|\Xcal| H T}{\delta} + \sqrt{\bar \mas H^3 T  \epsilon_{\max}} \ln \frac{|\Xcal| H T}{\delta} + H^2 T \epsilon_{\max} \right),\nonumber
\end{align} 
    where $\bar M = \inf_{\nu} \frac{\mas(G_{\geq \nu})}{\nu}$ and $\mas(G_{\geq \nu})$ is the mas-number of a feedback graph that only contains edges that have probability at least $\nu$. Parameter $\wh \numS \leq \numS$ denotes a bound on the number of possible successor states of each $x \in \Xcal$. Further, $Q^{\max} \leq H$ is a bound on all value bounds used in the algorithm for state-action pairs that have visitation probability under any policy $\pi_k$ for all $k \in [T]$, i.e., $Q^{\max}$ satisfies
    \begin{align}
        Q^{\max} &\geq \max_{\substack{k \in [T],  h \in [H]}} ~ \max_{x \colon w_{k,h}(x) > 0} Q^{\max}_h(x), \text{ \qquad and, }\\
        Q^{\max} &\geq\max_{k \in [T], h \in [H]} ~ \max_{x \colon w_{k,h}(x) > 0} V^{\max}_{h+1}(x).
    \end{align} 
    The bound in this theorem is an upper-bound on the cumulative size of certificates $\sum_{k=1}^T \Vub_1(s_{k,1}) - \Vlb_1(s_{k,1})$ and on the regret $R(T)$.
\label{thm:cipoc_independencenumber_app}
\end{theorem}

\begin{algorithm}[p]
\SetAlgoLined
\SetNoFillComment
\SetInd{0.7em}{0.5em}
\SetKwInOut{Inputa}{input}
\SetKwInOut{Outputa}{ouput}
\SetKwInOut{Returna}{return}
  \SetKwProg{Fn}{function}{:}{}
\Inputa{failure tolerance $\delta \in (0,1]$, state-action space $\Xcal$, episode length $H$}
\Inputa{known bound on maximum transition support $\wh \numS \leq \|P(x)\|_0 \leq \numS$}
\Inputa{known bounds on value $V^{\max}_{h+1}(x) \leq H$ and $Q^{\max}_h(x) \leq H$ with $V^{\max}_{h+1}(x) \geq \max_{s' : P(s' | x) > 0} V^\star_{h+1}(s')$ and $Q^{\max}_h(x) \geq Q^\star_h(x)$}
    $\phi(n) \defeq 1 \wedge \sqrt{\frac{0.52}{n}\left( 1.4 \ln \ln(e \vee 2n) + \ln\frac{5.2 \times |\Xcal|(4\wh \numS + 5H + 7)}{\delta}\right)} = \Theta\left(\sqrt{\frac{\ln \ln n}{n}}\right)$\;

    Initialize $n_1(x) \gets 0, ~  
    \wh \epsilon_1(x) \gets 0, ~  
    \wh r_1(x) \gets 0$ 
    $\rhatsq_1(x) \gets 0, ~ 
    \wh P_1(x) \gets e_1 \in \{0, 1\}^{\numS} $ for all  $x  \in \Xcal$\;
    \tcc{Main loop}
\For{episode $k=1, 2, 3, \dots$}{
    $\pi_k, \Vub_{k,h}, \Vlb_{k,h} \gets$ \FOptPlan{$n_k, \wh r_k, \rhatsq_k, \wh P_k, \wh \epsilon_{k}$}\;
    Receive initial state $s_{k,1}$\;
    $n_{k+1}, \wh r_{k+1}, \rhatsq_{k+1}, \wh P_{k+1}, \wh \epsilon_{k+1} \gets $\FSample{$\pi_k, s_{k,1}, n_k, \wh r_k, \rhatsq_k, \wh P_k, \wh \epsilon_{k}$}\;
}

\tcc{Optimistic planning subroutine with biases}
\Fn{\FOptPlan$(n, \wh r, \rhatsq, \wh P, \wh \epsilon)$}{
$\Vub_{H+1}(s) = 0;~  \Vlb_{H+1}(s) = 0 \quad \forall s \in \statespace, k \in \NN
    $\;
    \For
    {$h=H$ \KwTo $1$ \textbf{and} $s \in \statespace$}
    {

            \For{$a \in \actionspace$}
            {
            $x \gets (s,a)$\;
            $\eta \gets \sqrt{\rhatsq(x) - \wh r(x)^2} + 2\sqrt{\wh \epsilon(x)} H +  \sigma_{\wh P(x)}(\Vub_{h+1})$\;
            $\psi_h(x) \gets
            4 \eta \phi(n(x)) + 53 \wh \numS H V^{\max}_{h+1}(x) \phi(n(x))^2
    +\frac{1}{H} \wh P(x) (\Vub_{h+1} - \Vlb_{h+1})
    + (H + 1) \wh \epsilon(x)$\label{lin:psiterm_app}\;
                $\Qub_{h}(x) \gets 0 \vee \,\,( \wh r(x) 
                + \wh P(x)\Vub_{h+1} + \psi_{h}(x)) \,\,\wedge Q^{\max}_h(x)$\tcp*{UCB of $Q^\star_{h} \leq V^{\max}_h \leq H$}
                $\Qlb_{h}(x) \gets 0 \vee \,\,(  \wh r(x) 
                + \wh P(x)\Vlb_{h+1} - \psi_h(x))\,\,\wedge Q^{\max}_h(x)$\tcp*{LCB of $Q^{\pi}_{h} \geq 0$}
             }
 $\pi(s, h) \gets \argmax_{a} \Qub_{h}(s, a)$,
 $\Vub_{h}(s) \gets \Qub_{h}(s, \pi(h))$,
 $\Vlb_{h}(s) \gets \Qlb_{h}(s, \pi(h)) $\;
}
\Returna{$\pi, \Vub_h, \Vlb_h$}
}
\tcc{Sampling subroutine with biases}
\Fn{\FSample($\pi, n, \wh r, \rhatsq, \wh P, \wh \epsilon$)}{
\For{$h=1, \dots H$}{
    Take action $a_{h} = \pi(s_{h}, h)$ and transition to $s_{h+1}$ with reward $r_h$\;
    Receive transition observations $\Ocal_{h}(G)$\;
   \For(){transition $(x,r, s', \epsilon') \in \Ocal_{h}(G)$}{
        ~$n(x) \gets n(x) + 1$\;
        $\begin{array}{ll}
            \wh r(x) \gets \frac{n(x)-1}{n(x)} \wh r(x) + \frac{1}{n(x)}r,  &  \rhatsq(x) \gets \frac{n(x)-1}{n(x)} \rhatsq(x) + \frac{1}{n(x)}r^2, \\
            \wh \epsilon(x) \gets \frac{n(x)-1}{n(x)} \wh \epsilon(x) + \frac{1}{n(x)}\epsilon', & \qquad \wh P(x) \gets \frac{n(x)-1}{n(x)} \wh P(x) +  \frac{1}{n(x)} e_{s'},
        \end{array}$ \\ \quad where $e_{s'} \in \{0,1\}^{\numS}$ has $1$ on the $s'$-th position\;
    }
    \Returna{($n, \wh r, \rhatsq, \wh P, \wh \epsilon$)}
}
}
\caption{Optimistic model-based RL algorithm for biased side observations}
\label{alg:mb_app}
\end{algorithm}

\section{Analysis of Model-Based RL with Feedback Graphs}
\label{app:modelbasedproofs}

Before presenting the proof of the main Theorem~\ref{thm:cipoc_independencenumber_app} stated in the previous section, we show that 
Theorem~\ref{thm:cipoc_independencenumber} and Theorem~\ref{thm:cipoc_stochastic} indeed follow from Theorem~\ref{thm:cipoc_independencenumber_app}: 

\paragraph{Proof of Theorem~\ref{thm:cipoc_independencenumber}.}
\begin{proof} 
We will reduce from the bound in Theorem~\ref{thm:cipoc_independencenumber_app}. We start by setting the bias in Theorem~\ref{thm:cipoc_independencenumber} to zero by plugging in $\epsilon_{\max} = 0$. Next, we set the worst-case value $Q^{\max} = H$. Next, we set the thresholded mas-number of the stochastic graph $\bar \mas$ to  the mas-number $\mas$ of deterministic graphs (by setting $\nu = 1$ in the definition of $\bar \mas$). Finally,  we upper-bound the initial values for all played policies 
by the maximum value of $H$ rewards, i.e.,
\begin{align}
    \sum_{k=1}^T  V^{\pi_k}_1(s_{k,1}) \leq TH. \label{eqn:Vbound11}
\end{align}
Plugging all of the above in the statement of Theorem~\ref{thm:cipoc_independencenumber_app}, we get that Algorithm~\ref{alg:mb} satisfies the IPOC bound of 
\begin{align}
	&O\left(\sqrt{\mas H^2 T} \ln \frac{|\Xcal| H T}{\delta} 
	 + \wh \numS \mas  H^2 \ln^3 \frac{|\Xcal| H T}{\delta}\right). 
\end{align}
\end{proof}

\paragraph{Proof of Theorem~\ref{thm:cipoc_stochastic}.} 
\begin{proof}
The proof follows similar to the proof of Theorem~\ref{thm:cipoc_independencenumber} (above), while setting $\epsilon_{\max} \neq 0$.
Following Theorem~\ref{thm:cipoc_independencenumber_app}, this yields additional regret / cumulative certificate size of at most
\begin{align}
    O\left(\sqrt{\bar \mas H^3 T  \epsilon_{\max}} \ln \frac{|\Xcal| H T}{\delta} + H^2 T \epsilon_{\max} \right).
\end{align}
\end{proof}

\paragraph{Proof of the main theorem.} 
The proof of our main result, Theorem \ref{thm:cipoc_independencenumber_app}, is provided in parts in the following subsections:
\begin{itemize}
    \item \textbf{Section~\ref{app:modelbasedproofs_probs}} considers the event in which the algorithm performs well. The technical lemmas therein guarantee that this event holds with high probability. 
    \item \textbf{Section~\ref{app:modelbasedproofs_debias}} quantifies the amount of cumulative bias in the model estimates and other relevant quantities. 
    \item \textbf{Section~\ref{app:modelbasedproofs_valid}} proves technical lemmas that establish that \FOptPlan always returns valid confidence bounds for the value functions.
    \item \textbf{Section~\ref{app:modelbasedproofs_tightness}} bounds how far apart can the confidence bounds provided by \FOptPlan can be for each state-action pair.
    \item \textbf{Section~\ref{app:graphproofs}} (above) contains general results on self-normalized sequences on nodes of graphs that only depend on the structure of the feedback graph.
    \item \textbf{Section~\ref{app:modelbasedproofs_main}} connects all the results from the previous sections into the proof of Theorem~\ref{thm:cipoc_independencenumber_app}.
\end{itemize}
\noindent 

\subsection{High-Probability Arguments}
\label{app:modelbasedproofs_probs}
In the following, we establish concentration arguments for empirical MDP models computed from data collected by interacting with the corresponding MDP (with the feedback graph). 

We first define additional notation. To keep the definitions uncluttered, we will use the unbiased versions of the empirical model estimates and bound the effect of unbiasing in Section~\ref{app:modelbasedproofs_debias} below. The unbiased model estimates are defined as
\begin{align}
\bar r_k(x) &= \wh r_k(x) - \frac{1}{n_k(x)} \sum_{i=1}^{n_k(x)} \bar \epsilon_i(x),\\
\bar P_k(s' | x) &= 
\wh P_k(s' | x) - \frac{1}{n_k(x)} \sum_{i=1}^{n_k(x)} \bar \epsilon_{i}(x, s')
\label{eqn:debias_trans_def}
\end{align}
where $\bar \epsilon_i(x)$ is the bias of the $i^\text{th}$ reward observation $r_i$ for $x$ and
$\bar \epsilon_{i}(x, s')$ is the bias of the $i^{\text{th}}$ transition observation of $s'$ for $x$.
Recall that $\bar \epsilon_i(x)$ and $\bar \epsilon_i(x, s')$ are unknown to the algorithm, which, however, receives an upper bound $\epsilon_i'$ on $|\bar \epsilon_i(x)|$ and $\sum_{s' \in \statespace}| \bar \epsilon_i(x, s')|$ for each observation $i$. Additionally, for any probability parameter $\delta' \in (0, 1)$, define the function 
\begin{align}
\phi(n) \defeq 1 \wedge \sqrt{\frac{0.52}{n}\left( 1.4 \ln \ln(e \vee 2n) + \ln\frac{5.2}{\delta'}\right)} = \Theta\left(\sqrt{\frac{\ln \ln n}{n}}\right).  
\label{eqn:phidef}
\end{align}

We now define several events for which we can ensure that our algorithms exhibit good behavior with. 

\paragraph{Events regarding immediate rewards.} ~\\
The first two event $\mathsf{E}^{\text{R}}$ and $\mathsf{E}^{\text{RE}}$ are the concentration of (unbiased) empirical estimates $\bar r_k(x)$ of the immediate rewards around the population mean $r(x)$ using a Hoeffding and empirical Bernstein bound respectively, i.e., 

\begin{align}
    {\mathsf{E}}^{\text{R}} =& \left\{ \forall ~k \in \NN, x \in \Xcal: \, 
        |\bar r_k(x) - r(x)| 
        \leq \phi(n_k(x))
        \right\},\\
   {\mathsf{E}}^{\text{RE}} =& \bigg\{ \forall k \in \NN, x \in \Xcal: \, 
     |\bar r_k(x) - r(x)| 
    \leq \sqrt{8\overline{\operatorname{Var}}_k(r|x)}\phi(n_k(x)) + 7.49 \phi(n_k(x))^2\bigg\},
\end{align}
where the unbiased empirical variance is defined as 
$\overline{\operatorname{Var}}_k(r|x) = \frac{1}{n_k(x)}  
\sum_{i=1}^{n_k(x)} \left(r_i - \bar \epsilon_i(x) - \bar r_k(x)
\right)^2$. The next event ensures that the unbiased empirical variance estimates concentrate around the true variance $\operatorname{Var}(r|x)$
\begin{align}
           \mathsf{E}^{\text{Var}} =& \bigg\{ \forall k \in \NN, x \in \Xcal: \, 
     \sqrt{\overline {\operatorname{Var}}_k(r|x)} \leq \sqrt{\operatorname{Var}(r|x)} 
    + \sqrt{\frac{2 \ln(\pi^2 n^2 / 6 \delta')}{n}} \bigg\}.
\end{align}

\paragraph{Events regarding state transitions.} The next two events concern the concentration of empirical transition estimates. We consider the unbiased estimate of the probability to encounter state $s'$ after state-action pair $x$ as defined in Equation~\eqref{eqn:debias_trans_def}. As per Bernstein bounds, they concentrate around the true transition probability $P(s' | x)$ as
    \begin{align}
\mathsf{E}^{\text{P}} =& \bigg\{ \forall ~k \in \NN ,  s' \in \statespace, x \in \Xcal: \, | \bar P_k(s' | x) - P(s' | x)| 
        \leq 
        \sqrt{4 P(s' | x)} \phi(n_k(x)) + 1.56 \phi(n_k(x))^2\bigg\},\\
    \mathsf{E}^{\text{PE}} =& \bigg\{ \forall ~k \in \NN,  s' \in \statespace, x \in \Xcal: \, | \bar P_k(s' | x) - P(s' | x)| 
        \leq 
        \sqrt{4 \bar P_k(s' | x)} \phi(n_k(x)) + 4.66 \phi(n_k(x))^2\bigg\},
\end{align}
where the first event uses the true transition probabilities to upper-bound the variance and the second event uses the empirical version.
Both events above treat the probability of transitioning to each successor state $s' \in \statespace$ individually which can be loose in certain cases. We therefore also consider the concentration in total variation in the following event
\begin{align}
    \mathsf{E}^{\text{L}_1} =& \left\{ \forall ~k \in \NN, x \in \Xcal: \,  \|\bar P_k(x) - P(x)\|_1 
        \leq 2 \sqrt{\wh \numS}\phi(n_k(x)) \right\},
\end{align}
where $\bar P_k(x) = (\bar P_k(s' | x))_{s' \in \statespace} \in \RR^{\numS}$ is the vector of transition probabilities. The event $\mathsf{E}^{\text{L}_1}$ has the typical $\sqrt{\wh \numS}$ dependency in the RHS of an $\ell_1$ concentration bound. In the analysis, we will often compare the expected the empirical estimate of the expected optimal value of successor state $\bar P_k(x) V^{\star}_{h+1} = \sum_{s' \in \statespace} \bar P_k(s' | x)  V^{\star}_{h+1}(s')$ to its population mean $P(x) V^{\star}_{h+1}$ and we would like to avoid the $\sqrt{\wh \numS}$ factor. To this end, the next two events concern this difference explicitly 
    \begin{align}
    \mathsf{E}^{\text{V}} =& \left\{ \forall k \in \NN , h \in [H], x \in \Xcal: \, 
        |(\bar P_k(x) - P(x)) V^{\star}_{h+1}| 
        \leq \range(V^{\star}_{h+1}) \phi(n_k(x)) \right\}\\
    \mathsf{E}^{\text{VE}} =& \bigg\{ \forall k \in \NN, h \in [H], x \in \Xcal: \, 
     |(\bar P_k(x) - P(x)) V^{\star}_{h+1}| 
    \leq 2\sqrt{ \bar P_k(x)[( V^\star_{h+1} - P(x)V^\star_{h+1})^2]}
    \phi(n_k(x))\nonumber\\
    &\hspace{10cm}+
     4.66 \range(V^{\star}_{h+1}) \phi(n_k(x))^2\bigg\}
\end{align}
where $\range(V^{\star}_{h+1}) = \max_{s' \in \statespace} V^{\star}_{h+1}(s') - \min_{s' \in \statespace} V^{\star}_{h+1}(s')$ is the range of possible successor values. The first event $\mathsf{E}^{\text{V}}$ uses a Hoeffding bound and the second event $\mathsf{E}^{\text{VE}}$ uses empirical Bernstein bound.

\paragraph{Events regarding observation counts.} 
All events definitions above include the number of observations $n_k(x)$ to each state-action pair $x \in \Xcal$ before episode $k$. This is a random variable itself which depends on how likely it was in each episode $i < k$ to observe this state-action pair. The last events states that the actual number of observations cannot be much smaller than the total observation probabilities of all episodes so far. 
We denote by $w_{i}(x) = \sum_{h \in [H]} \prob(s_{i,h} = s(x), a_{i,h} = a(x)  ~ | ~ s_{i,1}, \Hcal_{1:i-1})$ the expected number of \emph{visits} to each state-action pair $x = (s(x), a(x)) \in \Xcal \subseteq \statespace \times \actionspace$ in the $i$th episode given all previous episodes $\Hcal_{1:i-1}$ and the initial state $s_{i,1}$. The event is defined as
    \begin{align}
	\mathsf{E}^{\text{N}} =&   \left\{ \forall ~k \in \NN ,x \in \Xcal \colon        n_{k}(x) \geq \frac 1 2 \sum_{i < k} \sum_{\bar x \in \Xcal} q(\bar x, x) w_{i}(\bar x) - H\ln \frac{1}{\delta'} \right\}.
\end{align}

The following lemma shows that each of the events above is indeed a high-probability event and that their intersection has high probability at least $1 - \delta$ for a suitable choice of the $\delta'$ in the definition of $\phi$ above.

\begin{lemma}
    Consider the data generated by sampling with a feedback graph from an MDP with arbitrary, possibly history-dependent policies. Then, for any $\delta' > 0$, the probability of each of the events, defined above, is bounded as  
        \begin{multicols}{2}
	\begin{enumerate}[label=(\roman*)] 
	\setlength{\itemindent}{0.5cm}
		\item $ \prob(\mathsf{E}^{\mathrm{RE}} \cup \mathsf{E}^{\mathrm{R}})  \geq 1 - 4 |\Xcal|   \delta'$, 
        \item $\prob(\mathsf{E}^{\mathrm{RE}} \cup \mathsf{E}^{\mathrm{R}})             \geq 1 - 4 |\Xcal|   \delta'$, 
        \item $\prob(\mathsf{E}^{\mathrm{Var}})                      \geq 1 - |\Xcal|   \delta' $
        \item $\prob(\mathsf{E}^{\mathrm{P}})             \geq 1 - 2 \wh \numS|\Xcal| \delta'$, 
        \item $ \prob(\mathsf{E}^{\mathrm{PE}})            \geq 1 - 2 \wh \numS|\Xcal| \delta'$, 
         \item $\prob(\mathsf{E}^{\mathrm{L}_1})            \geq 1 - 2 |\Xcal| \delta'$,
         \item $\prob\left(  \mathsf{E}^{\mathrm{V}} \right) \geq 1 - 2 |\Xcal|H \delta'$, 
         \item $\prob(\mathsf{E}^{\mathrm{VE}})            \geq 1 - 2 |\Xcal|H \delta'$,  
         \item $ \prob(E^{\mathrm{N} })            \leq |\Xcal| H \delta'. $
       \end{enumerate}
    \end{multicols}

\noindent
    Further, define the event $E$ as $E \ldef{} \mathsf{E}^{\mathrm{R}} \cap\mathsf{E}^{\mathrm{RE}} \cap\mathsf{E}^{\mathrm{Var}} \cap \mathsf{E}^P \cap \mathsf{E}^{\mathrm{PE}} \cap \mathsf{E}^{\mathrm{L}_1} \cap \mathsf{E}^V \cap \mathsf{E}^{\mathrm{VE}} \cap \mathsf{E}^{\mathrm{N}}$. Then, the event $E$ occurs with probability at least $1 - \delta$, i.e.
										$$ \prob(E) \geq 1 - \delta, $$    
    where $\delta = {\delta'}|\Xcal|(7 + 4 \wh \numS + 5H)$. 
   
    \label{lem:goodprob}
\end{lemma}
\begin{proof} 

We bound the probability of occurrence of the events $\mathsf{E}^R, \mathsf{E}^P, \mathsf{E}^{\text{PE}}, \mathsf{E}^{\text{L}_1}, \mathsf{E}^V$ and  $\mathsf{E}^{\text{VE}}$ using similar techniques as in the works of \citet{dann2019policy, zanette2019tighter} (see for example Lemma~6 in \citet{dann2019policy}). However, in our setting, we work with a slightly different $\sigma$-algebra to account for the feedback graph, and explicitly leverage the bound on the number of possible successor states $\wh \numS$. We detail this deviation from the previous works for events $\mathsf{E}^{\text{R}}$ and $\mathsf{E}^{\mathrm{RE}}$ in Lemma~\ref{lem:goodprob_ER} (below),  and the rest follow analogously. 

Further, we bound the probability of occurrence of the event $\mathsf{E}^{\mathrm{N}}$ in Lemma~\ref{lem:goodprob_wfeedbackgraph}. The proof significantly deviates from the prior work, as in our case, the number of observations for any state-action pair is different from the number of visits of the agent to that pair due to the feedback graph. Finally, the bound for the probability of occurrence of $\mathsf{E}^{\text{Var}}$ is given in Lemma~\ref{lem:goodprob_var}. 

Taking a union bound for all the above failure probabilities, and setting $\delta' = \frac{\delta}{|\Xcal|(7 + 4 \wh \numS + 5H)}$, we get a bound on the probability of occurrence of the event $\prob(E)$.  
\end{proof}

\begin{lemma}  \label{lem:goodprob_ER}
    Let the data be generated by sampling with a feedback graph from an MDP with arbitrary, possibly history-dependent policies. Then, the event  $\mathsf{E}^{\mathrm{R}} \cap\mathsf{E}^{\mathrm{RE}}$ occurs with probability at-least $1 - 4 |\Xcal| \delta'$, or \begin{align}
        \prob(\mathsf{E}^{\mathrm{R}} \cap\mathsf{E}^{\mathrm{RE}}) \geq 1 - 4 |\Xcal| \delta'.
    \end{align}
\end{lemma}
\begin{proof}
 Let $\mathcal F_{j}$ be the natural $\sigma$-field induced by everything (all observations and visitations) up to the time when the algorithm has played a total of $j$ actions and has seen which state-action pairs will be observed but not the actual observations yet. More formally, let $k = \lceil \frac{j}{H} \rceil $ and $h = j \mod H$ be the episode and the time index when the algorithm plays the $j^\text{th}$ action. Then everything in episodes $1 \dots k-1$ is $\Fcal_j$-measurable as well as everything up to $s_{k,h}, a_{k,h}$ and $\bar \Ocal_{k,h}(G)$ (which $x$ are observed at $k,h$) but not $\Ocal_{k,h}(G)$ (the actual observations) or $s_{k,h+1}$.

We will use a filtration with respect to the stopping times of when a specific state-action pair is observed. To that end, 
consider a fixed $x \in \Xcal$. Define 
\begin{align}
    \tau_i = \inf\left\{ (k-1)H + h \colon 
    \sum_{j=1}^{k-1} \sum_{h'=1}^H \one\{ x \in \bar \Ocal_{j,h'}(G) \} 
    + \sum_{h'=1}^{h} \one\{ x \in \bar \Ocal_{k,h'}(G) \}  \geq i \right\}
\end{align}
to be the index $j$ of $\Fcal_j$ where $x$
was observed for the $i^\text{th}$ time. Note that, for all $i$, $\tau_i$ are
stopping times with respect to $(\mathcal F_j)_{j=1}^\infty$.  Hence, $\Fcal^{x}_i = \mathcal F_{\tau_i} = \{
A \in \mathcal F_\infty \, : \, A \cap \{\tau_i \leq t \} \in \mathcal F_t \,\,\forall\, t \geq 0\}$
is a $\sigma$-field. Intuitively, it captures all information available at time $\tau_i$ \citep[Sec.~3.3]{lattimore2018bandit}. Since $\tau_i \leq \tau_{i+1}$, 
the sequence $(\mathcal F_{\tau_i})_{i=1}^\infty$ is a filtration as well.

Consider a fixed $x \in \Xcal$ and number of observations $n$. Define $X_i = \one\{\tau_i < \infty\} (r_i - \bar \epsilon_i(x) - r(x))$  where $r_i$ is the $i^{\text{th}}$ observation with bias $\bar \epsilon_i(x)$ of $x$. By construction $(X_i)_{i=1}^{\infty}$ is adapted to the filtration $(\Fcal^{x}_i)_{i=1}^\infty$. Further, recall that $r(x) = \EE[r | (s, a) = x] - \bar \epsilon_i$ is the immediate expected reward in $x$ and hence, we one can show that $(X_i)_{i=1}^{\infty}$ is a martingale with respect to this filtration. It takes values in the range $[ - r, 1 - r]$. 
We now use a Hoeffding bound and empirical Bernstein bound on $\sum_{i=1}^n X_i$ to show that the probability of $\mathsf{E}^{\mathrm{R}}$ and $\mathsf{E}^{\mathrm{RE}}$ is sufficiently large. We use the tools provided by \citet{howard2018uniform} for both concentration bounds.
The martingale $\sum_{i=1}^n X_i$ satisfies Assumption~1 in \citet{howard2018uniform} with $V_n = n / 4$ and any sub-Gaussian boundary (see Hoeffding I entry in Table 2
therein). The same is true for $- \sum_{i=1}^n X_i$. Using the sub-Gaussian boundary in Corollary~22 in \citet{dann2019policy}, we get that
\begin{align}
    \left|\frac 1 n \sum_{i=1}^n X_i \right| 
     \leq 1.44 \sqrt{\frac{n}{4 n^2}\left(1.4 \ln \ln ( e \vee n/2) + \ln \frac {5.2} {\delta'} \right)}
     \leq \phi(n)
\end{align}
holds for all $n \in \NN$ with probability at least $1 - 2\delta'$. It therefore also holds for all random $n$ including the number of observations of $x$ after $k-1$ episodes. Hence, the condition in $\mathsf{E}^R$ holds for all $k$ for a fixed $x$ with probability at least $1 - 2 \delta'$. An additional union bound over $x \in \Xcal$ gives $\prob(\mathsf{E}^R) \geq 1 - 2|\Xcal|\delta'$. 

We can proceed analogously for $\mathsf{E}^{\text{RE}}$, except that we use the uniform empirical Bernstein bound from  Theorem~4 in \citet{howard2018uniform} with the sub-exponential uniform boundary in Corollary~22 in \citet{dann2019policy} which yields
\begin{align}
    \left|\frac 1 n \sum_{i=1}^n X_i \right| 
     \leq 1.44 \sqrt{\frac{V_n}{n^2}\left(1.4 \ln \ln ( e \vee 2 V_n) + \ln \frac {5.2} {\delta'} \right)} + \frac{2.42}{n}\left(1.4 \ln \ln ( e \vee 2 V_n) + \ln \frac {5.2} {\delta'} \right)
     \label{eqn:r_emp_bern_1}
\end{align}
with probability at least $1 - 2 \delta'$ for all $n \in \NN$. Here, $V_n = \sum_{i=1}^n X_i^2 \leq n$. Using the definition of $\phi(n)$ in Equation~\eqref{eqn:phidef}, we can upper-bound the right hand side in the above equation with $2\sqrt{V_n/n}\phi(n) + 4.66 \phi(n)^2$. We next bound $V_n$ in the above by the de-biased variance estimate
\begin{align}
    V_n = \sum_{i=1}^n X_i^2 = &\sum_{i=1}^n (r_i - \bar \epsilon_i(x) - r(x))^2
    = \sum_{i=1}^n (r_i - \bar \epsilon_i(x) - r(x))^2 \\
    \leq &~ 2\sum_{i=1}^n (r_i - \bar \epsilon_i(x) - \bar r_{\tau_n}(x))^2 + 2n(r(x) - \bar r_{\tau_n}(x))^2
\end{align}
Applying the definition of event $\mathsf{E}^{\mathrm{R}}$, we know that $|r(x) - \bar r_{\tau_n}(x)| \leq \phi(n)$ and thus
$V_n / n \leq 2\overline{\operatorname{Var}}_{\tau_n}(r|x) + 2 \phi(n)^2$. Plugging this back into \eqref{eqn:r_emp_bern_1} yields
\begin{align}
   |\bar r_{\tau_n}(x) - r(x)| = \left|\frac 1 n \sum_{i=1}^n X_i \right| 
   &\leq 2\sqrt{2\overline{\operatorname{Var}}(r) + 2 \phi(n)^2}\phi(n) + 4.66 \phi(n)^2\\
   &\leq \sqrt{8 \overline{\operatorname{Var}}(r)}\phi(n) + 7.49\phi(n)^2
\end{align}
This is the condition of $\mathsf{E}^{\text{RE}}$ which holds for all $n$ and as such $k$ as long as $\mathsf{E}^{\text{R}}$ also holds. With a union bound over $\Xcal$, this yields $$\prob(\mathsf{E}^{\text{RE}} \cup \mathsf{E}^{\mathrm{R}}) \geq 1 - 4 |\Xcal| \delta'. $$
\end{proof}

\begin{lemma} \label{lem:goodprob_var} 
   Let the data be generated by sampling with a feedback graph from an MDP with arbitrary (and possibly history-dependent) policies. Then, the event $\mathsf{E}^{\text{Var}}$ occurs with probability at least $1 - |\Xcal|\delta'$, i.e., 
   $$\prob(\mathsf{E}^{\mathrm{Var}}) \geq  1 - |\Xcal|\delta'.$$ 
\end{lemma}
\begin{proof}
Consider first a fix $x \in \Xcal$ and let $K$ be the total number of observations for $x$ during the entire run of the algorithm. We denote the observations by $r_i$. Define now $X_i = r_i - \bar \epsilon_i(x)$ for $i \in [K]$ and $X_i \sim P_R(x)$ independently. Then by construction $X_i$ is a sequence of i.i.d. random variables in $[0,1]$. We now apply Theorem~10, Equation~4 by \citet{maurer2009empirical} which yields that for any $n$
\begin{align}
    \sqrt{\frac{n}{n-1}\widehat{\operatorname{Var}}(X_n)} \leq
    \operatorname{Var}(X) + \sqrt{\frac{2 \ln (n^2 \pi^2 / 6 \delta')}{n-1}}
\end{align}
holds with probability at least $1 - \frac{6\delta' }{\pi^2 n^2}$, where $\operatorname{Var}(X)$ is the variance of $X_i$ and $\widehat{\operatorname{Var}}(X_n) = \frac{1}{n} \sum_{i=1}^n (X_i - \bar X_n)^2$ with $\bar X_n = \frac 1 n \sum_{i=1}^n X_i$ is the empirical variance of the first $n$ samples. By applying a union bound over $n \in \NN$, and multiplying by $\sqrt{n / (n-1})$ we get that 
\begin{align}
    \sqrt{\widehat{\operatorname{Var}}(X_n)} \leq
    \sqrt{\frac{n-1}{n}}\operatorname{Var}(X) + \sqrt{\frac{2 \ln (n^2 \pi^2 / 6 \delta)}{n}}
    \leq 
    \operatorname{Var}(X) + \sqrt{\frac{2 \ln (n^2 \pi^2 / 6 \delta)}{n}}
\end{align}
holds for all $n \in \NN$ with probability at least $1 - \frac{6\delta'}{\pi^2}\sum_{n=1}^\infty \frac{1}{n^2} \geq 1 - \delta'$.
We now note that $\operatorname{Var}(X) = \operatorname{Var}(r|x)$ and for each episode $k$, there is some $n$ so that $\overline {\operatorname{Var}}_k(r|x) = \widehat{\operatorname{Var}}(X_n)$.
Hence, with another union bound over $x \in \Xcal$, the statement follows.
\end{proof}

\begin{lemma}    \label{lem:goodprob_wfeedbackgraph}
 Let the data be generated by sampling with a feedback graph from an MDP with arbitrarily (possibly adversarially) chosen initial states. Then, the event  $\mathsf{E}^{\mathrm{N}}$ occurs with probability at-least  $1 - H |\Xcal| \delta'$, or  $$\prob(\mathsf{E}^{\mathrm N}) \geq 1 - H |\Xcal| \delta'.$$
\end{lemma}
\begin{proof}
    Consider a fixed $x \in \Xcal$ and $h \in [H]$.
    We define $\mathcal F_k$ to be the sigma-field induced by the first $k-1$ episodes and $s_{k,1}$. Let 
    $X_{k,h} = \one\{x \in \bar \Ocal_{k,h}(G)\}$ be the indicator whether $x$ was observed in episode $k$ at time $h$. The probability that this indicator is true given $\Fcal_k$ is simply the probability 
$w_{k,h}(x) =\prob(s_{k,h} = s(x), a_{k,h} = a(x)  ~ | ~ s_{k,1}, \Hcal_{1:k-1})$
    of visiting each $\bar x \in \Xcal$ at time $h$ and the probability $q(\bar x, x)$ that $\bar x$ has an edge to $x$ in the feedback graph in the episode
    \begin{align}
        \prob(X_{k,h} = 1 ~ | \Fcal_k) = \sum_{\bar x \in \Xcal_h} q(\bar x, x) w_{k}(\bar x).
    \end{align} 
    We now apply Lemma~F.4 by \citet{dann2017unifying} with $W = \ln \frac{1}{\delta'}$ and obtain that
    \begin{align}
        \sum_{i=1}^k X_{i,h} \geq \frac 1 2 \sum_{i =1}^k \sum_{\bar x \in \Xcal_h} q(\bar x, x) w_{i}(\bar x) - \ln \frac{1}{\delta'}
    \end{align}
    for all $k \in \NN$ with probability at least $1 - \delta'$.
    We now take a union-bound over $h \in [H]$ and $x \in \Xcal$ get that
    $\prob( \mathsf{E}^{\mathrm N} ) \geq 1 - |\Xcal| H \delta'$ after summing over $h \in [H]$ because the total number of observations after $k-1$ episodes for each $x$ is simply $n_k(x) = \sum_{i=1}^{k-1} \sum_{h \in [H]} X_{k,h}$.
\end{proof}

 \subsection{Bounds on the Difference of Biased Estimates and Unbiased Estimates}
 \label{app:modelbasedproofs_debias} 
 
 We now derive several helpful inequalities that bound the difference of biased and unbiased estimates.
\begin{align}
|\bar r_k(x) - \wh r_k(x)| =& \frac{1}{n_k(x)} \sum_{i=1}^{n_k(x)} \bar \epsilon_i(x) \leq \wh \epsilon_k(x)\\
\|\bar P_k(x) - \wh P_k(x)\|_1 =& 2 \max_{\Bcal \subseteq \statespace} 
|\bar P_k(\Bcal | x) - \wh P_k(\Bcal | x)|
=
2\left| \sum_{s' \in \Bcal}  \frac{1}{n_k(x)} \sum_{i=1}^{n_k(x)} \bar \epsilon_{i}(x, s') \right|\\
\leq &
\frac{2}{n_k(x)} \sum_{i=1}^{n_k(x)} \left| \sum_{s' \in \Bcal}\bar \epsilon_{i}(x, s') \right|
\leq \wh \epsilon_k(x).
\end{align}
The final inequality follows from the fact that $\sum_{s' \in \Bcal}\bar \epsilon_{i}(x, s') \leq \frac{1}{2} \|P(x) - P'_i(x)\|_1 \leq \frac{\epsilon'_i}{2}$ where $P_i'(x)$ denotes the true distribution of the $i^\text{th}$ transition observation of $x$ and $\epsilon'_i$ denotes the bias parameter for this observation. From this total variation bound, we can derive a convenient bound on the one-step variance of any $``$value"-function  $f \colon \statespace \rightarrow [0, f_{\max}]$ over the states. In the following, we will use the notation
\begin{align}
\sigma^2_{P}(f) \defeq \Ex_{s' \sim P}[f(s')^2] - \Ex_{s' \sim P}[f(s')]^2.
\end{align}
Using this notation, we bound the difference of the one-step variance of the biased and unbiased state distributions as
\begin{align}
|\sigma^2_{\bar P_k(x)}(f) - \sigma^2_{\wh P_k(x)}(f)|
= &~
|\bar P_k(x) f^2 - (\bar P_k(x)f)^2 - \wh P_k(x)f^2 + (\wh P_k(x)f)^2|\\
=&~
|(\bar P_k(x) - \wh P_k(x)) f^2 + (\bar P_k(x) - \wh P_k(x))f (\bar P_k(x) + \wh P_k(x))f|
\\
\leq &~ f_{\max}^2 \|\bar P_k(x) - \wh P_k(x)\|_1 
    + 2 f_{\max}^2 \|\bar P_k(x) - \wh P_k(x)\|_1
\leq 3 f_{\max}^2 \wh \epsilon_k(x).
\label{eqn:varbiasbound}
\end{align}

We also derive the following bounds on quantities related to the variance of immediate rewards. In the following, we consider any number of episodes $k$ and $x \in \Xcal$. To keep notation short, we omit subscript $k$ and argument $x$ below. That is, $r = r(x)$ is the expected reward, $n = n_k(x)$ is the number of observations, which we denote by $r_1, \dots, r_n$ each. Further $\bar \epsilon_i = \bar \epsilon_i(x)$ is the bias of the $i$th reward sample for this $x$ and $\epsilon_i \geq \bar \epsilon_i$ the accompanying upper-bound provided to the algorithm. We denote by $\widehat{\operatorname{Var}}(r) = \frac 1 n \sum_{i=1}^n ( r_i - \wh r)^2$ the empirical variance estimate and by $\overline {\operatorname{Var}}(r) = \overline {\operatorname{Var}}_k(r|x) = \frac 1 n \sum_{i=1}^n(r_i - \bar \epsilon_i - \bar r)^2$. Thus, 
\begin{align}\overline {\operatorname{Var}}(r) =
    \frac 1 n \sum_{i=1}^n(r_i - \bar \epsilon_i - \bar r)^2
    \leq & ~
     \frac 2 n \sum_{i=1}^n ( r_i - \wh r)^2 + 
     \frac 2 n \sum_{i=1}^n (\wh r - \bar \epsilon_i - \bar r)^2 
    \\
    = & ~
    2 \widehat{\operatorname{Var}}(r) +
     \frac 2 n \sum_{i=1}^n \left(\Big(\frac 1 n \sum_{j=1}^n \bar \epsilon_j\Big) - \bar \epsilon_i\right)^2 
    \\
    \leq & ~
    2  \widehat{\operatorname{Var}}(r) +
     \frac 2 n \sum_{i=1}^n \bar \epsilon_i^2 
        \leq 
    2  \widehat{\operatorname{Var}}(r) +
     \frac 2 n \sum_{i=1}^n \epsilon_i^2
    \leq     2  \widehat{\operatorname{Var}}(r) +
    2 \wh \epsilon,
    \label{eqn:vardebias_to_emp}
\end{align}
 where the last inequality follows from the definition of  $\wh \epsilon$ and using the fact that $\bar \epsilon_i \leq 1$. The right hand side of the above chain of inequalities is empirically computable and, subsequently, used to derive the reward bonus terms.

Analogously, we can derive a reverse of this bound that upper bounds the computable variance estimate $\widehat{\operatorname{Var}}(r)$ by the unbiased variance estimate $\overline{\operatorname{Var}}(r)$. This is given as 
\begin{align}
    \widehat{\operatorname{Var}}(r) 
    = \frac 1 n \sum_{i=1}^n (r_i - \wh r)^2
    \leq & ~
     \frac 2 n \sum_{i=1}^n ( r_i - \bar \epsilon_i - \bar r)^2 +
     \frac 2 n \sum_{i=1}^n ( \bar \epsilon_i - \wh r + \bar r)^2 
    \\ 
    \leq & ~
         \frac 2 n \sum_{i=1}^n ( r_i - \bar \epsilon_i - \bar r)^2 +
     \frac 2 n \sum_{i=1}^n \epsilon_i^2
    =  2 \overline {\operatorname{Var}}(r) +    \frac 2 n \sum_{i=1}^n \epsilon_i^2.
    \leq 2 \overline {\operatorname{Var}}(r) +   2 \wh \epsilon. 
    \label{eqn:varemp_to_debias}
\end{align}

\subsection{Correctness of optimistic planning}
\label{app:modelbasedproofs_valid} 
In this section, we provide the main technical results to guarantee that in event  $E$ (defined in Lemma~\ref{lem:goodprob}), the output of \FOptPlan are upper and lower confidence bounds on the value functions. 

\begin{lemma}[Correctness of Optimistic Planning]
\label{lem:optplan_correct}
Let $\pi, \Vub, \Vlb$ be the policy and the value function bounds returned by \FOptPlan 
with inputs $n, \wh r, \widehat{r^2}, \wh P, \wh \epsilon$ after any number of episodes $k$. Then, in event $E$ (defined in Lemma~\ref{lem:goodprob}), the following hold. \begin{enumerate}
\item  The policy $\pi$ is greedy with respect to $\Vub$ and satisfies for all $h \in [H]$
\begin{align}
    \Vlb_h \leq V^{\pi}_h \leq V^\star_h \leq \Vub_{h}.
\end{align}
\item The same chain of inequalities also holds for the Q-estimates used in \FOptPlan, i.e.,
    $$\Qlb_{h} \leq Q^{\pi}_{h} \leq Q^\star_{h} \leq  \Qub_{h}.$$
\end{enumerate}
\end{lemma}
\begin{proof} 
    We show the statement by induction over $h$ from $H+1$ to $1$. For $h = H+1$, the statement holds for the value functions $\Vlb_{H+1}, \Vub_{H+1}$ by definition. We now assume it holds for $h+1$.
    Due to the specific values of $\psi_h$ in \FOptPlan, we can apply Lemmas~\ref{lem:validlowerboundall} and~\ref{lem:validupperbound} and get that $\Qlb_{h} \leq Q^{\pi}_{h} \leq Q^\star_{h} \leq  \Qub_{h}$. Taking the maximum over actions, gives that $\Vlb_h \leq V^{\pi}_h \leq V^\star_h \leq \Vub_{h}$. Hence, the claim follows. The claim that the policy is greedy with respect to $\Vub$ follows from the definition $\pi(s,h) \in \argmax_{a} \Qub_{h}(s.a)$. 
\end{proof}

\begin{lemma}[Lower bounds admissible]
\label{lem:validlowerboundall} 
Let $\pi, \Vub, \Vlb$ be the policy and the value function bounds returned by \FOptPlan 
with inputs $n, \wh r, \widehat{r^2}, \wh P, \wh \epsilon$ after any number of episodes $k$. 
Consider $h \in [H]$ and $x \in \Xcal$ and assume that
$\Vub_{h+1} \geq V^\star_{h+1} \geq V^{\pi}_{h+1} \geq \Vlb_{h+1}$ and that the confidence bound width is at least
\begin{align}
\psi_h(x) \geq &~
            4\left(\sqrt{\widehat{\operatorname{Var}}(r|x)} + \sigma_{\wh P(x)}(\Vub_{h+1})  + 2 \sqrt{\wh \epsilon(x)} H\right)\phi(n(x))
    + 53 \wh \numS H V^{\max}_{h+1}(x) \phi(n(x))^2\\
    &\qquad \quad +\frac{1}{H} \wh P(x) (\Vub_{h+1} - \Vlb_{h+1})
    + (H+ 1) \wh \epsilon(x).
\end{align}
 Then, in event $E$ (defined in Lemma~\ref{lem:goodprob}),  the lower confidence bound at time $h$ is admissible, i.e., 
    $$Q^{\pi}_h(x) \geq \Qlb_{h}(x).$$
\end{lemma}
\begin{proof}
When $\Qlb_h(x) = 0$, the statement holds trivially. Otherwise,
we can decompose the difference of the lower bound and the value function of the current policy as 
\begin{align}
    Q^{\pi}_h(x) - \Qlb_h(x)
    \geq &~
    \markedterm{a}{r(x) - \bar r(x) + (P(x) - \bar P(x))V^\star_{h+1}}
    + \markedterm{b}{(P(x) - \bar P(x))(V^{\pi}_{h+1} - V^\star_{h+1})}\\
    & \quad + \bar P(x)(V^{\pi}_{h+1} - \Vlb_{h+1}) 
    + \markedterm{c}{\bar r(x) - \wh r(x) + (\bar P(x) - \wh P(x)) \Vlb_{h+1}} + \psilb_h(x). \label{eqn:correctness1} 
\end{align}
Note that $\bar P(x)(V^{\pi}_{h+1} - \Vlb_{h+1}) \geq 0$ by assumption. We bound the terms \prnmarker{a}, \prnmarker{b} and \prnmarker{c} separately as follows.
\begin{itemize}
	\item \textbf{Bound on \prnmarker{a}.~} Given that the event $E$ occurs, the events $\mathsf{E}^{\text{RE}}$ and $\mathsf{E}^{\text{VE}}$ also hold (see definition of $E$ in Lemma~\ref{lem:goodprob}). Thus,
\begin{align}
    \qquad \qquad \quad &\hspace{-0.45in} |r(x) - \bar r(x) + (P(x) - \bar P(x))V^\star_{h+1}|\\
    \leq &~ \left(\sqrt{8 \overline{\operatorname{Var}}(r|x)} + 2 \sqrt{\bar P(x)[(V^\star_{h+1} - P(x)V^\star_{h+1})^2]} \right)\phi(n(x)) \\ 
    	& \qquad \qquad + (4.66 V^{\max}_{h+1}(x) + 7.49) \phi(n(x))^2\\
    \overset{(i)}{\leq} &~ \left(\sqrt{8\overline{\operatorname{Var}}(r|x)} + \sqrt{12} \sigma_{\bar P(x)}(\Vub_{h+1}) \right)\phi(n(x)) \\
    & \qquad  + (24H \sqrt{\wh \numS}V^{\max}_{h+1}(x) + 8.13 V^{\max}_{h+1}(x) + 7.49)  \phi(n(x))^2\nonumber + \frac 1 {2H} \bar P(x)(\Vub_{h+1} - \Vlb_{h+1}) \nonumber\\
   \overset{(ii)}{\leq} &~ 
    \left(\sqrt{16 \widehat{\operatorname{Var}}(r|x) + 2 \wh \epsilon(x)} + \sqrt{36 H^2 \wh \epsilon(x) + 12 \sigma^2_{\wh P(x)}(\Vub_{h+1})} \right)\phi(n(x)) \\
   & + (24H \sqrt{\wh \numS}V^{\max}_{h+1}(x) + 8.13 V^{\max}_{h+1}(x) + 7.49)   \phi(n(x))^2\\
      & + \frac 1 {2H} \wh P(x)(\Vub_{h+1} - \Vlb_{h+1})
         + \frac{\bar \epsilon(x)}{2}\\
    \leq &~ \left(4\sqrt{\widehat{\operatorname{Var}}(r|x)} +\sqrt{12} \sigma_{\wh P(x)}(\Vub_{h+1}) \right)\phi(n(x)) \\
   & + (24H \sqrt{\wh \numS}V^{\max}_{h+1}(x) + 8.13 V^{\max}_{h+1}(x) + 7.49)   \phi(n(x))^2\\
       & + \frac 1 {2H} \wh P(x)(\Vub_{h+1} - \Vlb_{h+1})
         + \frac{\bar \epsilon(x)}{2}
         + (6 H + \sqrt{2})\sqrt{\wh \epsilon(x) } \phi(n(x))
         \label{eqn:mainconcentration}
    \end{align}
    where the inequality $(i)$ is given by Lemma~10 in \citet{dann2019policy} and, the inequality $(ii)$ follows from equations~\eqref{eqn:varbiasbound} and~\eqref{eqn:vardebias_to_emp}.

\item \textbf{Bound on \prnmarker{b}. ~} An application of Lemma~17 in \citet{dann2019policy} implies that
    \begin{align}
        \qquad \quad & \hspace{-0.8in} |(P(x) - \bar P(x))(V^{\pi}_{h+1} - V^\star_{h+1})| \\ 
        \leq &~
    (8H + 4.66) \wh \numS V^{\max}_{h+1}(x) \phi(n(x))^2 + \frac{1}{2H} \bar P(x) (V^\star_{h+1} - V^{\pi}_{h+1})\\
    \leq  &~   (8H + 4.66) \wh \numS V^{\max}_{h+1}(x) \phi(n(x))^2 + \frac{1}{2H} \wh P(x) (\Vub_{h+1} - \Vlb_{h+1}) + \frac{\bar \epsilon(x)}{2}
    \end{align}
    where the last inequality uses the assumption that $\Vub_{h+1} \geq V^\star_{h+1} \geq V^{\pi}_{h+1} \geq \Vlb_{h+1}$.
       \item \textbf{Bound on $\prnmarker{c}$.} Note that 
    \begin{align}
        |\bar r(x) - \wh r(x) + (\bar P(x) - \wh P(x)) \Vlb_{h+1}|
        \leq \bar \epsilon(x) + (H-1) \bar \epsilon(x) = H \bar \epsilon(x). 
    \end{align}
 \end{itemize}
\noindent
    Plugging the above bounds back in \eqref{eqn:correctness1}, we get
    \begin{align}
        Q^{\pi}_h(x) - \Qlb_h(x)  \geq &~
    - \frac{1}{H}\wh P(x) (\Vub_{h+1} - \Vlb_{h+1})
    - 4\left(\sqrt{\widehat{\operatorname{Var}}(r|x)} + \sigma_{\wh P(x)}(\Vub_{h+1}) \right)\phi(n(x))
    \nonumber
    \\
    &- 53 \wh \numS H V^{\max}_{h+1}(x) \phi(n(x))^2 - (H + 1) \wh \epsilon(x) - 8H\sqrt{ \wh \epsilon(x)}  \phi(n(x)) + \psi_h(x)
    \end{align}
    which is non-negative by our choice of $\psi_h(x)$.
\end{proof}

\begin{lemma}[Upper bounds admissible]
\label{lem:validupperbound}
Let $\pi, \Vub, \Vlb$ be the policy and the value function bounds returned by \FOptPlan 
with inputs $n, \wh r, \widehat{r^2}, \wh P, \wh \epsilon$ after any number of episodes $k$. 
Consider $h \in [H]$ and $x \in \Xcal$ and assume that
$\Vub_{h+1} \geq V^\star_{h+1} \geq V^{\pi}_{h+1} \geq \Vlb_{h+1}$ and that the confidence bound width is at least
\begin{align}
    \psi_h(x) \geq &~
            4\left(\widehat{\operatorname{Var}}(r|x) + 2H\sqrt{\bar \epsilon(x)} + \sigma_{\wh P(x)}(\Vub_{h+1}) \right)\phi(n(x)) 
            + 40  \sqrt{\wh \numS} H V^{\max}_{h+1}(x) \phi(n(x))^2\\
    &\qquad + \frac{1}{2H} \wh P(x) (\Vub_{h+1} - \Vlb_{h+1})
    + (H+ 1/2) \wh \epsilon(x).
\end{align}
 Then, in event $E$ (defined in Lemma~\ref{lem:goodprob}), the upper confidence bound at time $h$ is admissible, i.e.,
    $$Q^\star_h(x) \leq \Qub_{h}(x).$$
\end{lemma}
\begin{proof}
When $\Qub_h(x) = Q^{\max}_h(x)$, the statement holds trivially. Otherwise,
we can decompose the difference of the upper bound and the optimal Q-function as
\begin{align}
    \Qub_h(x) - Q^\star_h(x)
    \geq &~
    \markedterm{a}{\bar r(x) - r(x) + (\bar P(x) - P(x))V^\star_{h+1}}
    + \wh P(x)(\Vub_{h+1} - V^\star_{h+1}) \\
    &+ \markedterm{c}{\wh r(x) - \bar r(x) + (\wh P(x) - \bar P(x)) V^\star_{h+1}} + \psi_h(x).
\end{align}
Note that by assumption $\wh P(x)(\Vub_{h+1} - V^\star_{h+1}) \geq 0$. The term, \prnmarker{a} is bound using Equation~\eqref{eqn:mainconcentration} in Lemma \ref{lem:optplan_correct}  and the bias terms $\prnmarker{c}$ is bound as
\begin{align}
        |\bar r(x) - \wh r(x) + (\bar P(x) - \wh P(x)) V^\star_{h+1}|
        \leq \bar \epsilon(x) + (H-1) \bar \epsilon(x) = H \bar \epsilon(x). 
    \end{align} Thus, 
\begin{align}
    \Qub_h(x) - Q^\star_h(x)
    \geq &~
    -4\left(\widehat{\operatorname{Var}}(r|x) + 2H\sqrt{\bar \epsilon(x)} + \sigma_{\wh P(x)}(\Vub_{h+1}) \right)\phi(n(x)) - 40  \sqrt{\wh \numS} H V^{\max}_{h+1}(x) \phi(n(x))^2
    \nonumber
    \\
       & \qquad - \frac 1 {2H} \wh P(x)(\Vub_{h+1} - \Vlb_{h+1})
        - \frac{\bar \epsilon(x)}{2}
    - H \bar \epsilon(x) + \psi_h(x) = \psi_h(x) - \psiub_h(x),
\end{align}
which is non-negative by our choice for $\psi_h$.
\end{proof}

\subsection{Tightness of Optimistic Planning}
\label{app:modelbasedproofs_tightness}

\begin{lemma}[Tightness of Optimistic Planning]
\label{lem:optplan_tightness}
Let $\pi, \Vub$ and $\Vlb$ be the output of \FOptPlan with inputs $n, \wh r, \widehat{r^2}, \wh P$ and $\wh \epsilon$ after any number of episodes $k$. In event $E$ (defined in Lemma~\ref{lem:goodprob}), we have for all $s \in \statespace, h \in [H]$, 
\begin{align}
    \Vub_h(s) - \Vlb_h(s) \leq 
    \sum_{x \in \Xcal} \sum_{t = h}^H \left(1 + \frac 3 H \right)^{2t} w_{t}(x)\left[Q^{\max}_t(x) \wedge (\gamma_t(x) \phi(n(x)) + \beta_t(x) \phi(n(x))^2 + \alpha \wh \epsilon(x))\right]\nonumber
\end{align}
where $\gamma_t(x) = 8\left(\sqrt{2\overline{\operatorname{Var}}(r|x)} +7\sqrt{\wh \epsilon(x)}H + 2\sigma_{P(x)}(V^\pi_{t+1})\right)$, $\beta_t(x) = 416 \wh \numS H V^{\max}_{t+1}(x)$, $\alpha = 3H+4$, and the weights $w_t(x) = \prob((s_{t},a_t) = x ~ | ~ s_{h} = s, a_{h:H} \sim \pi)$ are the probability of  visiting each state-action pair at time $t$ under policy $\pi$.
\end{lemma}
\begin{proof}
We start by considering the difference of Q-estimates for $h$ at a state-action pair $x \in \Xcal$
\begin{align}
    \Qub_{h}(x) - \Qlb_h(x)
    \leq &~
    2 \psi_h(x) + \wh P(x) (\Vub_{h+1} - \Vlb_{h+1})
    \\
    = &
    \left( 1 + \frac 2 H\right)\wh P(x) (\Vub_{h+1} - \Vlb_{h+1})
    + 106 \wh \numS H V^{\max}_{h+1}(x) \phi(n(x))^2 + (2H+2)\wh \epsilon(x)
    \\
    &+
    8\left(\sqrt{\widehat{\operatorname{Var}}(r|x)} +2\sqrt{\wh \epsilon(x)}H + \sigma_{\wh {P}(x)}(\Vub_{h+1})\right)\phi(n(x))\\
        \leq &
    \left( 1 + \frac 2 H\right)\bar P(x) (\Vub_{h+1} - \Vlb_{h+1})
    + 106 \wh \numS H V^{\max}_{h+1}(x) \phi(n(x))^2 + (3H+4)\wh \epsilon(x)
    \\
    &+
    8\left(\sqrt{2\overline{\operatorname{Var}}(r|x)} +7\sqrt{\wh \epsilon(x)}H + \sigma_{\bar {P}(x)}(\Vub_{h+1})\right)\phi(n(x)), 
\end{align}
where, the equality is given by the definition of $\psi_h$ and the inequality follows by using Equations~\eqref{eqn:varbiasbound} and \eqref{eqn:varemp_to_debias} to remove the biases. Next, using Lemma~11 from  \citet{dann2019policy} to convert the value variance to the variance with respect to the value function of $\pi$, we get, 
\begin{align}
        \Qub_{h}(x) - \Qlb_h(x)
        \leq &~
    \left( 1 + \frac 3 H\right)\bar P(x) (\Vub_{h+1} - \Vlb_{h+1})
    + 410 \wh \numS H V^{\max}_{h+1}(x) \phi(n(x))^2 + (3H+4)\wh \epsilon(x)
     \\
    &+
    8\left(\sqrt{2\overline{\operatorname{Var}}(r|x)} +7\sqrt{\wh \epsilon(x)}H + 2\sigma_{P(x)}(V^{\pi}_{h+1})\right)\phi(n(x))
    \\
            \leq &~
    \left( 1 + \frac 3 H\right)^2  P(x) (\Vub_{h+1} - \Vlb_{h+1})
    + 416 \wh \numS H V^{\max}_{h+1}(x) \phi(n(x))^2 + (3H+4)\wh \epsilon(x) \\
    &+
    8\left(\sqrt{2\overline{\operatorname{Var}}(r|x)} +7\sqrt{\wh \epsilon(x)}H + 2\sigma_{P(x)}(V^{\pi}_{h+1})\right)\phi(n(x)), \label{eq:tightness1}
    \end{align}
where the second inequality follows by using Lemma~17 from \citet{dann2019policy} to substiute $\bar P(x) (\Vub_{h+1} - \Vlb_{h+1})$ by $P(x) (\Vub_{h+1} - \Vlb_{h+1})$. Next, recalling that 
$$ \Vub_h(s) - \Vlb_h(s) = \Qub_h(s,\pi(s,h)) - \Qlb_h(s,\pi(s,h)), $$ and rolling the recursion in equation \eqref{eq:tightness1} from $s$ to $h$, we get,
\begin{align}
    \Vub_h(s) - \Vlb_h(s) \leq 
    \sum_{x \in \Xcal} \sum_{t = h}^H \left(1 + \frac 3 H \right)^{2t} w_{t}(x)[Q^{\max}_t(x) \wedge (\gamma_t(x) \phi(n(x)) + \beta_t(x) \phi(n(x))^2 + \alpha \wh \epsilon(x)], 
\end{align}
where, $\gamma_t(x) = 8\left(\sqrt{2\overline{\operatorname{Var}}(r|x)} +7\sqrt{\wh \epsilon(x)}H + 2\sigma_{P(x)}(V^{\pi}_{t+1})\right)$, $\beta_t(x) = 416 \wh \numS H V^{\max}_{t+1}(x)$ and $\alpha = 3H+4$. The final statement follows by observing that $(1 + 3/H)^{2t} \leq \exp(6)$. 
\end{proof}

\subsection{Proof of the Main Theorem~\ref{thm:cipoc_independencenumber_app}}
In this section, we provide the proof of the desired IPOC bound for  Algorithm~\ref{alg:mb_app}. 

\label{app:modelbasedproofs_main}
\begin{proof}
Throughout the  proof, we consider only outcomes in event $E$ (defined in Lemma~\ref{lem:goodprob}) which occurs with probability at least $1- \delta$. 
Lemma~\ref{lem:optplan_correct} implies that the outputs $\pi_k, \Vub_{k,h}$ and  $\Vlb_{k,h}$ from calls to \FOptPlan during the execution of Algorithm~\ref{alg:mb_app} satisfy 
\begin{align}
    \Vlb_{k,h} \leq V^{\pi_k}_h \leq V^\star_h \leq \Vub_{k,h}
\end{align}
and hence, all the certificates provided by Algorithm~\ref{alg:mb_app} are admissible confidence bounds. Further, Lemma~\ref{lem:optplan_tightness}  shows that the difference between the two value functions returned by \FOptPlan is bounded as
	\begin{align}
   \Vub_{k,1}(s_{k,1}) - \Vlb_{k,1}(s_{k,1})
    &\leq \exp(6) \sum_{x \in \Xcal}\sum_{h=1}^H  w_{k,h}(x)
    \left[Q_h^{\max}(x)  \wedge \Big(\beta_h(x) \phi(n_k(x))^2  \right. \\ 
    & \qquad + \gamma_{k,h}(x) \phi(n_k(x)) + \alpha \wh \epsilon_k(x)\Big)\Big], 
		\label{eqn:sasumbound}
	\end{align}
where, $w_{k,h}(x) = \prob((s_{k,h}, a_{k,h} = x ~|~ \pi_k, s_{k,1})$ denotes the probability of the agent visiting $x$ in episode $k$ at time $h$ given the policy $\pi_k$ and the initial state $s_{k, 1}$,  and $\alpha=3H+4$, $\beta_h(x) = 416 \wh \numS H V^{\max}_{h+1}(x)$ and $\gamma_{k,h}(x) =  8\left(\sqrt{2\overline{\operatorname{Var}}_k(r|x)} +7\sqrt{\wh \epsilon(x)}H + 2\sigma_{P(x)}(V^{\pi_k}_{h+1})\right)$. 
	
We define some additional notation, which will come in handy to control Equation~\eqref{eqn:sasumbound} above. Let $w_k(x) \ldef{} \sum_{h=1}^H w_{k,h}(x)$ denote the (total) expected visits of $x$ in the $k^\text{th}$ episode. Next, for some $\wmin > 0$, to be fixed later, define the following subsets of the state action pairs:
\begin{enumerate}[label=(\roman*)]
\item $L_k$: Set of all state-actions pairs $x$ that have low expected visitation in the $k^\text{th}$ episode, i.e. $$L_k \ldef{} \{ x \in \Xcal \colon w_{k}(x) < \wmin \}.$$ 
\item $U_k$: Set of all state-action pairs that had low observation probability in the past, and therefore have not been observed often enough, i.e. 
	$$U_k \ldef{} \left\{ x \in \Xcal \setminus L_k \colon \sum_{i < k} \sum_{\bar x \in \Xcal} q(\bar x, x) w_i(\bar x) < 4H \ln \frac{1}{\delta'}\right\}. $$
\item $W_k$: Set of the remaining state-action pairs that have sufficient past probability, i.e.   
	$$W_k \ldef{} \left\{ x \in \Xcal \setminus L_k \colon \sum_{i < k} \sum_{\bar x \in \Xcal} q(\bar x, x) w_i(\bar x) \geq 4H \ln \frac{1}{\delta '}\right\}.$$
\end{enumerate}

Additionally, let $Q^{\max}$ denote an upper bound on the value-bounds used in the algorithm for all relevant $x$ at all times in the first $T$ episodes, i.e.,    \begin{align}
        Q^{\max} &\geq \max_{k \in [T], h \in [H]} \max_{x \colon w_{k,h}(x) > 0} Q^{\max}_h(x)\qquad \textrm{and,}\\
        Q^{\max} &\geq\max_{k \in [T], h \in [H]} \max_{x \colon w_{k,h}(x) > 0} V^{\max}_{h+1}(x). 
    \end{align}

Next, we bound Equation~\eqref{eqn:sasumbound} (above) by controlling the right hand side separately for each of the above classes. For $L_k$ and $U_k$, we will use the upper bound $Q^{\max}$ and for the set $W_k$, we will use the bound $\beta_h(x) \phi(n_k(x))^2 + \gamma_{k,h}(x) \phi(n_k(x))) + \alpha \wh \epsilon_k(x))$. Thus, 
		\begin{align}
		 \sum_{k=1}^T \Vub_{k,1}(s_{k,1}) - \Vlb_{k,1}(s_{k,1}) 
    \leq &~  \exp(6) \Bigg( \markedterm{a}{ \sum_{k=1}^T \sum_{x \in L_k} w_{k}(x) Q^{\max}}
     +    \markedterm{b}{  \sum_{k=1}^T \sum_{x \in U_k} w_{k}(x) Q^{\max}} \\
    & +   \markedterm{c}{  \sum_{k=1}^T \sum_{x \in W_k} \sum_{h=1}^H w_{k,h}(x)
    (\beta_h(x) \phi(n_k(x))^2 + \gamma_{k,h}(x) \phi(n_k(x)) + \alpha \wh \epsilon_k(x))} \Bigg). \label{eq:term_b_thnm8_proof2}
	\end{align}

\noindent	We bound the terms $\prnmarker{a}, \prnmarker{b}$ and $\prnmarker{c}$ separately as follows: 
\begin{enumerate}
\setlength{\itemindent}{-4mm}
\item \textbf{Bound on $\prnmarker{a}$.} Since, for any $x \in L_k$, $w_k(x) < \wmin$ (by definition), we have 
\begin{align}
    Q^{\max} \sum_{k=1}^T  \sum_{x \in L_k} w_{k}(x)  \leq  Q^{\max} T |\Xcal| \wmin. 
\end{align}
\item \textbf{Bound on $\prnmarker{b}$.} By the definition of the set $U_k$, 
 \begin{align}
	 \sum_{k=1}^T  \sum_{x \in U_k} w_k(x) Q^{\max}
	= &~ Q^{\max} \sum_{k=1}^T \sum_{x \in \Xcal} w_k(x) \one\left\{ \sum_{i < k} \sum_{\bar x \in \Xcal} q(\bar x, x) w_i(\bar x) < 4 H  \ln \frac{1}{\delta '} \right\}. \label{eq:term_b_thnm8_proof}
\end{align}

Observe that, for any constant $\nu \in (0, 1]$, to be fixed later,  
	\begin{align}
		\sum_{i < k} \sum_{\bar x \in \Xcal} q(\bar x, x) w_i(\bar x) 
		\geq
		\sum_{i < k} \sum_{\bar x \in \Xcal} q(\bar x, x) w_i(\bar x) \one\{q(\bar x, x) \geq \nu\}
		\geq
		\sum_{i < k} \sum_{\bar x \in \Ncal^{-}_{\geq \nu}(x)} w_i(\bar x) \nu,  
		\label{eqn:indicatorexpr}
	\end{align}
where, $\Ncal^{-}_{\geq \nu}(x)$ denotes the of incoming neighbors of $x$ (and $x$ itself) in the truncated feedback graph $G_{\geq \nu}$. Plugging the above in Equation~\eqref{eq:term_b_thnm8_proof}, we get,  
	\begin{align}
	 \sum_{k=1}^T  \sum_{x \in U_k} w_k(x) Q^{\max}
	\leq &~ Q^{\max} \sum_{k=1}^T \sum_{x}w_k(x) 
		\one\left\{ 
		\sum_{i < k} \sum_{\bar x \in \Ncal^{-}_{\geq \nu}(x)} w_i(\bar x)  < \frac{4 H}{ \nu } \ln \frac{1}{\delta '}
		\right\}. 
	\end{align}
		Next, using a pigeon hole argument from Lemma~\ref{lem:mas_pigeonhole} in the above expression, we get,  
	\begin{align}
	 \sum_{k=1}^T  \sum_{x \in U_k} w_k(x) Q^{\max}
		\leq &~ 4 H Q^{\max} \frac{\mas(G_{\geq \nu})}{\nu}  \left( 1 + \ln \frac{1}{\delta '} \right).
	\end{align}
	Since the above holds for any $\nu \in (0, 1]$, taking the the infimum over $\nu$, we get 		\begin{align}
	 \sum_{k=1}^T  \sum_{x \in U_k} w_k(x)Q^{\max}
		\leq &~ 4 H Q^{\max} \bar \mas  \left( 1 + \ln \frac{1}{\delta '} \right), 
	\end{align} where, $\bar \mas \ldef{} \inf_{\nu} \frac{\mas(G_{\geq \nu})}{\nu}$. 

\item \textbf{Bound on $\prnmarker{c}$.} 
Setting $\beta = 410 \wh \numS Q^{\max} H$, we get, 
	\begin{align}
		\prnmarker{c}  &\leq 
		 \beta \sum_{k=1}^T \sum_{x \in W_k} w_k(x)
		 \phi(n_k(x))^2 
		 + 
		 \sum_{k=1}^T \sum_{x \in W_k}\sum_{h=1}^H w_{k,h}(x)  \gamma_{k,h}(x) \phi(n_k(x)) \\ 
		  & \qquad   + \alpha \sum_{k=1}^T \sum_{x \in W_k} w_k(x) \wh \epsilon_k(x)\\
		 & \circledmarked{1}{\lesssim} 
		 \beta \sqrt{\ln(HT)} \sum_{k=1}^T \sum_{x \in W_k} w_k(x)
		 \phi(n_k(x))^2 
		 + 
		 \sum_{k=1}^T \sum_{x \in W_k}\sum_{h=1}^H w_{k,h}(x)  \tilde \gamma_{k,h}(x) \phi(n_k(x))
		    + \epsilon_{\max}H^2T\\
		& \circledmarked{2}{\lesssim} 
		 \beta \sqrt{\ln(HT)} \markedterm{d}{\sum_{k=1}^T \sum_{x \in W_k} w_k(x)
		 \phi(n_k(x))^2} \\
		 &+ 
		 \sqrt{\markedterm{e}{\sum_{k=1}^T \sum_{x \in W_k} \sum_{h=1}^H w_{k,h}(x)  \tilde \gamma_{k,h}(x)^2}}
		 \sqrt{\markedterm{d}{\sum_{k=1}^T \sum_{x \in W_k} w_{k}(x) \phi(n_k(x))^2}}+
		 \epsilon_{\max}H^2T.
		 \label{eqn:decomp_thm1}
	\end{align}	
	Where, we use the symbol $\lesssim$ to denote $\leq$ up to multiplicative constants, and the inequality $\prnmarker{1}$ follows by bounded $\wh \epsilon_k(x)$ by the largest occurring bias $\epsilon_{\max}$ and using the definition of event $\mathsf{E}^{\text{Var}}$ from Lemma~\ref{lem:goodprob_var} to replace $\gamma_{k,h}(x)$ by $\tilde \gamma_{k,h}(x) = 8\sqrt{2\operatorname{Var}_k(r|x)} + 56\sqrt{\wh \epsilon(x)}H + 16\sigma_{P(x)}(V^{\pi_k}_{h+1})$ while paying for an additional term of order $\sqrt{\ln(n^2 / \delta') / n} \leq \sqrt{\ln(HT)} \phi(n)$.  Since this additional term is multiplied by an additional $\phi(n)$, it only appears in the first term of \eqref{eqn:decomp_thm1}.
	The inequality $\prnmarker{2}$ is given by the Cauchy-Schwarz inequality.  
	
	We bound the terms $\prnmarker{d}$ and $\prnmarker{e}$ separately in the following. 
	\begin{enumerate}
	\setlength{\itemindent}{-2mm}
		\item 
	\textbf{Bound on $\prnmarker{d}$.} The term $\prnmarker{a}$ essentially has the form $\sum_{k=1}^T \sum_{x \in W_k} w_k(x) \frac{\ln \ln n_k(x)}{n_k(x)}$. To make our life easier, we first replace the $\ln \ln n_k(x)$ dependency by a constant. Specifically,
	we upper-bound $\phi(n_k(x))^2$ 
	by a slightly simpler expression
    $\frac{J}{n_k(x)}$ where $J =  0.75 \ln \frac{5.2 \ln (2HT)}{\delta'} \geq 0.52 \times 1.4 \ln \frac{5.2 \ln(e \vee 2n_k(x))}{\delta'} \geq 0.52 (1.4 \ln \ln (e \vee 2n_k(x)) + \ln (5.2/\delta'))$ which replaces the dependency on the number of observations $n_k(x)$ in the log term by the total number of time steps $HT \geq Hk \geq n_k(x)$. This gives
	\begin{align}
		\prnmarker{a} \leq J  \sum_{k=1}^T  \one\{x \in W_k\} \frac{w_{k}(x)}{n_k(x)}.
	\label{eqn:wn1121}
	\end{align} 
By the definition of $W_k$, we know that for all $x \in W_k$ the following chain of inequalities holds
    \begin{align}
	    \sum_{i < k} \sum_{\bar x \in \Xcal} q(\bar x, x)w_i(\bar x) 
	    \geq 4 H \ln \frac{1}{\delta'} 
	    \geq 8 H 
	    \geq 8  \sum_{\bar x \in \Xcal} q(\bar x, x) w_k(\bar x) .
    \end{align}
    The second inequality is true because of the definition of $\delta'$ gives $\frac{1}{\delta'} = \frac{|\Xcal|(4\wh \numS + 5H + 7)}{\delta}$ which is lower bounded by $13 \geq \exp(2)$ because $\delta \leq 1$ and $|\Xcal| \geq 2$.
    Leveraging this chain of inequalities in combination with the definition of event $\mathsf{E}^{\text{N}}$, we can obtain a lower bound on $n_k(x)$  for $x \in W_k$ as
	\begin{align}
		n_k(x) \geq&  \frac 1 2
 \sum_{i < k} \sum_{\bar x \in \Xcal} q(\bar x, x)w_i(\bar x) 
		 - H \ln \frac{1}{\delta'}
		     \geq \frac{1}{4} 
 \sum_{i < k} \sum_{\bar x \in \Xcal} q(\bar x, x)w_i(\bar x) 
		     \geq \frac{2}{9}
 \sum_{i \leq k} \sum_{\bar x \in \Xcal} q(\bar x, x)w_i(\bar x)\nonumber\\
		\geq &\frac{2\nu}{9} 
\sum_{i < k} \sum_{\bar x \in \Ncal^{-}_{\geq \nu}(x)} w_i(\bar x) 
	\end{align}
	where the last inequality follows from \eqref{eqn:indicatorexpr}.
 Plugging this back into \eqref{eqn:wn1121} and applying Lemma~\ref{lem:wsum_mas} gives
	 \begin{align}
		 \prnmarker{a} \leq \frac{9J}{2\nu}  \sum_{k=1}^T \sum_{x \in W_k} \frac{w_{k}(x)}{
			 \sum_{i < k} \sum_{\bar x \in \Ncal^{-}_{\geq \nu}(x)} w_i(\bar x) }
		 \leq \frac{18eJ}{\nu} \operatorname{mas}(G_{\geq \nu}) 
		  \ln \left(\frac{eHT}{w_{\min}}\right).
	 \end{align}
	Since this holds for any $\nu$, we get
		 \begin{align}
		 \prnmarker{a} 
		 \leq 18eJ \bar \mas 
		  \ln \left(\frac{eHT}{w_{\min}}\right).
	 \end{align}
	 
\item  \textbf{Bound on $\prnmarker{e}$.} Using the law of total variance for value functions in MDPs (see Lemma~4 in \citet{dann2015sample} or see \citet{azar2012sample, lattimore2012pac} for the discounted setting), we get, 
	 \begin{align}
	     \sum_{k=1}^T \sum_{x \in \Xcal} \sum_{h=1}^H w_{k,h}(x)  \tilde \gamma_{k,h}(x)^2
	     \lesssim &  \sum_{k=1}^T \sum_{h=1}^H \sum_{x \in \Xcal} w_{k,h}(x) (\operatorname{Var}(r|x) + H^2 \epsilon(x)
	     + \sigma^2_{P(x)}(V^{\pi_k}_{h+1}))\\
	     \leq &  \sum_{k=1}^T \sum_{h=1}^H \sum_{x \in \Xcal} w_{k,h}(x) (\operatorname{Var}(r|x) 
	     + \sigma^2_{P(x)}(V^{\pi_k}_{h+1})) + \epsilon_{\max}H^3T \\
	     \leq & \sum_{k=1}^T \left(\sum_{x \in \Xcal} w_{k}(x)r(x)   +  \operatorname{Var}\left( \sum_{h=1}^H r_h ~ \bigg| ~ a_{1:H} \sim \pi_k, s_{k,1}\right) \right) \\ 
	     & \qquad \quad + \epsilon_{\max}H^3T\\
	     \leq & ~ \sum_{k=1}^T (H+1) \EE\left( \sum_{h=1}^H r_h ~ \bigg| ~ a_{1:H} \sim \pi_k, s_{k,1}\right) + \epsilon_{\max}H^3T\\
         \leq & (H+1) \sum_{k=1}^T  V^{\pi_k}_1(s_{k,1}) + TH^3 \epsilon_{\max},
	 \end{align}
	 where, the above inequalities use the fact that for any random variable $X \leq X_{\max}$ a.s., we have $\operatorname{Var}(X) \leq \EE[X^2] \leq \EE[X] X_{\max}$. 
		\end{enumerate}
\end{enumerate} 

\noindent
Plugging the above developed bounds for the terms $\prnmarker{a}$, $\prnmarker{b}$ and $\prnmarker{c}$ in \eqref{eq:term_b_thnm8_proof2}, we get, 
\begin{align}
	 	\sum_{k=1}^T \Vub_{k,1}(s_{k,1}) - \Vlb_{k,1}(s_{k,1}) 
	 \lesssim &~
	  |\Xcal| Q^{\max} T \wmin + 
		 \bar \mas Q^{\max} H\left( 1 + \ln \frac{1}{\delta '}  \right)
		  +  \beta \sqrt{\ln(HT)} J \bar \mas \ln \left(\frac{eHT}{w_{\min}}\right) 
		 \nonumber\\&+ 
		 \sqrt{J \left(H \sum_{k=1}^T  V^{\pi_k}_1(s_{k,1}) + H^3 \epsilon_{\max} T\right) \bar \mas \ln \left(\frac{eHT}{w_{\min}}\right) }
		 + H^2 T \epsilon_{\max}.
	 \end{align}
	 Setting $\wmin = \frac{1}{ Q^{\max}|\Xcal| T}$ gives 
	 	 \begin{align}
	 \sum_{k=1}^T \Vub_{k,1}(s_{k,1}) - \Vlb_{k,1}(s_{k,1}) 
	 = &~
	O\left(\sqrt{\bar \mas H \sum_{k=1}^T  V^{\pi_k}_1(s_{k,1})} \ln \frac{|\Xcal| H T}{\delta} 
	 + \bar \mas \wh \numS Q^{\max} H \ln^3 \frac{|\Xcal| H T}{\delta}\right)\\
	 &+ O\left(\sqrt{\bar \mas H^3 T  \epsilon_{\max}} \ln \frac{|\Xcal| H T}{\delta} + H^2 T \epsilon_{\max} \right).
	 \end{align} 
	\end{proof}

\subsection{Sample Complexity Bound for Algorithm~\ref{alg:mb} and Algorithm~\ref{alg:mb_app}}
 \label{sec:sample_complex_proof}
 For convenience, we here restate the sample-complexity bound of Algorithm~\ref{alg:mb} from Section~\ref{sec:sample_complex}. 

\mbpacbound*
\begin{proof}
This Corollary is a special case of Proposition~\ref{prop:general_samplecomplexity_mb} below.
We simply set $\gamma = 1$ and the quantities $\bar V(\bar T) = H$ and $Q^{\max} = H$ to their worst-case values. Note also that $\mas = \bar \mas$ in deterministic feedback graphs. Then $\bar T$ in Proposition~\ref{prop:general_samplecomplexity_mb} evaluates to
    \begin{align}
     \bar T = O\left( \frac{\mas H^2}{\epsilon^2} \ln ^2 \frac{ |\Xcal| H}{\epsilon \delta} + \frac{\mas \wh \numS H^2}{\epsilon} \ln ^3 \frac{|\Xcal| H}{\epsilon \delta} \right)
    \end{align}
    which is the desired sample-complexity.
\end{proof}

\begin{proposition}[Sample-Complexity of Algorithm~\ref{alg:mb_app}]
     Consider any tabular episodic MDP with state-action pairs $\Xcal$, episode length $H$ and stochastic independent directed feedback graph $G$ that provides unbiased observations ($\epsilon_{\max} = 0$). Then, with probability at least $1 - \delta$, for all $\epsilon > 0$ and $\gamma \in \NN$ jointly, Algorithm~\ref{alg:mb_app} outputs $\gamma$ certificates that are smaller than $\epsilon$ after at most
    \begin{align}
     \bar T = O\left( \frac{\bar \mas V(\bar T) H}{\epsilon^2} \ln ^2 \frac{ |\Xcal| H}{\epsilon \delta} + \frac{\bar \mas \wh \numS H Q^{\max}}{\epsilon} \ln ^3 \frac{|\Xcal| H}{\epsilon \delta} + \gamma \right)
    \end{align}
    episodes where $\bar V(T) \geq \frac 1 T \sum_{k=1}^T  V^{\pi_k}_1(s_{k,1}) \leq \frac 1 T \sum_{k=1}^T  V^{\star}_1(s_{k,1}) \leq H$ is a bound on the average expected return achieved by the algorithm during those episodes and can be set to $H$.
    \label{prop:general_samplecomplexity_mb}
\end{proposition}
\begin{proof}
Let $\epsilon_k = \Vub_{k,1}(s_{k,1}) - \Vlb_{k,1}(s_{k,1})$ be the size of the certificate output by Algorithm~\ref{alg:mb_app} in episode $k$.  
By Theorem~\ref{thm:cipoc_independencenumber_app}, the cumulative size after $T$ episodes is with high probability $1-\delta$ bounded by
\begin{align}
    	\sum_{k=1}^T \epsilon_k	\leq 
    		&O\left(\sqrt{\bar \mas H \bar V(T) T} \ln \frac{|\Xcal| H T}{\delta} 
	 + \bar \mas \wh \numS Q^{\max} H \ln^3 \frac{|\Xcal| H T}{\delta}\right).
\end{align}
Here, $\bar V(T) \geq \frac 1 T \sum_{k=1}^T  V^{\pi_k}_1(s_{k,1})$ is any non-increasing bound that holds in the high-probability event on the average initial values of all policies played. We can always set $\bar V(T) = H = O(1)$ but there may be smaller values appropriate if we have further knowledge of the MDP (such as the value of the optimal policy).

If the algorithm has not returned $\gamma$ certificates of size at most $\epsilon$ yet, then $\sum_{k=1}^T \epsilon_k > (T - \gamma) \epsilon$. Combining this with the upper bound above gives
\begin{align}
    \epsilon < \frac{\sqrt T}{T - \gamma} \sqrt{c\bar \mas H \bar V(T)} \ln \frac{|\Xcal| H T}{\delta} 
	 + \frac{c \bar \mas \wh \numS Q^{\max} H}{T - \gamma} \ln^3 \frac{|\Xcal| H T}{\delta}
\end{align}
for some absolute constant $c$.
Since the expression on the RHS is monotonically decreasing, it is sufficient to find a $\bar T$ such that
\begin{align}
    \frac{\sqrt{\bar T}}{\bar T - \gamma} \sqrt{c \bar \mas H \bar V(\bar T)} \ln \frac{|\Xcal| H \bar T}{\delta}  \leq \frac{\epsilon}{2} \qquad \textrm{and}\qquad
    \frac{c \bar \mas \wh \numS Q^{\max} H}{\bar T - \gamma} \ln^3 \frac{|\Xcal| H \bar T}{\delta}  \leq \frac{\epsilon}{2}.
\end{align}
to guarantee that the algorithm has returned $\gamma$ certificates of size at most $\gamma$ after $\bar T$ episodes.
Consider the first condition for $\bar T$ that satisfies
\begin{align}
   2 \gamma \vee  \bar c \frac{\bar \mas V(\bar T) H}{\epsilon^2} \ln ^2 \frac{\bar c |\Xcal| H}{\epsilon \delta} 
     \leq \bar T 
     \leq  \left[\frac{{\bar c} |\Xcal| H}{\epsilon \delta} \right]^5
     \label{eqn:barTcond1}
\end{align}
for some constant $\bar c$ large enough ($\bar c\geq 3456 c$ sufficies).
A slightly tedious computation gives
\begin{align}
    \frac{\sqrt{\bar T}}{\bar T - \gamma} \sqrt{c \bar \mas H \bar V(\bar T)} \ln \frac{|\Xcal| H \bar T}{\delta}
    &\leq 2 \sqrt{\frac{c \bar \mas H \bar V(\bar T)}{\bar T} \ln^2 \frac{|\Xcal| H \bar T}{\delta}} \\ 
    &\leq  \sqrt{\frac{\epsilon^2}{4 \cdot 6^2} \frac{\ln^2 \frac{|\Xcal| H \bar T}{\delta}}{\ln ^2 \frac{\bar c |\Xcal| H}{\epsilon \delta}}} =  \frac{\epsilon}{2} \cdot  \frac{\ln \frac{|\Xcal| H}{\delta} + \ln \bar T}{\ln \frac{|\Xcal| H}{\delta} + \ln \frac{\bar c^6 |\Xcal|^5 H^5}{\epsilon^6 \delta^5}}
\end{align}
and by the upper-bound condition in \eqref{eqn:barTcond1}, the RHS cannot exceed $\frac{\epsilon}{2}$.
Consider now the second condition for $\bar T$ that satisfies
\begin{align}
       2 \gamma \vee  \bar c \frac{\bar \mas \wh \numS H O^{\max}}{\epsilon} \ln ^3 \frac{\bar c |\Xcal| H}{\epsilon \delta} 
     \leq \bar T 
     \leq \left[\frac{{\bar c} |\Xcal| H}{\epsilon \delta} \right]^5
     \label{eqn:barTcond2}
\end{align}
which yields
\begin{align}
    \frac{c \bar \mas \wh \numS Q^{\max} H}{\bar T - \gamma} \ln^3 \frac{|\Xcal| H \bar T}{\delta}
    & \leq 
    \frac{2 c \bar \mas \wh \numS Q^{\max} H}{\bar T} \ln^3 \frac{|\Xcal| H \bar T}{\delta} \\ 
    & \leq
    \frac{\epsilon}{2} \cdot \frac{\ln^3 \frac{|\Xcal| H \bar T}{\delta}}{4 \cdot 6^3  \ln ^3 \frac{\bar c |\Xcal| H}{\epsilon \delta}} =
    \frac{\epsilon}{2} \cdot \left[\frac{\ln \frac{|\Xcal| H}{\delta} + \ln \bar T}{\ln \frac{|\Xcal| H}{\delta}
    + \ln \frac{\bar c^6 |\Xcal|^5 H^5}{\epsilon^6 \delta^5}}
    \right]^3.
\end{align}
Hence, we have shown that if $\bar T$ satisfies the conditions in \eqref{eqn:barTcond1} and \eqref{eqn:barTcond2}, then the algorithm must have produced at least $\gamma$ certificates of size at most $\epsilon$ within $\bar T$ episodes.
By realizing that we can pick
\begin{align}
    \bar T = 2 \gamma + \bar c \frac{\bar \mas V(\bar T) H}{\epsilon^2} \ln ^2 \frac{\bar c |\Xcal| H}{\epsilon \delta} + \bar c \frac{\bar \mas \wh \numS H O^{\max}}{\epsilon} \ln ^3 \frac{\bar c |\Xcal| H}{\epsilon \delta} \leq \left[\frac{{\bar c} |\Xcal| H}{\epsilon \delta} \right]^5,
\end{align}
as long as $\gamma$ is not significantly larger than the following quantities, the statement to show follows.
\end{proof}

\section{Technical Lemmas on Sequences on Vertices of a Graph}
\label{app:graphproofs}
In this section, we present several technical results that form the foundation for our performance bounds in terms of feedback graph properties. We begin with bounds on self-normalizing sequences on vertices. Lemma~\ref{lem:1overxsum_masbound} provides a bound for vertex-values sequences, which we then generalize to 
integer-valued vector sequences in
Lemma~\ref{lem:wsum_discrete_mas} and to real-values vector sequences in Lemma~\ref{lem:wsum_mas}. 
Finally, Lemma~\ref{lem:mas_pigeonhole} gives a bound on a cumulative thresholded process defined over vertices. These results may be of interest beyond the analysis of our specific algorithms and are therefore provided separately.

\begin{lemma}[Bound on self-normalizing vertex sequences]
Let $G = (\Xcal, \Ecal)$ be a directed graph and $x \in \Xcal^T$ be a vector of length $T$ taking values in $\Xcal$. Then
\begin{align}
\sum_{k=1}^T \frac{1}{\sum_{i \in  [k]} \sum_{x' \in \Ncal_G(x_k)} \one\{ x_i = x'\}} \leq \mas(G) \ln (eT),
\label{eqn:1overxum}
\end{align}
where $\Ncal_G(x) = \{x\} \cup \{ x' \in \Xcal \colon (x', x) \in \Ecal\}$ are all incoming neighbors of $x$ and $x$ itself.
\label{lem:1overxsum_masbound}
\end{lemma}

\begin{proof}
The proof works by re-ordering the sum over $T$ in groups based on the graph structure.
Consider any mapping $\ell$ of indices to groups that satisfies
$\ell(k) = \min\{ l \in [T] ~ : ~ \forall i < k ~ \ell(i) = l \Rightarrow x_i \notin \Ncal_G(x_k) \}$ which can be constructed inductively. It assigns each index to the smallest group that does not already contain an earlier incoming neighbor. This assignment has two convenient properties:
\begin{itemize}
	\item There can be at most $\mas(G)$ indices be assigned to a group because otherwise the subgraph of the associated vertices contains a cycle. If there were a cycle then there would be an index in that cycle that is the child of an earlier index. This violates the definition of $\ell$.
	\item For all occurrences it holds that $\sum_{i \leq k}  \one\{x_i \in \Ncal_G(x_k\} \geq \ell(k)$. This is true because in all layers $l < \ell(k)$ there must be at least one earlier index that is a parent. Otherwise $\ell(k)$ would be $l$ instead. 
\end{itemize}
We now leverage both properties to bound the left hand side of Equation~\eqref{eqn:1overxum} as
\begin{align}
		(\textrm{LHS of }\ref{eqn:1overxum}) = &\sum_{l =1}^{T} \sum_{k = 1}^T \frac{\one\{\ell(k) = l\}}
		{\sum_{i=1}^k  \one\{x_i \in \Ncal_G(x_k\}}
		\leq 		\sum_{l =1}^{T} \sum_{k=1}^T \frac{\one\{\ell(k) = l\}}
		{l} 
		\leq \sum_{l =1}^{T} \frac{\mas(G)}{l} \leq  \mas(G)\ln(e T),
\end{align} where the last inequality comes from a bound on the harmonic number $\sum_{i=1}^T 1/i \leq \ln(T) + 1 = \ln(eT)$. This grouping argument bears resemblance with the argument by \citet{lykouris2019graph}.
\end{proof}

\begin{lemma}[Bound on self-normalizing integer-valued sequences]
Let $G = (\Xcal, \Ecal)$ be a directed graph defined on a finite vertex set $\Xcal$ with a maximum acyclic subgraph of size $\mas(G)$ and let $(w_k)_{k \in [T]}$ be a sequence of bounded integer weight functions $w_k ~:~ \Xcal \rightarrow \{0\} \cup [W]$. The following quantity is bounded from above as
	\begin{align}
		\sum_{k=1}^T \sum_{x \in \Xcal} \frac{w_k(x)}{\sum_{i=1}^k \sum_{x' \in \Ncal_G(x)} w_i(x')} 
		\leq
		\mas(G) \ln \left(e \sum_{x}\sum_{k=1}^T w_k(x)\right)
		\label{eqn:wsumdiscmas}
	\end{align}
	where $\Ncal_{G}(x) = \{x\} \cup \{ y \in \Xcal ~:~ (y,x) \in \Ecal\}$ is the set of all neighbors pointing to $x$ (and $x$ itself) in $G$. 
	\label{lem:wsum_discrete_mas}
\end{lemma}

\begin{proof}
We will first reduce this statement to the case where all weights are binary by extending the length of the sequence by a factor of at most $W$. For each index $k$ and value $m \in [W]$ define the weights $\bar w_{W(k-1)+m}(x) = \one\{ w_k(x) \geq m\}$. Each original index $k$ corresponds now to a block of $W$ indices of which the first $w_k(x)$ are set to $1$. Then we rewrite the quantity of interest in terms of these binary weights as
\begin{align}
(\textrm{LHS of }\ref{eqn:wsumdiscmas})
= &
\sum_{k=1}^T \sum_{m=1}^W \sum_{x \in \Xcal} \frac{\bar w_{(k-1)W +m}(x)}
		{\sum_{i=1}^k \sum_{x' \in \Ncal_G(x)} \sum_{m=1}^W \bar w_{(i-1)W +m}(x')}  \\
		\leq &
\sum_{k=1}^{WT} \sum_{x \in \Xcal} \frac{\bar w_{k}(x)}
		{\sum_{i=1}^k \sum_{x' \in \Ncal_G(x)}  \bar w_{i}(x')}.
		\label{eqn:wsumbinary}
\end{align}
The inequality holds because we have only changed the indexing but both sides are identical except that the right-hand side potentially contains up to $W$ fewer terms in the denominator per $x \in \Xcal$.

Let now $\Ocal$ be the set of all occurrences of $\bar w_k(x) > 0$ and with slight abuse of notation denote by $k(o)$ and $x(o)$ the index and vertice of the occurrence. Note that the total number of occurrences is bounded $|\Ocal| = \bar T := \sum_{x}\sum_{k=1}^T w_k(x) \leq |\Xcal|WT$. Further, consider any total order of this set that satisfies
$o \leq o'$ implies $k(o) \leq k(o')$ for any $o, o' \in \Ocal$ (i.e., order respects index order but occurrences at the same index can be put in any order). We then rewrite \eqref{eqn:wsumbinary} in terms of occurrences
\begin{align}
\eqref{eqn:wsumbinary}
		\leq
		\sum_{o \in \Ocal} \frac{1}
		{\sum_{o' \leq o}  \one\{x(o') \in \Ncal_G(x(o))\}}.
		\label{eqn:osum}
\end{align}
The inequality holds because the denominator on the right-hand side includes all occurrences of all incoming neighbors at previous indices (but might not count occurrences of neighbors at the current index).
Let $X \in \Xcal^{\bar T}$ be the vertex-valued sequence of these ordered occurrences, that is, $X = [x(o_1), \dots, x(o_{\bar T})]$ for $o_{1} < \dots < o_{\bar T}$ and apply Lemma~\ref{lem:1overxsum_masbound}. This gives the desired bound
\begin{align}
		\eqref{eqn:osum} \leq  \mas(G)\ln(e \bar T) = \mas(G) \ln \left(e \sum_{x}\sum_{k=1}^T w_k(x)\right).
\end{align}
\end{proof}

\begin{lemma}[Bound on self-normalizing real-valued sequences, Restatement of Lemma~\ref{lem:wsum_mas_main}]
Let $G = (\Xcal, \Ecal)$ be a directed graph defined on a finite vertex set $\Xcal$ with a maximum acyclic subgraph of size $\mas(G)$ and let $(w_k)_{k \in [T]}$ be a sequence of non-negative weight functions $w_k ~:~ \Xcal \rightarrow \RR^+$ which satisfy for all $k$ that $\sum_{x \in \Xcal}w_k(x) \leq w_{\max}$. For any $w_{\min} > 0$, the following quantity is bounded from above as
	\begin{align}
		\sum_{k=1}^T \sum_{x \in \Xcal} \frac{\one\{w_k(x) \geq w_{\min}\} w_k(x)}{\sum_{i=1}^k \sum_{x' \in \Ncal_G(x)} w_i(x')} 
		\leq
		2\mas(G) \ln \left(\frac{eT w_{\max}}{w_{\min}}\right) 
	\end{align}
	where $\Ncal_{G}(x) = \{x\} \cup \{ y \in \Xcal ~:~ (y,x) \in \Ecal\}$ is the set of all neighbors pointing to $x$ (and $x$ itself) in $G$. 
	\label{lem:wsum_mas}
\end{lemma}

\begin{proof}
Without loss of generality, we can assume that all weights take values in $\{0\} \cup [w_{\min}, w_{\max}]$ and ignore the indicator in the numerator. This is because
\begin{align}
\sum_{k=1}^T \sum_{x \in \Xcal} \frac{\one\{w_k(x) \geq w_{\min}\} w_k(x)}{\sum_{i=1}^k \sum_{x' \in \Ncal_G(x)} w_i(x')}
\leq \sum_{k=1}^T \sum_{x \in \Xcal} \frac{\one\{w_k(x) \geq w_{\min}\} w_k(x)}{\sum_{i=1}^k \sum_{x' \in \Ncal_G(x)} \one\{w_i(x') \geq w_{\min}\} w_i(x')}.
\end{align}
We define a new set of integer-values weights $\hat w_k(x) = \left\lfloor \frac{w_k(x)}{w_{\min}} \right\rfloor$. These new weights have several convenient properties.
First, $\hat w_k(x)$ are integers bounded by $\frac{w_{\max}}{w_{\min}}$. 
Second, their total sum is nicely bounded as $\sum_{k=1}^T \sum_{x \in \Xcal} \hat w_k(x) \leq \frac{T w_{\max}}{w_{\min}}$. 
Third, from the assumption that $w_k(x) \in \{0\} \cup [w_{\min}, w_{\max}]$, it follows that $\hat w_k(x) \in \{0\} \cup \left[1, \frac{w_{\max}}{w_{\min}}\right]$. This implies that
\begin{align}
\frac{w_k(x)}{2w_{\min}} \leq \hat w_k(x) \leq \frac{w_k(x)}{w_{\min}}
\end{align}
as the flooring has the largest relative effect when $\frac{w_k(x)}{w_{\min}} \nearrow 2$. Rearranging terms, we get $w_{\min} \hat w_k(x) \leq w_k(x) \leq 2 w_{\min} \hat w_k(x)$. We now use this relationship to exchange the original weights with the discretized weights and only pay a factor of 2. Specifically,
\begin{align}
		\sum_{k=1}^T \sum_{x \in \Xcal} \frac{w_k(x)}{\sum_{i=1}^k \sum_{x' \in \Ncal_G(x)} w_i(x')} 
\leq \sum_{k=1}^T \sum_{x \in \Xcal} \frac{2 w_{\min} \hat w_k(x)}{\sum_{i=1}^k \sum_{x' \in \Ncal_G(x)} w_{\min} \hat w_i(x')} 
\leq 2\mas(G) \ln \left(\frac{eT w_{\max}}{w_{\min}}\right).
\end{align}
The final inequality is an application of Lemma~\ref{lem:wsum_discrete_mas}.
\end{proof}

\begin{lemma}[Restatement of Lemma~\ref{lem:mas_pigeonhole_main}]
	Let $G = (\Xcal, \Ecal)$ be a graph with finite vertex set $\Xcal$ and let $w_k$ be a sequence of weights $w_k \colon \Xcal \rightarrow \RR^+$. For any threshold $C \geq 0$, 
	\begin{align}
		\sum_{x \in \Xcal}\sum_{k=1}^\infty  w_k(x) \one\left\{ \sum_{i=1}^{k} \sum_{x' \in \Ncal_G(x)} \!\!\!w_i(x') \leq C\right\}
		\leq \mas(G) C
	\end{align}
	where $\Ncal_{G}(x) = \{x\} \cup \{ y \in \Xcal \colon (y,x) \in \Ecal\}$ is the set of $x$ and all in-neighbors in $G$.
	\label{lem:mas_pigeonhole}
\end{lemma}
\begin{proof}
	We proceed with an inductive argument that modifies the weight function sequence. To that end, we define $w_k^{(0)} = w_k$ for all $k$ as the first element in this sequence (over sequences of weight functions). We then give the value of interest with respect to $(w^{(t)}_k)_{k \in \NN}$ an explicit name
	\begin{align}
		F^{(t)} = \sum_{x \in \Xcal} \sum_{k=1}^\infty  w^{(t)}_k(x) \one\left\{ \sum_{i \leq k} \sum_{x' \in \Ncal_G(x)} w^{(t)}_i(x') \leq C \right\}. 
	\end{align}
	Let $y^{(t)}(x) = \sum_{k=1}^\infty \one\left\{ \sum_{i \leq k} \sum_{x' \in \Ncal_G(x)} w^{(t)}_i(x') \leq C \right\}$ be the largest index for each $x$ that can have positive weight in the sum. Note that $y^{(t)}(x)$ can be infinity. Let $\hat y^{(t)} = \max_{x \in \Xcal} y^{(t)}$ be the largest index and $x^{(t)} \in \argmax_{x} y^{(t)}(x)$ a vertex that hits the threshold last (if at all). We now effectively remove it and its parents from the graph by setting their weights to $0$. Specifically, define
	\begin{align}
		w^{(t+1)}_k(x) = w^{(t)}_k(x) \one\{ x \notin \Ncal_G(x^{(t)})\}\one\{k \leq \hat y^{(t)}\} \qquad \textrm{for all } k \in \NN
	\end{align}
	as the weight function of the next inductive step. First note that all weights after $\hat y^{(t)}$ can be set to $0$ without affecting $F^{(t)}$ because of how we picked $\hat y^{(t)}$. 
	Second, by the condition in the first indicator, $x \notin \Ncal_G(x^{(t)})$ the total sum of zeroed weights before $\hat y^{(t)}$ is
	\begin{align}
	\sum_{i=1}^{\hat y^{(t)}} \sum_{x' \in \Ncal_G(x^{(t)})} w_i(x') 
	\end{align}
	which can be at most $C$ because $\hat y^{(t)}$ was picked as exactly the index where this bound holds. Hence, $F^{(t+1)}$ can decrease at most by $C + w_{\max}$, i.e., 
	$F^{(t+1)} \geq F^{(t)} - C $.
	We now claim that all weights are $0$ after at most $\mas(G)$ steps. This is true because in each step we zero out the weights of at least one vertex that must have at least one positive weight as well as all its parents. We can do this at most the size of the largest acyclic subgraph. Hence $F^{(\mas(G))} = 0$ and therefore
	\begin{align}
		F^{(0)} \leq F^{(1)} + C \leq \dots \leq F^{(\mas(G))} + \sum_{t=1}^{\mas(G)} C = \mas(G) C 
	\end{align}
	which completes the proof.
\end{proof}

\begin{corollary}
	Let $G = (\Xcal, \Ecal)$ be a graph defined on a finite vertex set $\Xcal$ and let $w_k$ be a sequence of non-negative bounded weight functions $w_k \colon \Xcal \rightarrow [0, w_{\max}]$. For any threshold $C \geq 0$, the following bound holds
	\begin{align}
		\sum_{x \in \Xcal}\sum_{k=1}^\infty  w_k(x) \one\left\{ \sum_{i < k} \sum_{x' \in \Ncal_G(x)} w_i(x') \leq C\right\}
		\leq \mas(G)( C + w_{\max})
		\label{eqn:wcumsum1}
	\end{align}
	where $\Ncal_{G}(x) = \{x\} \cup \{ y \in \Xcal \colon (y,x) \in \Ecal\}$ is the set of $x$ and all its parents in $G$
	\label{cor:mas_pigeonhole}
\end{corollary}
\begin{proof}
	We match the index ranges in front of and within the indicator by increasing the threshold $C$ by the maximum value $w_{\max}$ that the weight can take when the indicator condition is met for the last time 
	\begin{align}
	(\textrm{LHS of }\ref{eqn:wcumsum1})
		\leq \sum_{x \in \Xcal} \sum_{k=1}^\infty  w_k(x) \one\left\{ \sum_{i \leq k} \sum_{x' \in \Ncal_G(x)} w_i(x') \leq C + w_{\max}\right\}.
	\end{align}
	We can now apply Lemma~\ref{lem:mas_pigeonhole}. 
\end{proof}

\section{Proofs for Domination Set Algorithm}
\label{app:domsetproofs}

In this section, we will prove the main sample-complexity bound for Algorithm~\ref{alg:dominatingset}
in Theorem~\ref{thm:samplecomplexity_domset}.
We will do this in two steps:
\begin{enumerate}
    \item We show an intermediate, looser bound with an additional additive $\frac{ \mas H^2}{p_0^2}$ term stated in Theorem~\ref{thm:samplecomplexity_domset_loose} in Section~\ref{sec:domset_proof_loose}.
    \item We prove the final bound in Theorem~\ref{thm:samplecomplexity_domset} based on the intermediate bound in Section~\ref{sec:domset_proof_tight}.
\end{enumerate}

\subsection{Proof of Intermediate Sample-Complexity Bound}
\label{sec:domset_proof_loose}
\begin{theorem}[Sample-Complexity of Algorithm~\ref{alg:dominatingset}, Loose Bound]
	For any tabular episodic MDP with state-actions $\Xcal$, horizon $H$, feedback graph with mas-number $\mas$ and given dominating set $\Xcal_D$ with $|\Xcal_D| = \gamma$ and accuracy parameter $\epsilon > 0$, Algorithm~\ref{alg:dominatingset} returns with probability at least $1 - \delta$ an $\epsilon$-optimal policy after 
\begin{align}
			 O\left(\left(\frac{\gamma H^3}{p_0\epsilon^2}
+ \frac{\gamma \wh \numS H^3}{p_0\epsilon}
+ \frac{\mas \wh \numS H^2}{p_0} + \frac{\mas H^2}{p_0^2}
\right)			 
			 \ln^3 \frac{|\Xcal| H}{\epsilon\delta}\right)
\end{align}
episodes. Here, $p_0 = \min_{i \in [\gamma]} p^{(i)}$ is the expected number of visits to the vertex in the dominating set that is hardest to reach.
\label{thm:samplecomplexity_domset_loose}
\end{theorem}

\begin{proof}

Algorithm~\ref{alg:dominatingset} can be considered an instance of Algorithm~\ref{alg:mb} executed on the extended MDP with two differences:
\begin{itemize}
    \item We choose $\delta/2$ as failure probability parameter in \FOptPlan. The remaining $\frac{\delta}{2}$ will be used later.
    \item We choose the initial states per episode adaptively. This does not impact any of the analysis of Algorithm~\ref{alg:mb} as it allows potentially adversarially chosen initial states.
    \item In the second phase, we do not collect samples with the policy proposed by the \FOptPlan routine but with previous policies.
\end{itemize}
We therefore can consider the same event $E$ as in the analysis of Algorithm~\ref{alg:mb} which still has probability at least $1 - \frac{\delta}{2}$ by Lemmas~\ref{lem:goodprob}. In this event, by Lemma~\ref{lem:optplan_correct} it holds that $\Vlb_h \leq V^{\pi}_h \leq V^\star_h \leq \Vub_h$ for $\Vlb_h, \Vub_h, \pi$ returned by all executions of \FOptPlan.
As a result, the correctness of the algorithm follows immediately as
$\wh \pi$ is guaranteed to be $\epsilon$-optimal in the considered event. It remains to bound the number of episodes collected by the algorithm before returning.

While the regret bound of Algorithm~\ref{alg:mb} in Theorem~\ref{thm:cipoc_independencenumber} does not apply to the second phase, it still holds in the first phase. We can therefore use it directly to bound the number of episodes collected in the first phase.

\paragraph{Length of first phase:} We first claim that the first phase must end when the algorithm encounters a certificate for the chosen task that has size at most $\frac{p_0}{2}$. This is true from the stopping condition in Line~\ref{lin:stop_phase1}. The algorithm removes $i$ from $\Ical$ as soon as $\Vub_1((s_1, i)) \leq 2 \Vlb_1((s_1, i))$. This implies that when the  stopping condition is met
\begin{align}
    \Vlb_1((s_1, i)) \geq \frac{\Vub_1((s_1, i))}{2} \geq \frac{V^{\star}_1((s_1, i))}{2}
    = \frac{p^{(i)}}{2}, 
    \label{eqn:firstphasestop}
\end{align}
where the second inequality follows from the fact that $\Vub_1 \geq V^{\star}_1$ in event $E$.
That means that policy $\pi^{(i)}$ visits node $X_i$ indeed at least $\wh p^{(i)} \geq \frac{p^{(i)}}{2}$ times per episode in expectation.

When the stopping condition is not met, then $\Vub_1((s_1, i)) > 2 \Vlb_1((s_1, i))$ and hence
$\Vub_1((s_1, i)) - \Vlb_1((s_1, i)) > \Vlb_1((s_1, i))$. Note also that $\Vub_1((s_1, i)) - \Vlb_1((s_1, i)) \geq V^{\star}_1((s_1, i)) - \Vlb_1((s_1, i))$ at all times in event $E$.
Combining both lower bounds gives
\begin{align}
    \Vub_1((s_1, i)) - \Vlb_1((s_1, i)) \geq (V^{\star}_1((s_1, i)) - \Vlb_1((s_1, i))) \vee \Vlb_1((s_1, i)) \geq \frac{V^{\star}_1((s_1, i))}{2} = \frac{p^{(i)}}{2}.
    \label{eqn:nostopcond}
\end{align}
Assume the algorithm encounters a certificate that satisfies
\begin{align}
\Vub_1((s_1, i)) - \Vlb_1((s_1, i)) \leq \frac{p_0}{4},
\end{align} 
where $i$ is the task which is about to be executed.
By the task choice of the algorithm,  this implies for any $j \in \Ical$
\begin{align}
\Vub_1((s_1, j)) - \Vlb_1((s_1, j)) \leq 
\Vub_1((s_1, i)) - \Vlb_1((s_1, i)) \leq \frac{p_0}{4} < \frac{p^{(j)}}{2},
\end{align}
where the last inequality follows the definition of $p_0$.
As a result, by contradiction with \eqref{eqn:nostopcond}, all remaining tasks would be removed from $\Ical$. 
Hence, the first phase ends when or before the algorithm has produced a certificate for the chosen task of size $\frac{p_0}{4}$.
By Proposition~\ref{prop:general_samplecomplexity_mb}, this can take at most
\begin{align}
    O \left(
        	\frac{\mas H^2}{p_0^2} \ln^2\frac{ |\Xcal| H}{p_0 \delta} 
	 + \frac{\mas \wh \numS H^2}{p_0} \ln^3 \frac{\ |\Xcal| H}{p_0 \delta}
        \right)
\end{align}
episodes. Note that even though the algorithm operates in the extended MDP, the size of the maximum acyclic subgraph $\mas$ is identical to that of the original feedback graph since all copies of a state-action pair form a clique in the extended feedback graph $\bar G$. Further note that even though the number of states $\bar \numS$ in the extended MDP is larger than in the original MDP by a factor of $(\gamma + 1)$, this factor does not appear in the lower-order term as the number of possible successor states (which can have positive transition probability) are still bounded by $\wh \numS$ in each state-action pair of the extended MDP. It only enters the logarithmic term due to the increased state-action space. 

\paragraph{Length of second phase:} 
We now determine a minimum number of samples per state-action pair that ensures that the algorithm terminates.
By Lemma~\ref{lem:optplan_tightness}, the difference $\Vub_1((s_1, 0)) - \Vlb_1((s_1, 0))$ can be bounded for the case where $\epsilon_{\max} = 0$ by
\begin{align}
 \exp(6) \sum_{x \in \Xcal}\sum_{h=1}^H w_{\pi, h}(x) 
	(H \wedge (\beta \phi(n(x))^2 + \gamma_{h}(x) \phi(n(x))))
	\label{eqn:sasumbounddom}
\end{align}
	with $\beta = 416 \wh \numS H^2$ and $\gamma_{h}(x) = 16\sigma_{P(x)}(V^{\pi}_{h+1}) + 16$ (where we use $Q^{\max} = H$ and $1$ as an upper-bound to $\overline{\operatorname{Var}}(r|x)$). The weights $w_{\pi, h}(x) = \EE_{\pi}\left[\one\{(s_h, a_h) = x\} \right]$ are the probability of $\pi$ visiting each state-action pair at a certain time step $h$.
	This can be upper-bounded by
	\begin{align}
\hspace{0.5in}  & \hspace{-0.5in}   \exp(6) \left(\beta \sum_{x \in \Xcal} w_{\pi}(x) \phi(n(x))^2
 +  \sum_{x \in \Xcal}\sum_{h=1}^H \gamma_h(x) w_{\pi, h}(x) \phi(n(x)) \right) \\
  \leq& ~ 
   \exp(6)\left(  \beta \sum_{x \in \Xcal} w_{\pi}(x) \phi(n(x))^2
 + 
  \sqrt{\sum_{x \in \Xcal}\sum_{h=1}^H \gamma_h^2(x) w_{\pi, h}(x) }
 \sqrt{\sum_{x \in \Xcal} w_{\pi}(x)  \phi(n(x))^2} \right), \label{eqn:simpledom1}
\end{align}
where we used the shorthand notation $w_\pi(x) = \sum_{h=1}^H w_{\pi, h}(x)$ and applied Cauchy-Schwarz in the second step. Assume now that we had at least $\bar n \in \NN$ samples per state-action pair. Then \eqref{eqn:simpledom1} is again upper-bounded by
	\begin{align}
 & \exp(6) \beta H \phi(\bar n)^2 
 + 
  \exp(6)\sqrt{H} \phi(\bar n) \sqrt{\sum_{x \in \Xcal}\sum_{h=1}^H \gamma_h^2(x) w_{\pi, h}(x) }.\label{eqn:test11}
\end{align}
For the remaining term under the square-root,  we use the law of total variance for value functions in MDPs \citep{azar2017minimax, dann2015sample} and bound
	 \begin{align}
    \sum_{x} \sum_{h=1}^H w_{\pi,h}(x) \gamma_{h}(x)^2
	&\leq  2 \times 16^2\sum_{x} \sum_{h=1}^H w_{\pi,h}(x) 
	     +  2 \times 16^2 \sum_{x} \sum_{h=1}^H w_{\pi,h}(x) \sigma^2_{P(x)}(V^{\pi}_{h+1})\\
	     &\leq 2 \times 16^2 (H +  H^2) \leq 4^5 H^2.
	 \end{align}
Plugging this back into \eqref{eqn:test11} gives
	\begin{align}
 & 416 \exp(6) \wh \numS H^3 \phi(\bar n)^2 
 + 
  4^{5/2} \exp(6) H^{3/2} \phi(\bar n)
  \leq \frac{c\wh \numS H^3 \ln \ln \bar n}{\bar n}\ln\frac{|\Xcal| H}{\delta} + \sqrt{\frac{c H^3 \ln \ln \bar n}{\bar n} \ln\frac{|\Xcal| H}{\delta}}
\end{align}
for some absolute constant $c$ where we bounded $\phi(\bar n)^2 \lesssim \frac{\ln \ln \bar n}{\bar n}\ln\frac{|\Xcal| H}{\delta}$. Then there is an absolute constant $\bar c$ so that this expression is smaller than $\epsilon$ for
\begin{align}
    \bar n = \frac{\bar c H^3}{\epsilon^2} \ln^2 \frac{|\Xcal| H}{ \epsilon \delta} + \frac{\bar c \wh \numS H^3}{\epsilon} \ln^2 \frac{|\Xcal| H}{ \epsilon \delta}.
\end{align}
Hence, the algorithm must stop after collecting $\bar n$ samples for each state-action pair. By the property of the dominating set, it is sufficient to collected $\bar n$ samples for each element of the dominating set. Analogously to event $\mathsf{E}^{\mathrm{N}}$ in Lemma~\ref{lem:goodprob}, we can show that with probability at least $1 - \delta / 2$, for all $k$ and $i$, the number of visits to any element of the dominating set $X_i$ are lower-bounded by the total visitation probability so far as
\begin{align}
    v(X_i) \geq \frac 1 2 \sum_{j \leq k}  w_{j}(X_i) - H\ln \frac{2 \gamma}{\delta}, \label{eqn:temp222}
\end{align}
where $k$ is the total number of episodes collected so far and $w_j(X_i)$ is the expected number of visits to $X_i$ of the policy played in the $j$th episode of the algorithm. 
Further, the stopping condition in the first phase was designed so that $\pi^{(i)}$ visits $X_i$ at least $\wh p^{(i)} \geq \frac{p^{(i)}}{2}$ times per episode in expectation (see Equation~\eqref{eqn:firstphasestop}). This follows from the definition of the reward in the extended MDP and the fact that certificates are valid upper and lower confidence bounds on the value function, that is
\begin{align}
\wh p^{(i)} = \Vlb_1((s_1, i)) \geq \frac{\Vub_1((s_1, i))}{2} \geq \frac{V^\star_1((s_1,i))}{2} = \frac{p^{(i)}}{2}.
\end{align}
Hence, if $\pi^{(i)}$ is executed for $m_i$ episodes in the second phase, the total observation probability for $X_i$ is at least $\frac{m_i p^{(i)}}{2}$. Plugging this back in \eqref{eqn:temp222} gives
\begin{align}
        v(X_i) \geq \frac 1 4 m_i p^{(i)}  - H\ln \frac{2\gamma}{\delta} .
\end{align}
Hence, to ensure that the algorithm has visited each vertex of the dominating set sufficiently often, i.e., $\min_{i \in [\gamma]} v(X_i) \geq \bar n$, it is sufficient to play 
\begin{align}
    m_i = O\left(\frac{H^3}{p^{(i)}\epsilon^2} \ln^2 \frac{|\Xcal| H}{ \epsilon \delta} + \frac{\wh \numS H^3}{p^{(i)}\epsilon} \ln^2 \frac{|\Xcal| H}{ \epsilon \delta}\right)
\end{align}
episodes with each policy $\pi^{(i)}$ in the second phase. 
Hence, we get a bound on the total number of episodes in the second phase by summing over $\gamma$, which completes the proof.
\end{proof}

\subsection{Proof of Tighter Sample Complexity Bound Avoiding $1 / p_0^2$}
\label{sec:domset_proof_tight}
The sample complexity proof of Algorithm~\ref{alg:dominatingset} in Theorem~\ref{thm:samplecomplexity_domset_loose} follows with relative ease from the guarantees of Algorithm~\ref{alg:mb}. It does however have a $\tilde O\left( \frac{\mas H^2}{p_0^2} \right)$ dependency which is absent in the lower-bound in Theorem~\ref{thm:domsetlowerbound}. We now show how to remove this additive $\tilde O\left( \frac{\mas H^2}{p_0^2} \right)$ term and prove the main result for Algorithm~\ref{alg:dominatingset} which we restate here:
\domsetalgsc*

Before presenting the formal proof, we sketch the main argument.
The proof of the intermediate result in Theorem~\ref{thm:samplecomplexity_domset_loose} relies on Corollary~\ref{cor:epspolicy} for Algorithm~\ref{alg:mb} to bound the length of the first episode. Yet, Proposition~\ref{prop:general_samplecomplexity_mb} shows that the dominant term of the sample-complexity of  Algorithm~\ref{alg:mb} only scales with $\frac{1}{\epsilon^2}\mas H \frac{1}{T} \sum_{k=1}^T V^\star_1(s_{k,1})$ for some $T$ instead of the looser $\frac{\mas H^2}{\epsilon^2}$ in Corollary~\ref{cor:epspolicy}.
We can upper-bound each summand $V^\star_1(s_{k,1})$ by the optimal value of the task of the episode, e.g., $p^{(i)}$ for task $i$. If all vertices in the dominating set are equally easy to reach, that is, $p^{(1)} = p^{(2)} = \ldots = p^{(\gamma)} = p_0$, this yields $V^\star_1(s_{k,1}) = p_0$ and $\epsilon \approx p_0$. In this case, this term in the sample-complexity evaluates to
\begin{align}
    \frac{\mas H p_0}{p_0^2} \approx \frac{\mas H}{p_0}, 
\end{align}
and gets absorbed into the last term $\frac{\mas \wh \numS H^2}{p_0}$ of the sample-complexity in Theorem~\ref{thm:samplecomplexity_domset}. However, there is a technical challenge when $p^{(i)}$s vary significantly across tasks $i$, i.e., some vertices in the dominating set can be reached easily while others can only be reached with low probability. 
A straightforward bound only yields
\begin{align}
    \frac{\mas H \max_{i \in [\gamma]} p^{(i)}}{p_0^2},
\end{align}
which can be much larger when $\max_{i} p^{(i)} \gg \min_{i} p^{(i)} = p_0$.
To avoid this issue, we will apply a careful argument that avoids a linear factor of the number of policies learned $\gamma$ (which a separate analysis of every task would give us, see Section~\ref{sec:multitaskrl}) while at the same time still only having a $1/p_0$ dependency instead of the $1 / p_0^2$.

The key is an inductive argument that bounds the number of episodes for the $j$ vertices of the dominating set that are the easiest to reach for any $j \in [\gamma]$. 
Thus, assume without loss of generality that vertices are ordered with decreasing reachability, i.e., $p^{(1)} \geq p^{(2)} \geq \dots \geq p^{(\gamma)}$.
We will show that the algorithm plays tasks $1, \dots, j$ in at most 
\begin{align}
    O\left(j + \frac{\mas \wh \numS H^2}{p^{(j)}} \ln^3 \frac{|\Xcal| H}{\delta p_0}\right)
    \label{eqn:task1jcompl}
\end{align}
episodes. For $j = \gamma$, this gives the total length of the first phase and yields the desired reduction in sample complexity for Theorem~\ref{thm:samplecomplexity_domset}. Assuming that this bound holds for $1$ to $j-1$, we consider the subset of episodes $\Kcal_j$ in which the algorithm plays tasks $[j]$ and show the average optimal value in these episodes is not much larger than $p^{(j)}$
\begin{align}
    \frac{1}{|\Kcal_j|} \sum_{k \in \Kcal_j} V^\star_1(s_{k,1}) \lesssim p^{(j)} \ln \frac{e p^{(1)}}{p^{(j)}}.
\end{align}
This insight is the key to prove \eqref{eqn:task1jcompl} for $j$. 

\paragraph{Full proof:}
\begin{proof}[Proof of Theorem~\ref{thm:samplecomplexity_domset}]
The proof of Theorem~\ref{thm:samplecomplexity_domset_loose} can be directly applied here. It yields that with probability at least $1 - \delta$, Algorithm~\ref{alg:dominatingset} returns an $\epsilon$-optimal policy and event $E$ from Lemma~\ref{lem:goodprob} holds. We further know that the algorithm collects at most $T_1$ and $T_2$ episodes in the first and second phase respectively, where
\begin{align}
    T_1 = 
    			 O\left(\left(\frac{\mas H^2}{p_0^2} + \frac{\mas \wh \numS H^2}{p_0}
    			 \right)			 
			 \ln^3 \frac{|\Xcal| H}{\epsilon\delta}\right), \quad \text{and,} \quad 
    T_2 = O\left( \frac{\gamma \wh \numS H^3}{p_0\epsilon}		 
			 \ln^2 \frac{|\Xcal| H}{\epsilon\delta}\right).
\end{align}
It is left to provide a tighter bound for the length of the first phase.
As mentioned above, assume without loss of generality that the nodes of the dominating set are ordered with decreasing reachability, i.e., $p^{(1)} \geq p^{(2)} \geq \dots \geq p^{(\gamma)}$. 
For any $j \in [\gamma]$, let $\Kcal_j \subseteq [T_1]$ be the set of episodes where the algorithm played task $1, \dots, j$. To reason how large this set can be, we need slightly refined versions of the IPOC bound of Algorithm~\ref{alg:mb_app} in Theorem~\ref{thm:cipoc_independencenumber_app} and the corresponding sample-complexity result in Proposition~\ref{prop:general_samplecomplexity_mb}. We state them below as Lemmas~\ref{lem:subsetcumipoc} and \ref{lem:subsetsamplecompl}. They allow us to reason over arbitrary subset of episodes instead of consecutive episodes. Their proof is virtually identical to those of Theorem~\ref{thm:cipoc_independencenumber_app} and Proposition~\ref{prop:general_samplecomplexity_mb}.

As we know from the proof of Theorem~\ref{thm:samplecomplexity_domset_loose}, the algorithm cannot play task $i$ anymore once it has encountered a certificate $\Vub_{1}((s_1, i)) - \Vlb_1((s_1, i)) \leq \frac{p^{(i)}}{4}$. Hence, it can only encounter at most $j$ episodes in $\Kcal_j$ where the certificate was at most $\frac{p^{(j)}}{4}$. 
Thus by Lemma~\ref{lem:subsetsamplecompl} below
\begin{align}
    |\Kcal_j| \leq O\left( j + 1 +
        \mas H  \frac{ \sum_{k \in \Kcal_j}  V^{\pi_k}_1(s_{k,1})}{|\Kcal_j| \cdot (p^{(j)})^2} \ln^2 \frac{|\Xcal| H T_1}{\delta}
    + \frac{ \mas \wh \numS H^2}{p^{(j)}} \ln^3 \frac{|\Xcal| H T_1}{\delta}\right).
    \label{eqn:basebound}
\end{align}
Since $j \leq \gamma$ and we can assume that the provided dominating set is of sufficient quality, i.e., $\gamma \leq \frac{\mas \wh \numS  H^2}{p^{(j)}} \ln^3 \frac{|\Xcal| H T_1}{\delta}$, the $j+1$ term is dominated by the later terms in this bound. We now claim that
\begin{align}
        |\Kcal_j| = O\left(\frac{ \mas \wh \numS H^2}{p^{(j)}} \ln^3 \frac{|\Xcal| H}{p_0 \delta}\right) \label{eqn:induc1}
\end{align}
which we will show inductively. Assume that \eqref{eqn:induc1} holds for all $1, \dots j-1$ and consider the sum of policy values in $\Kcal_j$ from \eqref{eqn:basebound}
\begin{align}
    \sum_{k \in \Kcal_j}  V^{\pi_k}_1(s_{k,1})
    \leq 
     \sum_{k \in \Kcal_j}  V^\star_1(s_{k,1})
     = \sum_{i = 1}^j \sum_{k \in \Kcal_j \setminus \Kcal_{j-1}} p^{(i)} 
     = \sum_{i = 1}^j p^{(i)}  (|\Kcal_i| - |\Kcal_{i-1}|)
\end{align}
where we define $\Kcal_0 = \varnothing$ for convenience. Consider $C = c \mas \wh \numS H^2 \ln^3 \frac{|\Xcal| H}{p_0 \delta}$ with a large enough numerical constant $c$ so that induction assumption implies $|\Kcal_i| \leq C / p^{(i)}$ for $i = 1, \dots, j-1$. Assume further that $|\Kcal_j| \geq C / p^{(j)}$. Then with $ 1/ p^{(0)} \defeq 0$
\begin{align}
   \frac{1}{|\Kcal_j|} \sum_{k \in \Kcal_j}  V^{\pi_k}_1(s_{k,1})
    \leq 
p^{(j)} \sum_{i = 1}^j p^{(i)}  \left( \frac{1}{p^{(i)}} - \frac{1}{p^{(i-1)}} \right).
\end{align}
Define now $w_i = \frac{1}{p^{(i)}} - \frac{1}{p^{(i-1)}}$, which allows us to write $p^{(i)} = \frac{1}{\sum_{l=1}^i w_l}$ because $\sum_{l=1}^i w_l = \frac{1}{p^{(i)}} - \frac{1}{p^{(0)}} = \frac{1}{p^{(i)}}$. Writing the expression above in terms of $w_i$  yields
\begin{align}
   \frac{1}{|\Kcal_j|} \sum_{k \in \Kcal_j}  V^{\pi_k}_1(s_{k,1})
    \leq 
p^{(j)} \sum_{i = 1}^j \frac{w_i}{\sum_{l=1}^i w_l}
&\overset{(i)}{=} p^{(j)}\left( 1 + \ln\left(\sum_{i=1}^j w_i \right) - \ln w_1 \right)\\
&= p^{(j)}\left( 1 + \ln\frac 1 {p^{(j)}} - \ln \frac 1 {p^{(1)}}\right)\\
&= p^{(j)}\ln\frac{e p^{(1)}} {p^{(j)}} \leq p^{(j)}\ln\frac{e H} {p_0},
\end{align} where $(i)$ follows from the fundamental theorem of calculus (see e.g. Lemma~E.5 by \citet{dann2017unifying}). We just showed that if $|\Kcal_j| \geq C / p^{(j)}$, the average policy value $\frac{1}{|\Kcal_j|} \sum_{k \in \Kcal_j}  V^{\pi_k}_1(s_{k,1})$ cannot be much larger than $1 / p^{(j)}$. Plugging this back into \eqref{eqn:basebound} gives that
\begin{align}
        |\Kcal_j| = O\left(
        \frac{ \mas H p^{(j)}}{(p^{(j)})^2} \ln\frac{e H} {p_0} \ln^2 \frac{|\Xcal| H T_1}{\delta}
    + \frac{ \mas \wh \numS H^2}{p^{(j)}} \ln^3 \frac{|\Xcal| H T_1}{\delta}\right) =
     O\left(
      \frac{\mas \wh \numS  H^2}{p^{(j)}} \ln^3 \frac{|\Xcal| H}{p_0 \delta}\right),
\end{align}
where the equality follows since $\ln (T_1) \lesssim \ln \frac{|\Xcal| H}{p_0 \delta}$. We have just shown that \eqref{eqn:induc1} also holds for $j$ which completes the inductive argument. Evaluating \eqref{eqn:induc1} for $j = \gamma$ shows that the length of the first phase is indeed $O\left(
      \frac{ \mas \wh \numS  H^2}{p_0} \ln^3 \frac{|\Xcal| H}{p_0 \delta}\right)$
      which completes the proof.
\end{proof}

\begin{lemma}
For any tabular episodic MDP with episode length $H$, state-action space $\Xcal$ and a directed feedback graph $G$, the total size of certificates of Algorithm~\ref{alg:mb} on any (possibly random) set of episodes indices $\Kcal$ as is bounded in event $E$ (defined in Lemma~\ref{lem:goodprob}) as
\begin{align}
    \sum_{k \in \Kcal} \Vub_{k,1}(s_{k,1}) - \Vlb_{k,1}(s_{k,1})
    = 
	&O\left(\sqrt{\mas H \sum_{k \in \Kcal}  V^{\pi_k}_1(s_{k,1})} \ln \frac{|\Xcal| H T}{\delta} 
	 + \mas \wh \numS H^2 \ln^3 \frac{|\Xcal| H T}{\delta}\right),
\end{align}   where $T = \max \{k \colon k \in \Kcal\}$ is the largest episode index in $\Kcal$.
\label{lem:subsetcumipoc}
\end{lemma}
\begin{proof}
The proof of this lemma is in complete analogy to the proof of Theorem~\ref{thm:cipoc_independencenumber_app}, except that we take the sum $\sum_{k \in \Kcal}$ instead of $\sum_{k=1}^T$. In the decomposition in Equation~\eqref{eqn:decomp_thm1}, we replace in term $\prnmarker{d}$ the sum over $\Kcal$ with $[T]$ and proceed normally (which yields the $\ln T$ terms). But in term $\prnmarker{e}$ we keep the sum over $\Kcal$ which yields the  $\sum_{k \in \Kcal}  V^{\pi_k}_1(s_{k,1})$ term in the bound above.
\end{proof}

\begin{lemma}
Consider any tabular episodic MDP with state-action space $\Xcal$, episode length $H$ and directed feedback graph $G$ with mas-number $\mas$. For any $\epsilon > 0$, $m \in \NN$ and (possibly random) subset of episodes $\Kcal \subseteq [T]$ with
\begin{align}
    |\Kcal| = O \left(m + 
    \frac{ \mas H \frac{1}{|\Kcal|} \sum_{k \in \Kcal}  V^{\pi_k}_1(s_{k,1})}{\epsilon^2} \ln^2 \frac{|\Xcal| H T}{\delta}
    + \frac{ \mas \wh \numS H^2}{\epsilon} \ln^3 \frac{|\Xcal| H T}{\delta}\right). 
\end{align}
Algorithm~\ref{alg:mb_app} produces in event $E$ (defined in Lemma~\ref{lem:goodprob} at least $m$ certificates with size $\Vub_{k,1}(s_{k,1}) - \Vlb_{k,1}(s_{k,1}) \leq \epsilon$ with $k \in \Kcal$.
\label{lem:subsetsamplecompl}
\end{lemma}
\begin{proof}
The proof of this lemma is in complete analogy to the proof of Proposition~\ref{prop:general_samplecomplexity_mb}, except that we take the sum $\sum_{k \in \Kcal}$ instead of $\sum_{k=1}^T$ when we consider the cumulative certificate size and apply Lemma~\ref{lem:subsetcumipoc} instead of Theorem~\ref{thm:cipoc_independencenumber_app}.
\end{proof}

\subsection{Comparison to Lower Bound} 
In general MDPs where we do not have a good idea about how reachable the dominating set is and whether the MDP has sparse transitions, the sample-complexity of Algorithm~\ref{alg:dominatingset} is 
\begin{align}
			 \widetilde O\left(\frac{\mas \numS H^2}{p_0} + \frac{\gamma H^3}{p_0\epsilon^2}
+ \frac{\gamma \numS H^3}{p_0\epsilon}\right),	
\end{align}
while the lower bound is
\begin{align}
    \widetilde \Omega\left( \frac{\alpha H^2}{\epsilon^2} \wedge \left(\frac{\alpha}{p_0} + \frac{\gamma H^2}{p_0 \epsilon^2}\right)\right).
\end{align}
When $\epsilon$ is small enough and the dominating set is of good quality, i.e., $\gamma < \alpha$, the second term dominates the first in the lower bound. We see that the $1/p_0$ dependency in our sample-complexity upper bound is tight up to log factors. Nonetheless, there is a gap of $H^2$ and $\numS H$ between our upper- and lower-bound even when the feedback graph is symmetric (where $\mas = \alpha$).
It should be noted that the explicit $\numS$ dependency in the $1/\epsilon$-term is typical for model-based algorithms and it is still an open problem whether it can be removed without increase in $H$ for model-based algorithms in MDPs with dense transitions. 

However, the lower bound in Theorem~\ref{thm:domsetlowerbound} relies on a class of MDPs that in fact have sparse transitions. If we know that the true MDP belongs to this class, then we can run Algorithm~\ref{alg:dominatingset} with the planning routine of Algorithm~\ref{alg:mb_app} that supports state-action-dependent upper-bounds and set 
\begin{align}
    Q^{\max}_{h}(x) &= 1, ~ \quad V^{\max}_{h+1}(x) = 1 \quad \textrm{for $x$ in tasks } \left\{1, \dots, \gamma\right\} \text{ and  } \wh \numS = 2,
\end{align} because each dominating node can only be reached once per episode and each state-action pair can only transition to one of two states. With these modifications, one can show that Algorithm~\ref{alg:dominatingset} terminates within 
\begin{align}
			 \widetilde O\left(\frac{\mas H}{p_0} + \frac{\gamma H^3}{p_0\epsilon^2}\right)	
\end{align}
episodes matches the lower-bound up to one factor of $H$ and log-terms in symmetric feedback graphs for small enough $\epsilon$.

\section{Lower Bound Proofs}
\label{app:lowerboundproofs}

\subsection{Lower Regret Bound with Independence Number}

\begin{figure}
    \centering
    \includegraphics[width=0.5\textwidth]{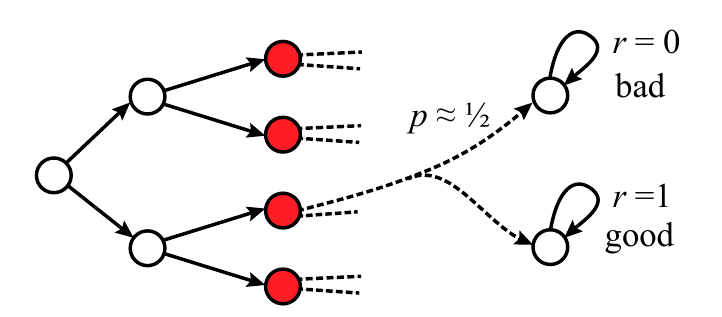}
    \caption{Lower bound construction depicted for $\numA=2$ actions: This family of MDPs is equivalent to a Bernoulli bandit with $N\numA$ arms where rewards are scaled by $\bar H = \lfloor H - 1 - \log_{\numA} N \rfloor$.}
    \label{fig:bandit_as_mdp}
\end{figure}
For convenience, we restate the theorem statement:
\lowerboundindep*
\begin{proof}[Proof of Theorem~\ref{thm:indeplowerbound}]
    We first specify a family of MDPs that are hard to learn with feedback graphs, then show that learning in hard instances of MABs with $\alpha$ arms can be reduced to learning in this family of MDPs and finally use this reduction to lower bound the regret of any agent.
    
    \paragraph{Family of hard MDPs $\Mcal$: }
    Without loss of generality, we assume that $N = A^k$ for some $k \in \NN$.We consider a family $\Mcal$ of $\alpha$ MDPs which are illustrated in Figure~\ref{fig:bandit_as_mdp}. Each MDP in $\Mcal$ has $N$ red states (and $N+1$ white states) that form the leaves of a deterministic tree with fan-out $\numA$. This means that each red state is deterministically reachable by a sequence of actions of length $\lceil \log_{\numA} N \rceil$. From each red state, the agent transitions to a good absorbing state with certain probability and with remaining probability to a bad absorbing state. All rewards are $0$ except in the good absorbing state where the agent accumulates reward of $1$ until the end of the episode (for a total of $\bar H \defeq H - 1 - \lceil \log_{\numA} N \rceil$ time steps).
    
    Let now $G_1$ and $G_2$ be the feedback graphs for the red and white state-actions respectively. Further let $\Ncal$ be an independent set of $G_1$. Each MDP $M_i \in \Mcal$ is indexed by an optimal pair $i = (s^\star, a^\star) \in \Ncal$ of a red state-action pair. When the agent takes $a^\star$ in $s^\star$ it transitions to the good state with probability $\delta + \epsilon$. For all other pairs in $\Ncal$, it transitions to the good state with probability $\delta$. All remaining pairs of red states and actions have probability $0$ of reaching the good state. The values of $\delta, \epsilon > 0$ will be specified below.

    \paragraph{Reduction of learning in MABs to RL in MDPs $\Mcal$:} We now use a reductive argument similar to \citet[Theorem 4]{mannor2011bandits} to show learning in MABs with $\alpha$ actions cannot me much harder than learning in $\Mcal$.
    
    Let $\Bcal$ be any MDP algorithm and denote by $R_{\Bcal, M_i}(T)$ its expected regret after $T$ episodes when applied to problem instance $M_i \in \Mcal$. We can use $\Bcal$ to construct a multi-armed bandit algorithm $\Bcal'$ for a family of MABs $\Mcal'$ indexed by $\Ncal$. Each MAB $M'_i \in \Mcal'$ has $|\Ncal|$ arms, all of which have Bernoulli($\delta$) rewards except $i$ which has Bernoulli($\delta + \epsilon$) rewards.
    To run $\Bcal'$ on $M'_i \in \Mcal'$, we apply $\Bcal$ to $M_i \in \Mcal$.
    Whenever $\Bcal$ chooses to execute an episode that visits a $j \in \Ncal$, $\Bcal'$ picks arm $j$ in $M'_i$ and passes on the observed reward as an indicator of whether the good state was reached. When $\Bcal$ chooses to execute an episode that passes through a vertex $x$ of $G_1$ that is not in the independent set $\Ncal$, then $\Bcal'$ pulls all children $\{ y \in \Ncal \colon x \rightarrow_{G_1} y \}$ that are in the independent set in an arbitrary order. The observed rewards are again used to construct the observed feedback for $\Bcal$ by interpreting them as indicators for whether the good state was reached.
    
    \paragraph{Lower bound on regret: } 
    We denote by $T'$ the (random) number of pulls $\Bcal'$ takes until $\Bcal$ has executed $T$ episodes and by $U$ the expected number of times $\Bcal$ plays episodes that do not visit the independent set. 
    The regret of $\Bcal'$ after $T'$ pulls can then be written as
    \begin{align}
       R_{\Bcal', M_i'}(T') \leq  \frac{R_{\Bcal, M_i}(T)}{\bar H} 
       + \epsilon \alpha U - \delta U,
    \end{align}
    where the first term $\frac{R_{\Bcal, M_i}(T)}{\bar H}$ is the regret accumulated from pulls where $\Bcal$ visits the independent set and the second term from the pulls where $\Bcal$ did not visit the independent set. Each such episode incurs $\delta$ regret for $\Bcal$ and up to $\alpha \epsilon$ regret for $\Bcal'$. We rearrange this inequality as
    \begin{align}
       R_{\Bcal, M_i}(T) 
       \geq 
       \bar H (R_{\Bcal', M_i'}(T')  
       - \epsilon \alpha U + \delta U) 
       \circledmarked{1}{\geq} &
       \bar H (R_{\Bcal', M_i'}(T)  
       - \epsilon \alpha U + \delta U) \\
              \circledmarked{2}{\geq} & 
       \bar H R_{\Bcal', M_i'}(T)  
        + \bar H T[\delta - \epsilon \alpha]^{-} 
        \label{eqn:reglb1},
    \end{align}
    where \circledmarker{1} follows from monotonicity of regret and \circledmarker{2} from considering the best case $U \in [0, T]$ for algorithm $\Bcal$.
    The worst-case regret of $\Bcal'$ in the $\Mcal'$ has been analyzed by \citet{osband2016lower}. We build on their result and use their Lemma~3 and Proposition~1 to lower bound the regret for $\Bcal'$ as follows
    \begin{align}
       \max_i R_{\Bcal', M_i'}(T) 
      \geq \epsilon T \left( 1 - \frac 1 \alpha - \epsilon\sqrt{\frac{T}{2 \delta \alpha}}\right)
      = &
     \frac{1}{4}\sqrt{\frac{\alpha}{2T}} T \left( 1 - \frac 1 \alpha - \frac{1}{4}\sqrt{\frac{\alpha}{2T}} \sqrt{\frac{4T}{2 \alpha}}\right)\\
           = &
     \sqrt{\frac{\alpha T}{32}} \left( \frac{3}{4} - \frac 1 \alpha\right),
    \end{align}
    where we set $\delta = \frac 1 4$ and $\epsilon = \frac{1}{4}\sqrt{\frac{\alpha}{2T}}$ (which satisfy $\epsilon \leq 1 - 2 \delta$ required by Proposition~1 for $T \geq \alpha / 8$). Plugging this result back into \eqref{eqn:reglb1} gives a worst-case regret bound for $\Bcal$ of 
    \begin{align}
        \max_i R_{\Bcal, M_i}(T)
        \geq 
               \sqrt{\frac{\alpha T}{32}} \left( \frac{3}{4} - \frac 1 \alpha\right)
        + \bar H T[\delta - \epsilon \alpha]^{-} \geq \bar H \sqrt{\frac{\alpha T}{32}} \left( \frac{3}{4} - \frac 1 \alpha\right) \geq
        \frac{H}{32}\sqrt{\frac{\alpha T}{2}}, 
    \end{align}
    where we first dropped the second term because $\delta \geq \epsilon \alpha$ for $T \geq \alpha^3 / 8$ and then used the assumptions $\alpha \geq 2$ and $H \geq 2 + 2\log_{\numA} N$.
\end{proof}

\subsection{Lower Sample Complexity Bound with Domination Number}

\begin{proof}[Proof of Theorem~\ref{thm:domsetlowerbound}]
Let $Z = \frac{\numS}{8}$ and $\bar Z = Z \numA$ which we assume to be integer without loss of generality. 
The family of MDPs consists of $\bar Z \times \bar Z$ MDPs, indexed by $(i,j) \in [\bar Z]^2$. All MDPs have the same structure:

\paragraph{Family of MDPs:} 
We order $4Z$ states in a deterministic tree so any of the $2 Z$ leaf nodes can be reached by a specific action sequence. See Figure~\ref{fig:domsetowerbound} for an example with two actions. We split the state-action pairs at the leafs in two sets $\Bcal_1= \{x_{1}, \dots x_{\bar Z} \}$ and $\Bcal_2 = \{z_{1}, \dots z_{\bar Z}\}$, each of size $\bar Z$. 
Playing $x_i$ transitions to the good absorbing state with some probability $g$ and otherwise to the bad absorbing state $b$. The reward is $0$ in all states and actions, except in the good state $g$, where agent receives a reward of $1$. The transition probabilities from $x_i$ depend on the specific MDP. Consider MDP $(j,k)$, then
\begin{align}
    P( g | x_1) = \frac{1}{2} + \frac{\epsilon}{2H}, \qquad \textrm{and,} \qquad
    P( g | x_i) = \frac{1}{2} + \frac{\epsilon}{H} \one\{i = k\}.
\end{align}
Hence, the first index of the MDP indicates which $x_i$ is optimal. Since the agent will stay in the good state for at least $H/2$ time steps (by the assumption that $H \geq 2 \log_{\numA}(\numS / 4)$ by assumption), the agent needs to identify which $x_i$ to play in order to identify an $\frac{\epsilon}{4}$-optimal policy.
All pairs $z_i$ transition to the bad state deterministically, except for pair $z_j$ in MDPs $(i,j)$. This pair transitions with probability $p_0$ to another tree of states (of size at most $2Z$) which has $\gamma$ state-action pairs at the leafs, denoted by $\Dcal = \{d_1, \dots, d_{\gamma}\}$. All pairs in this set transition to the bad state deterministically.

The feedback graph is sparse. There are no edges, except each node $d_i$ has exactly $\frac{\bar Z}{\gamma}$ edges (which we assume to be integer for simplicity) to pairs in $\Bcal_1$. No nodes $d_i$ and $d_j$ point to the same node. Hence, $\Dcal$ forms a dominating set of the feedback graph.

\paragraph{Sample Complexity:}
The construction of $\Bcal_1$ is equivalent to the multi-armed bandit instances in Theorem~1 by \citet{mannor2004sample}.
Consider any algorithm and let $o_i$ be the number of observations an algorithm has received for $x_i$ when it terminates.
By applying Theorem~1 by \citet{mannor2004sample}, we know that if the algorithm indeed outputs an $\frac{\epsilon}{4}$-optimal policy with probability at least $1 - \delta$ in any instance of the family, it has to collect in instances $(1, j)$ at least the following number of samples in expectation
\begin{align}
    \EE_{(1,j)}[o_i] \geq \frac{c_1H^2}{\epsilon^2}\ln \frac {c_2} \delta
    \label{eqn:banditlb}
\end{align}
for some absolute constants $c_1$ and $c_2$. Let $v(x)$ be the number of times the algorithm actually visited a state-action pair $x$. Then
$o_i = v(x_i) + v(d_j)$ for $j$ with $d_j \rightarrow_{G} x_i$ because the algorithm can only observe a sample for $x_i$ if it actually visits it or the node in the dominating set.
Applying this identity to \eqref{eqn:banditlb} yields
\begin{align}
    \EE_{(1,k )}[v(d_j)] \geq \frac{c_1H^2}{\epsilon^2}\ln \frac {c_2} \delta 
        -  \EE_{(1,k)}[v(x_i)]
\end{align}
for all $d_j$ and $x_i$ with $d_j \rightarrow x_i$. Summing over $i \in \bar Z$ and using the fact that each $d_j$ is counted $\bar Z / \gamma$ times, we get
\begin{align}
    \frac{\bar Z}{\gamma} \sum_{j=1}^{\gamma} \EE_{(1,k)}[v(d_j)] \geq \frac{c_1 \bar Z H^2}{\epsilon^2}\ln \frac {c_2} \delta 
        -  \sum_{i=1}^{\bar Z} \EE_{(1,k)}[v(x_i)].
\end{align}
After renormalizing, we get,
\begin{align}
    \sum_{j=1}^{\gamma} \EE_{(1,k)}[v(d_j)] \geq \frac{c_1 \gamma H^2}{\epsilon^2}\ln \frac {c_2} \delta 
        -  \frac{\gamma}{\bar Z} \sum_{i=1}^{\bar Z} \EE_{(1,k)}[v(x_i)].
\end{align}
Next, observe that either the algorithm needs to visit nodes in $\Bcal_1$ at least $\frac{c_1 \bar Z H^2}{2\epsilon^2}\ln \frac {c_2} \delta$ times in expectation or nodes in the dominating set $\Dcal$ at least
$\frac{c_1 \gamma H^2}{2\epsilon^2}\ln \frac {c_2} \delta$ times in expectation. The former case gives an expected number of episodes of $\Omega\left( \frac{ \numS \numA H^2}{\epsilon^2} \ln \frac 1 \delta \right)$ which is the second term in the lower-bound to show. 

It remains the case where $\sum_{j=1}^{\gamma} \EE_{(1,k)}[v(d_j)] \geq \frac{c_1 \gamma H^2}{2\epsilon^2}\ln \frac {c_2} \delta$. The algorithm can only reach the dominating set through $z_k$ and it can visit only one node in the dominating set per episode. Further, when the algorithm visits $z_k$, it only reaches the dominating set with probability $p_0$. Hence, 
\begin{align}
    \EE_{(1,k)}[v(z_k)] = p_0  \sum_{j=1}^{\gamma} \EE_{(1,k)}[v(d_j)]  \geq \frac{c_1 \gamma H^2}{2 p_0 \epsilon^2}\ln \frac {c_2} \delta,
\end{align}
but the algorithm may also visit other pairs $z_i$ for $i \neq k$ as well. To see this, consider the expected number of visits to all $z_i$s before the algorithm visits the dominating set for the first time. By Lemma~\ref{eqn:coinfind}, this is at least $\frac{\bar Z}{4 p_0}$ in the worst case over $k$. This shows that 
\begin{align}
    \max_{k \in [\bar Z]} \sum_{i=1}^{\bar Z}\EE_{(1,k)}[v(z_i)] \geq \frac{c_1 \gamma H^2}{2 p_0 \epsilon^2}\ln \frac {c_2} \delta - 1 + \frac{\bar Z}{4 p_0} 
    = \Omega\left(\frac{\gamma H^2}{p_0 \epsilon^2} \ln \frac{c_2}{\delta} + \frac{\numS \numA}{p_0} \right). 
\end{align}
\end{proof}

\begin{lemma}
Consider $k$ biased coins, where all but one coin have probability $0$ of showing heads. Only one coin has probability $p$ of showing head. The identity $i$ of this coin is unknown. The expected number of coin tosses $N$ until the first head is for any strategy 
\begin{align}
    \EE[N] \geq \frac{k}{4p} 
\end{align}
in the worst case over $i$.
\label{eqn:coinfind}
\end{lemma}
\begin{proof}
Let $N$ be the number of coin tosses when the first head occurs and let Alg be a strategy. The quantity of interest is
\begin{align}
     \inf_{\textrm{Alg}} \max_{i \in [k]} \EE_{t(i)}[ \EE_{R(\textrm{Alg})}[N]],
\end{align}
where $R(\textrm{Alg})$ denotes the internal randomness of the strategy and $t(i)$ the random outcomes of coin tosses. We first simplify this expression to
\begin{align}
     \inf_{\textrm{Alg}} \max_{i \in [k]} \EE_{t(i)}[ \EE_{R(\textrm{Alg})}[N]]
     & \geq
    \inf_{\textrm{Alg}} \frac{1}{k} \sum_{i=1}^k \EE_{t(i)}[ \EE_{R(\textrm{Alg})}[N]]
    = 
    \inf_{\textrm{Alg}}\EE_{R(\textrm{Alg})}\left[ \frac 1 k \sum_{i=1}^k \EE_{t(i)}[ N]\right]\\
    & = 
    \inf_{\textrm{Alg}}\EE_{R(\textrm{Alg})}\left[ \frac 1 k \sum_{i=1}^k \sum_{m=0}^\infty \PP_{t(i)}[ N > m]\right]\\
    & = 
    \inf_{\textrm{Alg}}\EE_{R(\textrm{Alg})}\left[\sum_{m=0}^\infty \frac 1 k \sum_{i=1}^k  \PP_{t(i)}[ N > m]\right], \label{eqn:coineq1}
\end{align}
and derive an explicit expression for $\frac 1 k \sum_{i=1}^k \PP_{t(i)}[ N > m]$. Since the strategy and its randomness is fixed, it is reduced to a deterministic sequence of coin choices. That is, for a given number of total tosses $N$, a deterministic strategy is the number of tosses of each coin $n_1, \dots, n_k$ with $\sum_{i=1}^k n_i = N$.
Consider any such strategy and let a $n_1, \dots, n_k$ with $\sum_{i=1}^k n_i = m$ be the coins selected up to $m$. If $N > m$, then the first $n_i$ tosses of coin $i$ must be tail. Hence, using the geometric distribution, we can explicitly write the probability of this event as
\begin{align}
\frac 1 k \sum_{i=1}^k  \PP_{t(i)}[ N > m] \geq \frac{1}{k} \sum_{i=1}^k (1 - p)^{n_i}
\geq \inf_{n_{1:k} \colon \sum_{i=1}^k n_i = m} \frac{1}{k} \sum_{i=1}^k (1 - p)^{n_i}.
\end{align}
The second inequality just considers the worst-case. The expression on the RHS is a convex program over the simplex with a symmetric objective. The optimum can therefore only be attained at an arbitrary corner of the simplex or the center.
The value at the center is $(1-p)^{m/k}$ and by Young's inequality, we have
\begin{align}
    (1-p)^{m/k} \leq \frac{((1 - p)^{m/k})^k}{k} + \frac{1 ^ {k / k-1}}{k / (k-1)} = 
    \frac{(1 - p)^{m}}{k} + \frac{k-1}{k},
\end{align}
where the RHS is the value of the program at a corner. Hence
$\frac 1 k \sum_{i=1}^k  \PP_{t(i)}[ N > m] \geq (1-p)^{m/k}$ holds and plugging this back into \eqref{eqn:coineq1} gives
\begin{align}
    \inf_{\textrm{Alg}} \max_{i \in [k]} \EE_{t(i)}[ \EE_{R(\textrm{Alg})}[N]]
     \geq
    \sum_{m=0}^\infty (1-p)^{m/k} = \frac{1}{1 - (1 - p)^{1/k}} \geq \frac{k}{4p}, 
\end{align}
where the last inequality follows from basic algebra and holds for $p < 0.5$. For $p \geq 0.5$, the worst case is $\frac{k}{4p} \leq \frac{k}{2}$ anyway because that is the expected number of trials until one can identify a coin with $p=1$.
\end{proof}
\end{document}